%% file: arxiv.tex
\documentclass[runningheads]{llncs}
 
\usepackage[mobile]{eccv}

\usepackage{eccvabbrv}
\usepackage{lmodern}
\usepackage{graphicx}
\usepackage{wrapfig}
\usepackage{microtype}
\usepackage[svgnames,table]{xcolor}
\usepackage[T1]{fontenc}
\usepackage[utf8]{inputenc}
\usepackage{pifont}
\usepackage{booktabs,makecell,multicol,multirow}
\usepackage{amssymb,amsfonts,amsmath}
\usepackage[percent]{overpic}
\usepackage[accsupp]{axessibility}  
\usepackage[svgnames]{xcolor}
\usepackage{tikz}
\usetikzlibrary{positioning, calc, fit, arrows.meta, decorations.pathreplacing, backgrounds, shapes.geometric}
\tikzset{
    img/.style={inner sep=0,anchor=center},
    mylabel/.style={anchor=north,yshift=.75cm,align=center,text depth=.75cm,font=\small},
}
\usetikzlibrary{spy}
\usepackage{pgfplots}
\pgfplotsset{compat=1.17}
\usepackage{algorithmic}
\usepackage{algorithm}
\usepackage{svg}
\usepackage{caption}
\usepackage{rotating}

\usepackage[breaklinks,colorlinks,citecolor=eccvblue]{hyperref}
\usepackage{orcidlink}

\definecolor{lightgray}{gray}{0.85}

\newcommand{\bbR}{\mathbb{R}}
\renewcommand{\Re}{\mathbb{R}}

\newcommand{\cD}{\mathcal{D}}

\newcommand{\cI}{\mathcal{I}}
\newcommand{\cN}{\mathcal{N}}

\newcommand{\cR}{\mathcal{R}}
\newcommand{\cS}{\mathcal{S}}
\renewcommand{\cS}{{S}}
\newcommand{\cT}{\mathcal{T}}
\renewcommand{\cT}{{T}}
\newcommand{\cU}{\mathcal{U}}

\renewcommand{\le}{\leqslant}
\renewcommand{\leq}{\leqslant}
\renewcommand{\ge}{\geqslant}

\def\A{\mat A}
\def\I{\mat I}

\def\R{\mat R}

\def\W{\mat W}

\def\b{\bs b}

\def\n{\bs n}
\def\p{\bs p}
\def\q{\bs q}
\def\s{\bs s}
\def\u{\bs u}
\def\v{\bs v}

\def\x{\bs x}
\def\y{\bs y}
\def\z{\bs z}

\newcommand{\Tpgd}{\cT_{\gamma}}
\newcommand{\Thqs}{\cT_{\rho}}
\newcommand{\Tadmm}{\cT_{\alpha}}
\newcommand{\Tred}{\cT_{\gamma,\lambda}}

\newcommand{\Dctr}{\cD}
\newcommand{\Dsig}{\Dctr_{\sigma}}
\newcommand{\Sctr}{\cS_{\mathrm{ctr}}}
\renewcommand{\Sctr}{\cS}
\newcommand{\Rctr}{\cR_{\mathrm{ctr}}}
\renewcommand{\Rctr}{\cR}

\newcommand{\etap}{\eta_{\p}}
\newcommand{\ttheta}{\tilde{\theta}}

\newcommand{\mat}[1]{\mathbf{#1}}
\newcommand{\bs}[1]{\boldsymbol{#1}}

\newcommand{\Lspline}{\bs \psi}
\newcommand{\Qspline}{\bs \phi}
\newcommand{\lspline}{\psi}
\newcommand{\qspline}{\phi}
\newcommand{\LsplineScal}{\bs{\s_{\nu}}}
\newcommand{\lsplineScal}[1][]{s_{\nu_{#1}}}
\newcommand{\ScalFunc}{\bs \nu}
\newcommand{\Poten}{\varphi}
\newcommand{\PotenScal}{\varphi_\sigma}

\newcommand{\norm}[2][]{\left\lVert #2\right\rVert_{#1}}

\newcommand{\BestCol}[1]{\textcolor{Teal}{#1}}
\newcommand{\runupCol}[1]{\textcolor{orange!90!black}{#1}}

\newcommand{\cmark}{\textcolor{Green}{\ding{51}}} 
\newcommand{\xmark}{\textcolor{Red}{\ding{55}}} 

\newcommand{\lenSR}{0.04\textwidth}
\newcommand{\lenBlur}{0.06\textwidth}
\newcommand{\Rtwo}{0.47\columnwidth}
\newcommand{\Rthree}{0.28\columnwidth}
\newcommand{\RfourC}{0.25\columnwidth}
\newcommand{\Rfour}{0.24\columnwidth}

\DeclareMathOperator{\prox}{\mathrm{prox}}
\makeatletter
\def\@fnsymbol#1{\ensuremath{\ifcase#1\or \dagger\or \ddagger\or
        \mathsection\or \mathparagraph\or \|\or **\or \dagger\dagger
        \or \ddagger\ddagger \else\@ctrerr\fi}}
\makeatother
\begin{document}

\title{Stabilizing Deep Reconstruction Operators with Contractive Anchoring}


\author{Arghya Sinha\inst{1}\orcidlink{0009-0005-7745-1082}\thanks{Corresponding Author} \and
Trishit Mukherjee\inst{1}\orcidlink{0009-0000-1849-9290} \and
Kunal N. Chaudhury\inst{1}\orcidlink{0000-0002-8136-605X}}

\authorrunning{A.~Sinha et al.}

\institute{Indian Institute of Science, Bengalore, India \\
\email{\{arghyasinha,trishitm,kunal\}@iisc.ac.in}
}

\maketitle

\begin{abstract}
Pretrained deep denoisers can be used to solve a wide range of model-based image reconstruction tasks via Plug-and-Play (PnP) and Regularization-by-Denoising (RED) algorithms, without retraining per task. These denoisers are trained only for single-step denoising. Using them as Image Reconstruction (IR) regularizers in an iterative process can destabilize reconstruction. A common failure mode is the peak-and-collapse behaviour: metrics such as PSNR improve for early iterations and then abruptly degrade, making these algorithms unreliable in practice. We propose a data-driven stabilization framework that (i) formalizes this instability of any IR operator through a local quantity and (ii) prevents collapse by regularizing this quantity adaptively, requiring no retraining or modification of the given pretrained network. Our key idea is to control the potentially unstable IR operator with a contractive operator whose stable iterates act as an anchor and prevent collapse. We further introduce an efficient family of trainable contractive operators that serve as strong anchors while remaining lightweight. Extensive experiments across proximal algorithms, denoiser architectures, noise levels, and imaging tasks show consistent, collapse-free performance and improved reliability of PnP and RED reconstruction.
\end{abstract}

\section{Introduction}
\label{sec:intro}

The problem of image reconstruction (IR) from noisy linear measurements arises in tasks such as deblurring, super-resolution, magnetic resonance imaging, and computed tomography~\cite{bouman2022foundations}. 
Image reconstruction is commonly formulated as the regularized least-squares problem
\begin{equation}
\label{eq:VarProb}
\min_{\x \in \bbR^n} f(\x) + g(\x),\qquad f(\x) = \frac{1}{2}\!\norm[]{\A\x-\b}^2,
\end{equation}
where $\b$ is the measurement of the ground-truth image $\bar{\x}\in\bbR^n$ under the degradation model $\b=\A\bar{\x}+\n$. The forward operator $\A\in\bbR^{m\times n}$ is assumed known, and $\n$ denotes measurement noise. The data-fidelity term $f$ enforces consistency with the measurement model, while the regularizer $g$ incorporates prior information about the image.
In practice, however, designing an explicit regularizer $g$ that yields high-quality reconstructions is often difficult. 
Plug-and-Play (PnP)~\cite{venkatakrishnan2013pnp} and Regularization-by-Denoising (RED)~\cite{romano2017little,reehorst2018regularization} avoid explicit regularizer design by incorporating a pretrained Gaussian denoiser into standard proximal solvers such as Proximal Gradient Descent (PGD)~\cite{beck2017book} and Half-Quadratic Splitting (HQS)~\cite{geman1995nonlinear,zhang2021plug}. The denoiser acts as an implicit prior and has shown strong performance across diverse IR problems~\cite{zhang2021plug,hurault2022gradient,cohen2021has}.

In PnP, the denoiser directly parametrizes the fixed-point iteration associated with the chosen proximal algorithm. For example, the PnP-PGD iteration is
\begin{equation}
\label{eq:pnp}
\x_{k+1} = \cD\!\left(\x_k - \gamma \nabla f(\x_k)\right),
\end{equation}
In contrast, RED uses the denoiser $\cD$ to define a Laplacian-type regularizer $g$ with gradient $\nabla g=\cI-\cD$. The corresponding gradient-descent iteration is
\begin{equation}
\label{eq:red-i}
\x_{k+1} = \x_k - \gamma\big(\nabla f(\x_k) + \lambda\nabla g(\x_k)\big),
\end{equation}
where $\lambda>0$ is a regularization weight and $\gamma>0$ is a step size.

Since the denoiser is decoupled from the measurement operator $\A$, the same pretrained denoiser $\cD$ can be reused across different inverse problems without retraining. This modularity is one of the main practical advantages of the PnP framework: the forward model enters only through the data-fidelity step, while the image prior is provided by the denoiser.
\begin{figure}[t]
    \centering
        \subfloat{
        \includegraphics[width=\columnwidth]{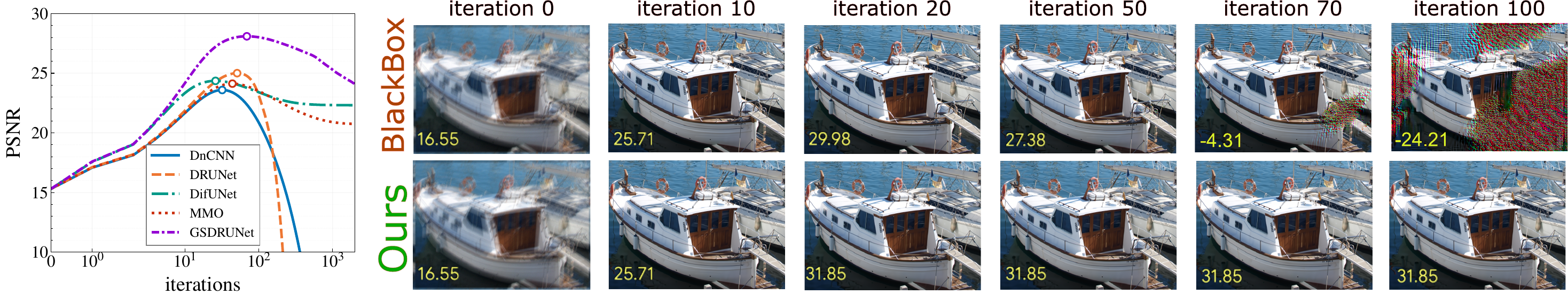}}
\caption{Illustration of peak-and-collapse instability in vanilla PnP across modern deep denoisers. Reconstruction quality improves during the early iterations but can deteriorate abruptly after reaching its peak. A detailed analysis of this behavior is provided in~\Cref{fig:rot} and~\Cref{tab:cbsd10_deblurring_vista,tab:cbsd10_superresolution_vista}.}
\label{fig:inst_deep_den}
\end{figure}
A key limitation of this method is that standard denoising networks are trained for \emph{single-step} inference. When such networks are used repeatedly within a PnP algorithm, they often behave unreliably and produce unpredictable reconstructions. This has motivated a large body of work on \emph{convergent} PnP, where additional constraints are imposed to guarantee convergence of the iterates~\cite{pesquet_learning_2021,hertrich_convolutional_2021,cohen_regularization_2021,hurault2022gradient,hurault_proximal_2022,nair2024averaged,goujon2023neural,goujon_learning_2024,pourya2025dealing}. These methods typically aim to control the Lipschitz behavior of PnP system, often by retraining the denoiser or imposing architectural constraints.

Broadly, two main approaches have emerged. In the first approach, a neural network is used to define an explicit regularizer, which is then optimized together with the model-based loss in an iterative manner~\cite{romano2017little,cohen2021has,reehorst2018regularization,cohen_regularization_2021,hurault2022gradient,hurault_proximal_2022,tan2024provably,goujon2023neural,goujon_learning_2024,pourya2025dealing}. Typically, the network is optimized alongside the proximal algorithm to ensure they work well together. For example,  CNN-based Laplacian regularizers are used in~\cite{hurault2022gradient,cohen_regularization_2021}. Similarly,~\cite{pourya2025dealing,goujon_learning_2024} use quadratic or weakly convex regularizers and train denoisers tailored to their specific iterative schemes. However, the denoiser is closely tied to the base algorithm and cannot simply be swapped for another pretrained model while still guaranteeing convergence. 
In the second approach, PnP is viewed as a dynamical system~\cite{pesquet_learning_2021,wei2024learning}, and its convergence is studied using ideas from fixed-point theory~\cite{bauschke2011convex}. A common strategy is to train denoisers from restricted operator classes that satisfy properties such as nonexpansivity or proximability~\cite{pesquet_learning_2021,hurault_proximal_2022,wei2024learning}. 
This restriction limits the use of existing powerful denoisers that were not designed or trained to satisfy these properties.

\paragraph{\normalfont\bfseries Motivation.}
In this work, we aim to leverage existing powerful pretrained denoisers without retraining or modifying them. Since these networks are not trained to satisfy the conditions required for convergence, their repeated use within an iterative reconstruction scheme can lead to unstable behavior. Empirically, we observe a consistent \emph{peak-and-collapse (PC)} pattern (see \Cref{fig:inst_deep_den}): \emph{the reconstruction quality improves during the early iterations, reaches a peak, and then deteriorates abruptly.}
 
This behavior occurs across diverse denoiser families, including CNNs, diffusion models, and transformers. Although a PnP system exhibiting PC can produce strong reconstructions when stopped near the peak (see \Cref{tab:cbsd10_deblurring_vista}), identifying this point reliably is difficult. We therefore adopt a weaker but practically meaningful notion of stability based on reconstruction quality over iterations, with the goal of maintaining stable PSNR beyond the peak without modifying the pretrained network.

 While early stopping at the peak is one way to achieve this stability, it is difficult in practice because identifying the peak accurately is not easy (see Appendix H). A related heuristic in~\cite{terris2024equivariant} improves this stability by applying random group transformations to the iterates. We instead propose a robust \emph{drop-in} mechanism with provable guarantees. More precisely, let 
\begin{equation}
\label{eq:vanilla_pnp}
\x_{k+1} = \cT(\x_k)
\end{equation}
denote the iterates generated by a possibly unstable IR operator $\cT:\Re^n\to\Re^n$. 
A major challenge in our setting is that we treat $\cT$ as a black box, accessible only through input-output evaluations, without assuming knowledge of its internal structure. Our goal is to design a stabilizer that is agnostic to the proximal solver, pretrained denoiser, and measurement model, while imposing no additional assumptions on $\cT$.

\paragraph{\normalfont\bfseries Contributions.}
We propose a practical, data-driven stabilization framework based on \emph{sequential averaging}~\cite{sabach_first_2017}. Sequential averaging is classical in bilevel optimization~\cite{sabach_first_2017,sabach-nonsmooth} and viscosity-based fixed-point methods for selecting a particular fixed-point when multiple fixed points exist~\cite{xu_viscosity_2004,attouch_viscosity_1996}. In our setting, sequential averaging combines a high-performing but potentially unstable black-box operator $\cT$ with a stable contractive anchor. Our main contributions are as follows:
\begin{enumerate}
    \item  We introduce the \emph{stability index} $\eta$ (\Cref{def:eta-index}), which quantifies the local expansiveness of the black-box operator around an anchor point. Using this index together with a contractive anchor, we develop a mechanism that provably controls $\eta$ whenever required.

    \item We derive \Cref{algo:vista}, which adaptively adjusts the amount of anchoring required to control $\eta$ and automatically stabilize the iterations near the peak. A key advantage is that the algorithm operates as a complete \emph{drop-in} mechanism and requires no additional parameter tuning.

    \item We propose a lightweight, noise-aware denoiser that is \emph{contractive by construction} while retaining expressiveness through trainable filters and spline nonlinearities. This design allows unconstrained training while preserving provable contractivity, yielding a strong reconstruction anchor for the stabilization module without compromising denoising quality.
\end{enumerate}

We evaluate the proposed stabilization framework across multiple proximal solvers (PGD, HQS, and ADMM), denoiser families (CNNs, diffusion models, and transformers), and imaging tasks. The results show consistent stabilization while preserving near-peak reconstruction quality.

\section{Stabilization Mechanism}
\label{sec:method}

\paragraph{\normalfont\bfseries Denoiser-driven IR operators.}
Before presenting our stabilization method, we briefly review the denoiser-driven fixed-point operators used in PnP and RED. These operators combine a pretrained denoiser with standard proximal algorithms for image reconstruction.

We begin with PnP-PGD, in which the denoiser is applied after a gradient step on the data-fidelity term $f$:
\begin{equation}
\label{eq:pnppgd}
(\mathrm{PnP}\mbox{-}\mathrm{PGD})\quad 
\x_{k+1} = \Tpgd(\x_k),\qquad 
\Tpgd := \cD \circ \big(\cI - \gamma \nabla f \big),
\end{equation}
where $\cI$ is the identity operator on $\Re^n$ and $\circ$ denotes composition. The corresponding operators for PnP-HQS~\cite{zhang2021plug}, PnP-ADMM~\cite{sreehari2016plug,ryu2019plug}, and RED-GD are
\begin{align}
\hspace{-1em} (\mathrm{PnP}\mbox{-}\mathrm{HQS})\quad 
\x_{k+1} &= \Thqs(\x_k),\quad \Thqs := \cD \circ \prox_{\rho f}, \label{eq:pnphqs}\\
\hspace{-1em}  (\mathrm{PnP}\mbox{-}\mathrm{ADMM})\quad 
\x_{k+1} &= \Tadmm(\x_k),\quad
\Tadmm := \frac{1}{2}\big( \cI + (2\cD - \cI)\circ(2\,\prox_{\alpha f} - \cI) \big), \label{eq:pnpadmm}\\
\hspace{-1em} (\mathrm{RED}\mbox{-}\mathrm{GD})\quad 
 \x_{k+1} &= \Tred(\x_k), \quad 
\Tred := (1-\gamma\lambda)\cI - \gamma\big(\nabla f - \lambda\cD\big), \label{eq:redgd}
\end{align}
where $\rho,\alpha,\gamma,\lambda>0$ are tunable parameters. Together with the choice of denoiser and solver, these parameters strongly influence the behavior of the iterates; see \Cref{fig:rot}. We refer to the operators in~\eqref{eq:pnppgd}–\eqref{eq:redgd} as \textbf{vanilla} operators and denote them generically by $\cT$.

\paragraph{\normalfont\bfseries Local stability measure.}
Since $\cT$ is treated as a black box, its instability cannot be detected from its internal structure.
We therefore seek a surrogate quantity that can be evaluated during the iterations and used to interpret instability. Since a black box
provides only input-output evaluations, the surrogate should depend only on this information. This motivates a quantity that captures how $\cT$ acts around a reference point $\p$. Specifically, we define the following stability index.
\begin{definition}[stability index]
\label{def:eta-index}
Let $\cT:\Re^n\to\Re^n$ be an operator, and let $\p\in\Re^n$ be a reference point. For any $\x\in\Re^n$, the stability index of $\cT$ at $\x$ relative 
to $\p$ is defined as
\begin{equation}
\label{eq:eta-function}
\etap (\x,\cT)
= 
\begin{cases}
    \dfrac{\|\cT(\x)-\p\|}{\|\x-\p\|}, \quad &\x\neq\p\\
    0, \quad &\x = \p.
\end{cases}
\end{equation}
\end{definition}

This quantity measures the relative expansion or contraction induced by $\cT$ at $\x$ with respect to $\p$. Larger values of $\etap(\x,\cT)$ indicate greater local expansiveness and hence a higher risk of instability. The reference point $\p$ serves as a baseline reconstruction that the stabilized iterations seek to preserve.


We regard the system as locally unstable at $\x$ relative to $\p$ when $\etap(\x,\cT)$ becomes large. Accordingly, we view \emph{stabilization} as the task of controlling the stability index along the iterates. Thus we need to develop a tool that can reduce the $\etap$ whenever \emph{required}. Our approach achieves this by blending the black-box operator $\cT$ with a contractive anchor.

\paragraph{\normalfont\bfseries Adaptive anchoring.}
Let $\cS:\Re^n\to\Re^n$ denote the contractive anchor. Recall that $\cS$ is a $\kappa$-contraction with $\kappa<1$ if
\begin{equation}
\|\cS(\x)-\cS(\y)\|\le \kappa\|\x-\y\| \qquad (\x,\y \in \Re^n).
\end{equation}
A contraction has a unique fixed point, which we denote by $\p$. We then define the averaged operator
\begin{equation}
\label{eq:Ttheta-def}
\cT_\theta := (1-\theta)\cT + \theta\cS \qquad (\theta\in[0,1]).
\end{equation}
The following lemma shows that averaging $\cT$ with the contraction $\cS$ strictly reduces the stability index measured relative to $\p$.

\begin{lemma}
\label{lem:eta-reduces}
Let $\cS$ be a $\kappa$-contraction with unique fixed point $\p$, and define $\cT_\theta$  by~\eqref{eq:Ttheta-def}. For any $\x$ with $\etap\big(\x,\cT\big)>\kappa$ and any $\theta\in(0,1]$,
\begin{equation}
\label{eq:strict-eta-reduce}
\etap\big(\x,\cT_{\theta}\big)<\etap\big(\x,\cT\big).
\end{equation}
\end{lemma}


Lemma~\ref{lem:eta-reduces} provides a way of reducing $\etap$ of the iterates. The next question is how strongly we should reduce $\etap$, and how to choose the corresponding weight $\theta$. From~\eqref{eq:Ttheta-def}, taking $\theta$ too large diminishes the influence of the black-box operator $\cT$ and makes the iteration dominated by the contraction $\cS$, which can reduce reconstruction quality. In the extreme case $\theta=1$, the update reduces to $\cS$, and the iterates converge toward its fixed point $\p$, which often provides a weaker reconstruction than the peak attained by $\cT$.

To stabilize the iteration without compromising reconstruction quality, we seek a precise relation between the mixing weight $\theta$ and the resulting stability index $\etap(\x,\cT_\theta)$. The following theorem shows that, for any target level $\xi>\kappa$, the stability index can be reduced to at most $\xi$ by choosing $\theta$ above an explicit threshold. This threshold depends on the current point $\x$ and is therefore adaptive. It is also tight, since the stability index equals $\xi$ when $\theta$ is set at the threshold.
\begin{algorithm}[t]
\caption{Stabilization mechanism}
\label{algo:vista}
\begin{algorithmic}[1]
\REQUIRE Black-box operator $\cT$, contractive anchor $\cS$, initialization $\x_0$.
\STATE Obtain the fixed point $\p$ of $\cS$.
\FOR{$k=0,1,2,\ldots$}
    \STATE Compute $\eta_k \gets \etap(\x_k,\cT)$ using \Cref{def:eta-index}.
    \IF{$\eta_k>1$}
        \STATE Compute $\ttheta(\x_k,1)$ and set $\theta_k \gets \ttheta(\x_k,1)$ (\Cref{eq:quadratic})
    \ELSE
        \STATE Set $\theta_k \gets 0$.
    \ENDIF
    \STATE Update $\x_{k+1}\gets \cT_{\theta_k}(\x_k)$.
\ENDFOR
\end{algorithmic}
\end{algorithm}

\begin{theorem}
\label{thm:theta-threshold}
Let $\cT:\Re^n\to\Re^n$ be an arbitrary operator, and let $\cS:\Re^n\to\Re^n$ be a $\kappa$-contraction with fixed point $\p$. For any $\xi>\kappa$ and any $\x\in\Re^n$ with $\etap(\x,\cT)>\xi$, there exists a threshold $\ttheta=\ttheta(\x,\xi)$ with $0 < \ttheta < 1$ such that
\begin{equation*}
\etap(\x,\cT_\theta)\leqslant \xi
\end{equation*}
for every $\ttheta \leqslant \theta \leqslant 1$. Moreover, equality holds at $\theta=\ttheta$, i.e., $\etap (\x,\cT_{\ttheta})=\xi$.
\end{theorem}

The target level $\xi$ determines how strongly the stability index is controlled. A smaller value of $\xi$ generally requires a larger mixing weight. Since increasing $\theta$ in~\eqref{eq:Ttheta-def} reduces the contribution of the black-box operator $\cT$ and may lower reconstruction quality, the threshold $\ttheta(\x,\xi)$ is the natural choice: it is the smallest weight that guarantees
\begin{equation*}
\etap(\x,\cT_\theta)
\leqslant
\xi.
\end{equation*}

The question is whether controlling the stability index in this way is sufficient to prevent divergence. This is not immediate because the black-box operator $\cT$ is not assumed to be nonexpansive or even globally Lipschitz, so standard boundedness arguments do not apply. Nevertheless, the following corollary shows that eventually enforcing the target level $\xi=1$ guarantees bounded iterates and thereby prevents the divergent behavior observed in~\Cref{fig:inst_deep_den}.

\begin{corollary}
\label{cor:bdd-threshold}
Fix $N\geqslant 1$. Suppose that for $k\geqslant N$, the weight $\theta_k\in[0,1]$ satisfies
\begin{equation*}
\theta_k \in
\begin{cases}
[\ttheta(\x_k,1),1],
& \text{if } \etap(\x_k,\cT)>1,\\
0, & \text{otherwise}.
\end{cases}
\end{equation*}
Then the iterates $\{\x_k\}$ generated by $\x_{k+1} = \cT_{\theta_k}(\x_k)$ are bounded.
\end{corollary}
Thus, the target level $\xi=1$ provides a sufficient condition. 
Motivated by \Cref{cor:bdd-threshold}, we construct the stabilization mechanism in \Cref{algo:vista} by setting $\xi=1$ and choosing the smallest admissible weight. At iteration $k$, we compute $\ttheta(\x_k,1)$ by solving
\begin{equation}
\label{eq:quadratic}
\etap(\x_k,\cT_\theta)^2=1
\end{equation}
and selecting the smallest root in $[0,1]$. This yields a closed-form expression for $\ttheta$ and avoids iterative tuning of $\theta_k$. In practice, as shown in~\Cref{fig:theta-evolution}, the resulting weights remain small, so the updates are still driven primarily by $\cT$ while the stability index remains controlled.

\begin{figure}[t]
\captionsetup[subfloat]{labelformat=empty,labelsep=none}
\centering
\subfloat{\includegraphics[width=\RfourC]{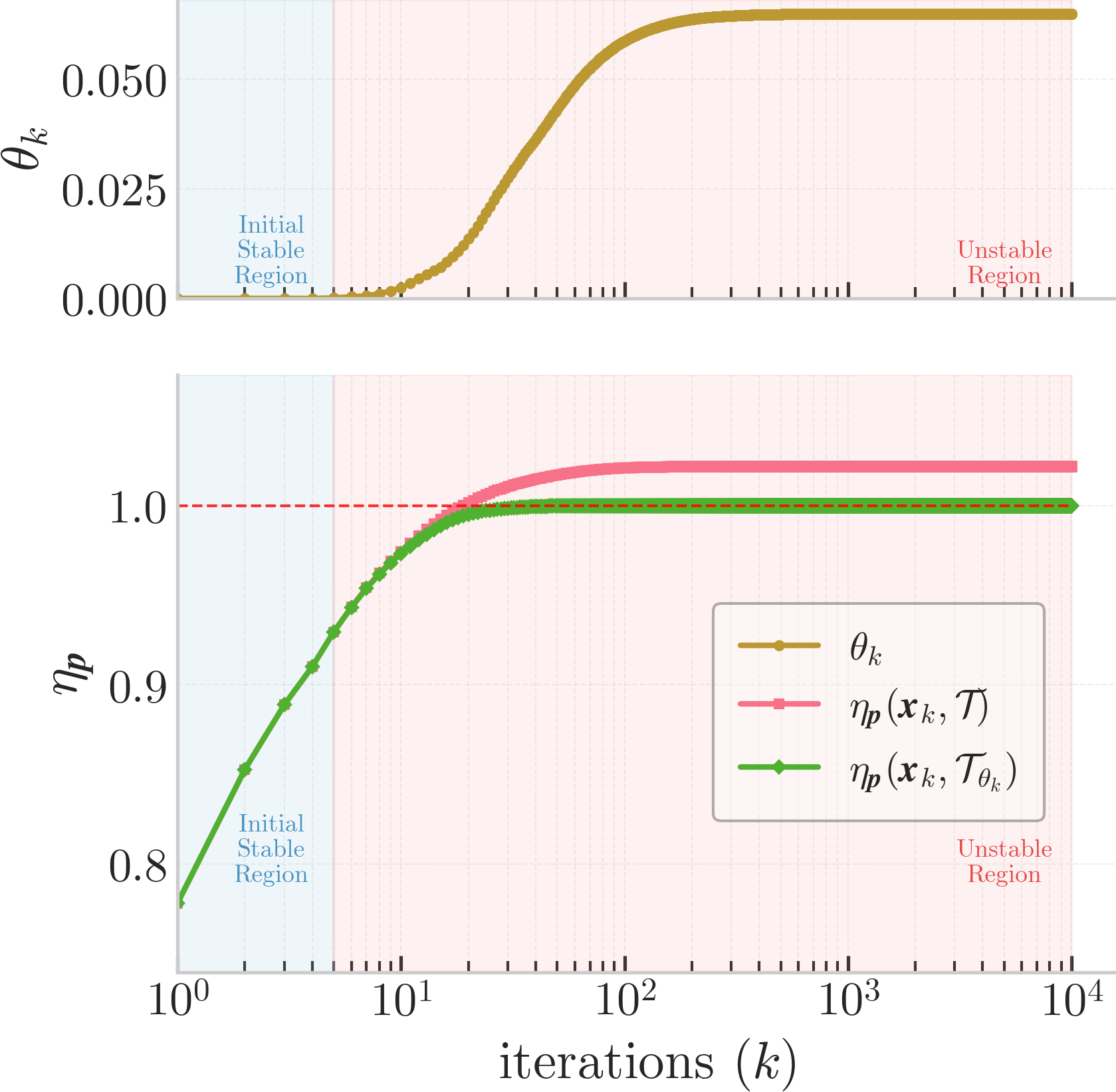}}\hfill
\subfloat{\includegraphics[width=\RfourC]{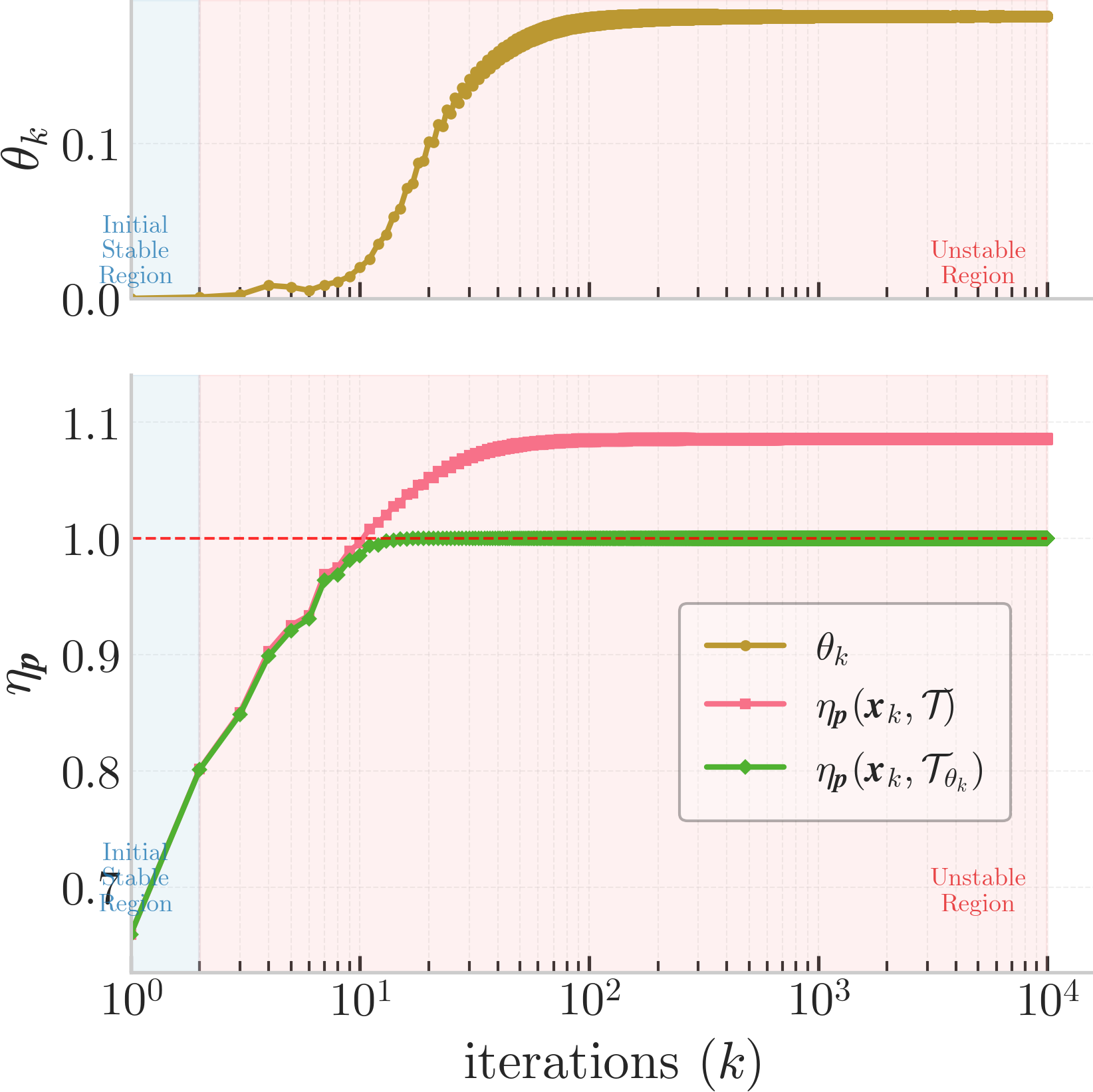}}\hfill
\subfloat{\includegraphics[width=\RfourC]{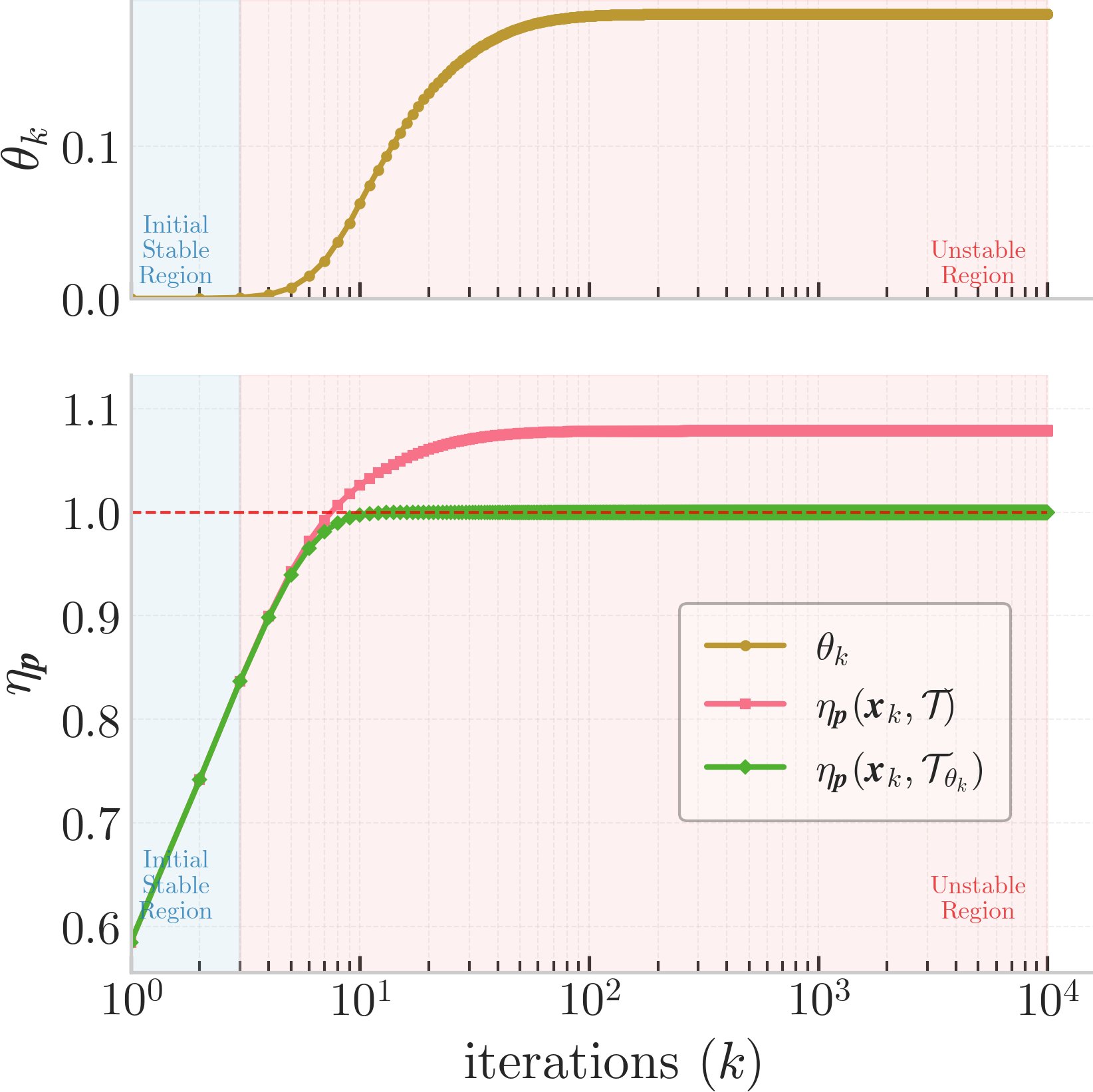}}\hfill
\subfloat{\includegraphics[width=\RfourC]{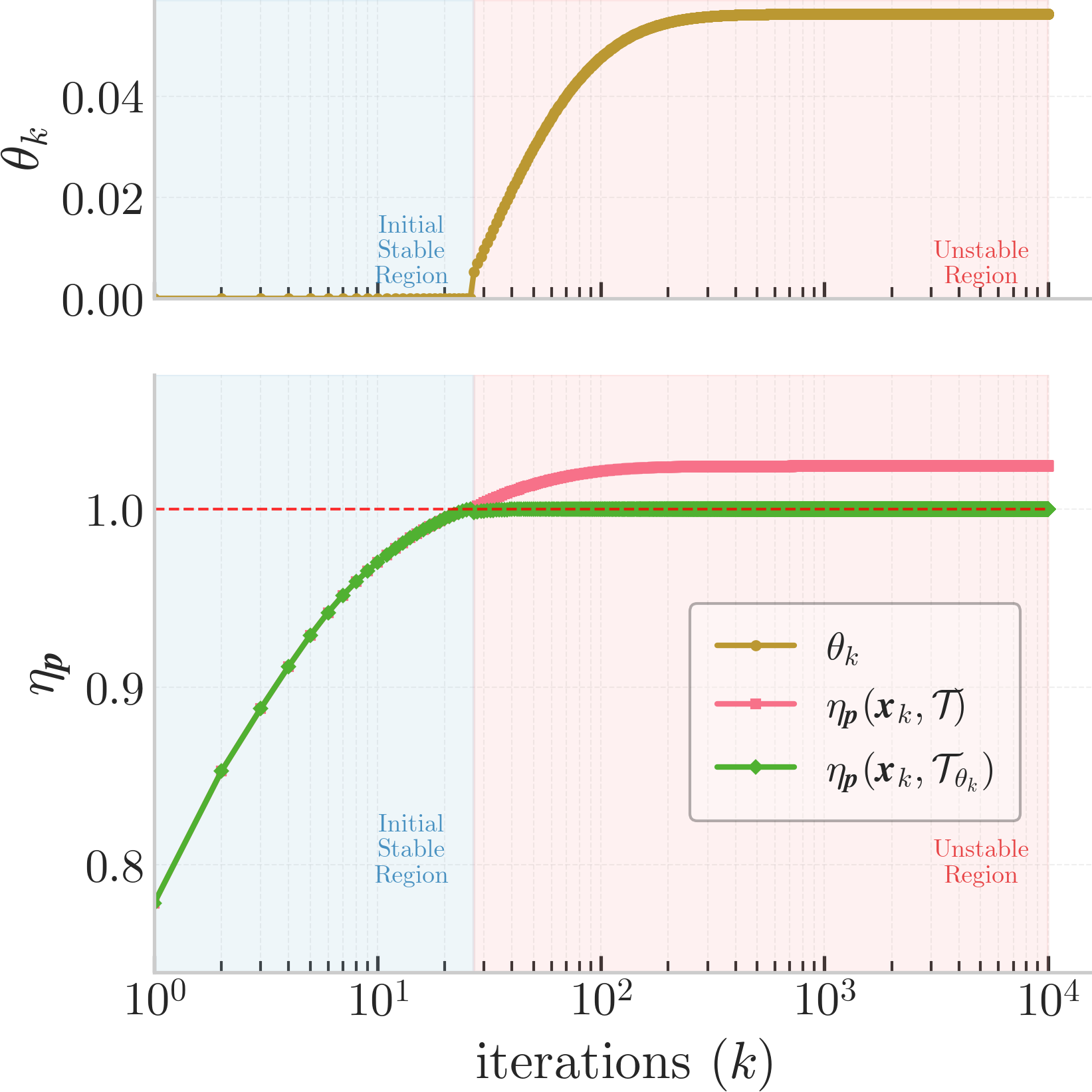}}\hfill
\subfloat{\includegraphics[width=\RfourC]{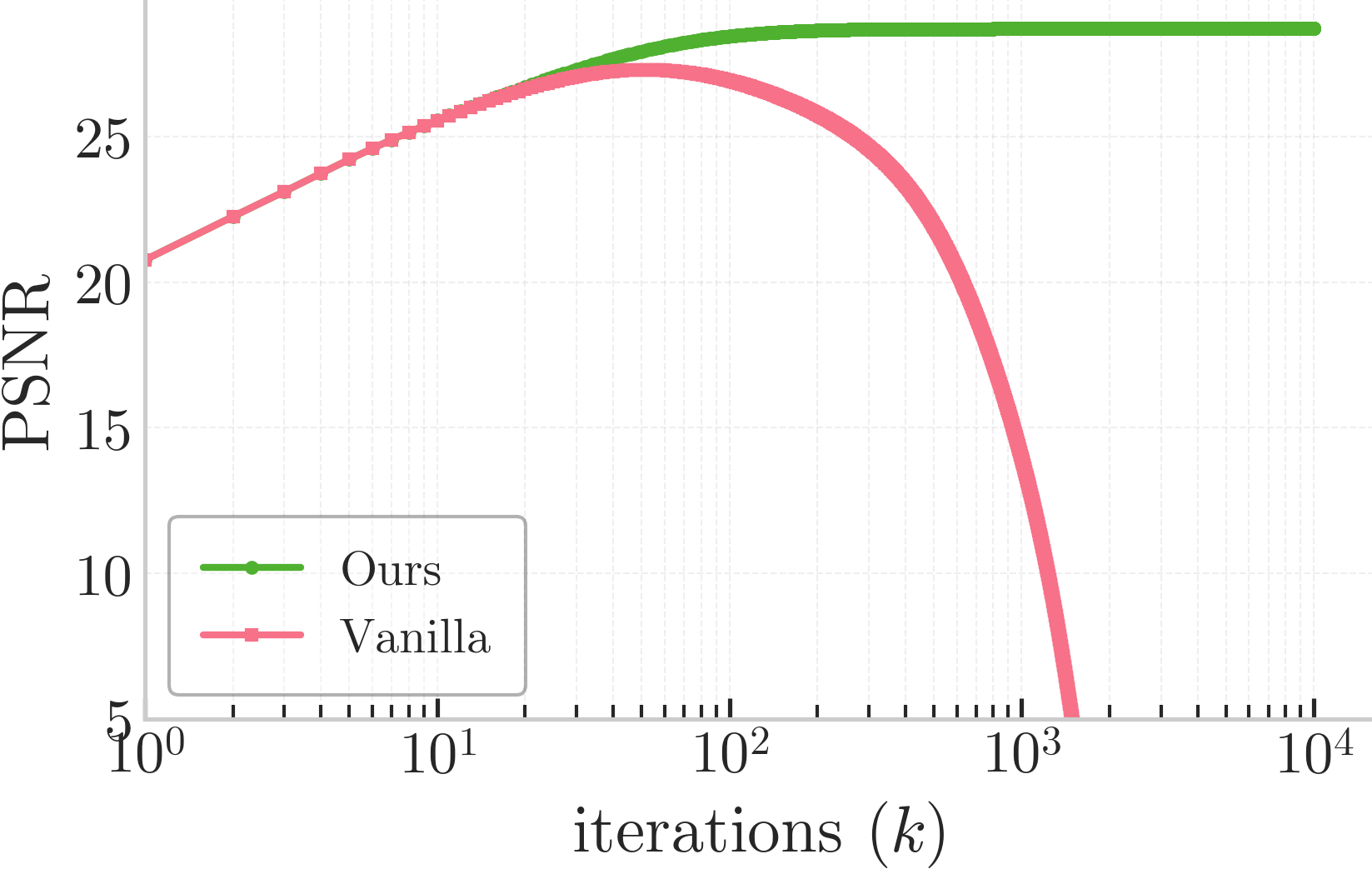}}\hfill
\subfloat{\includegraphics[width=\RfourC]{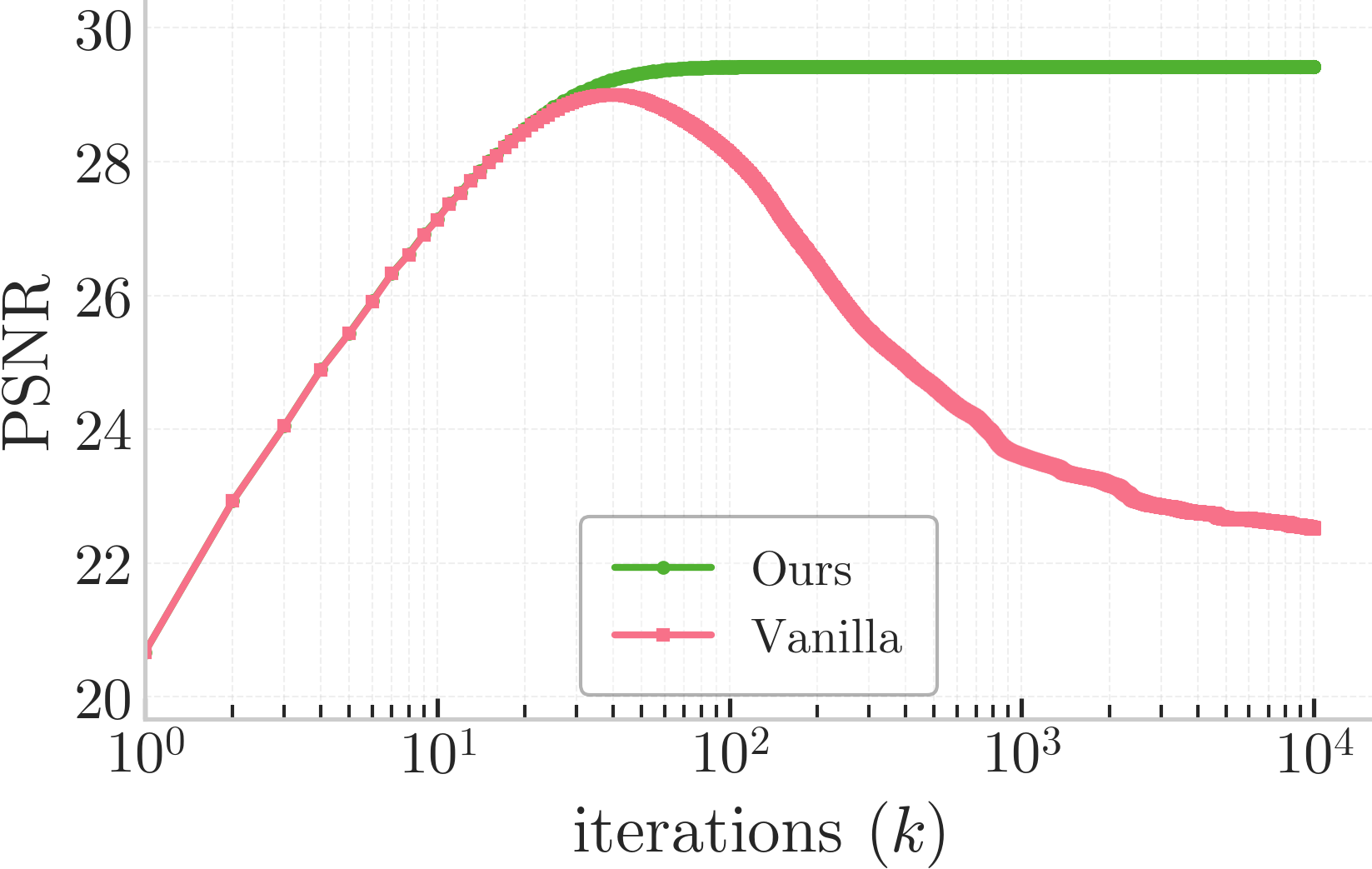}}\hfill
\subfloat{\includegraphics[width=\RfourC]{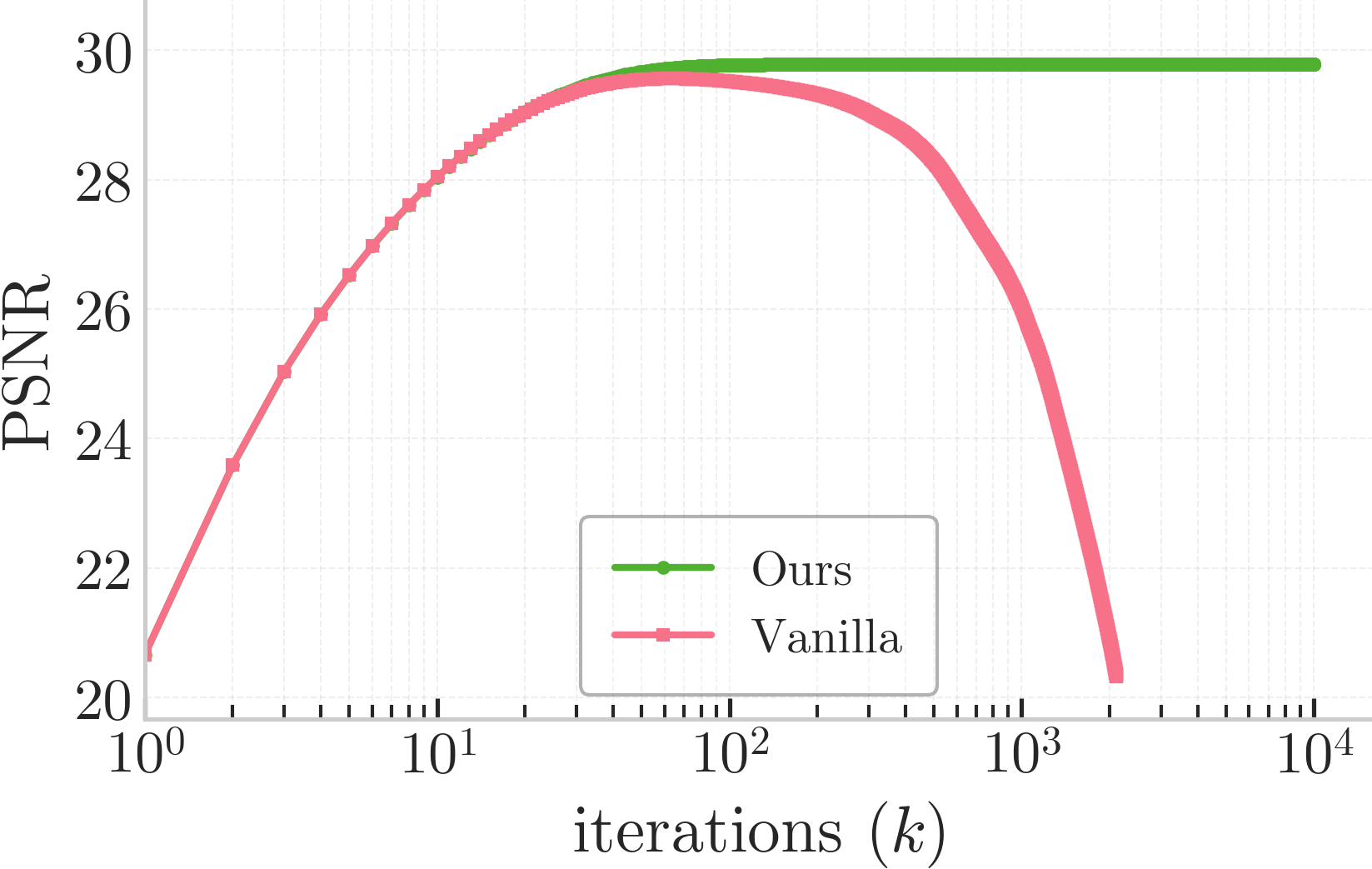}}\hfill
\subfloat{\includegraphics[width=\RfourC]{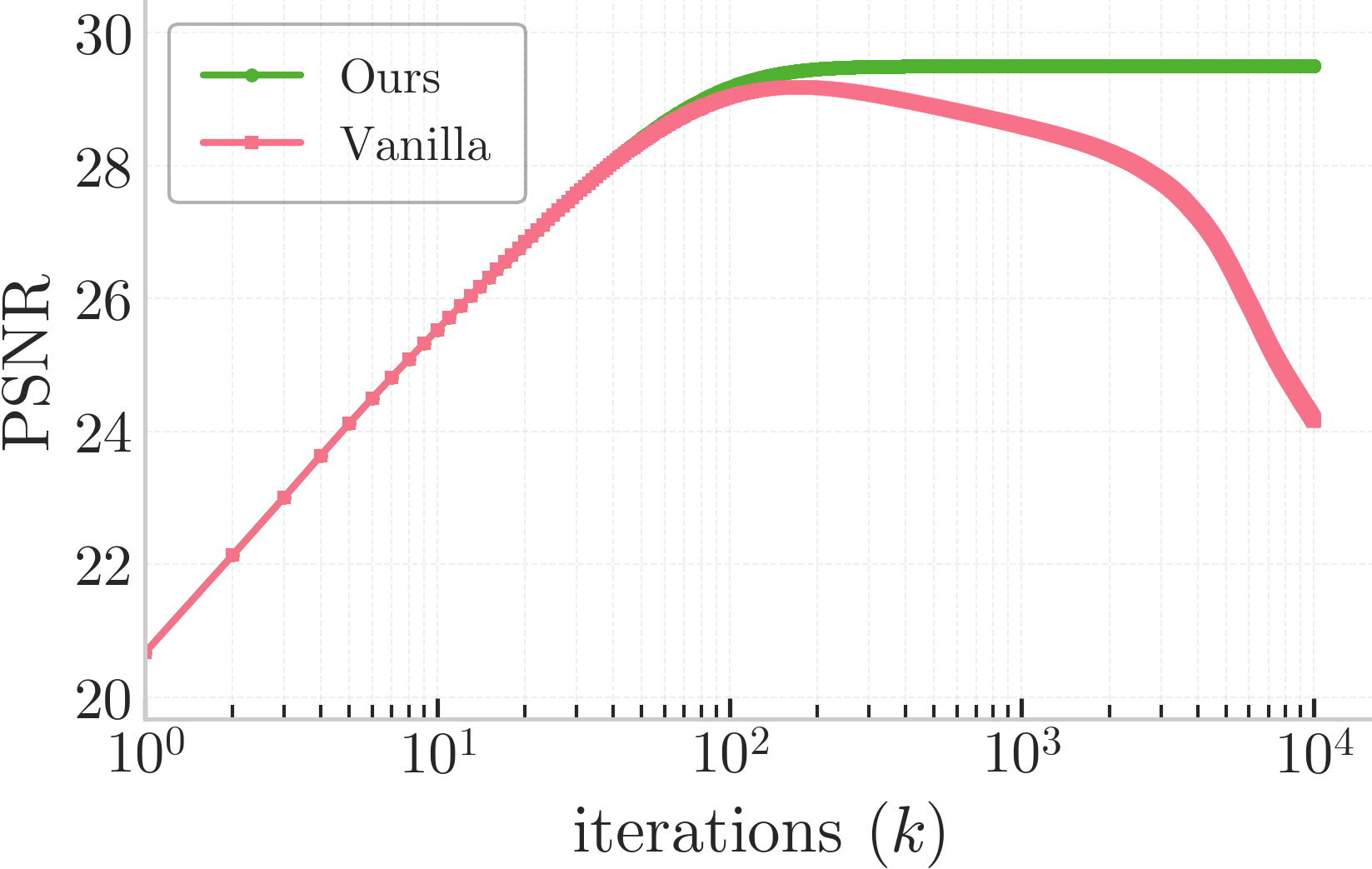}}\hfill
\subfloat[FBS+DnCNN]{\includegraphics[width=\RfourC]{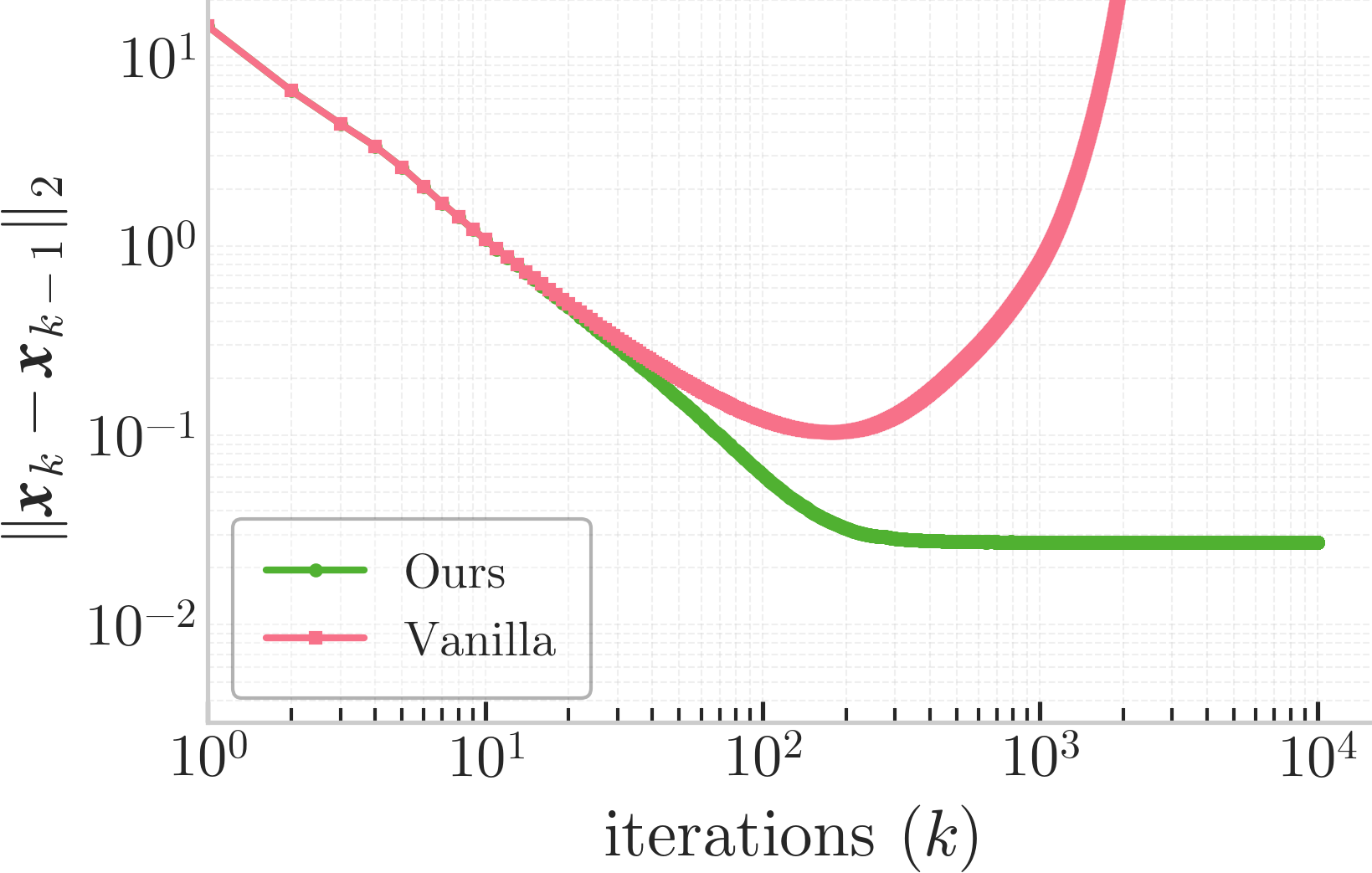}}\hfill
\subfloat[HQS+DiffUNet]{\includegraphics[width=\RfourC]{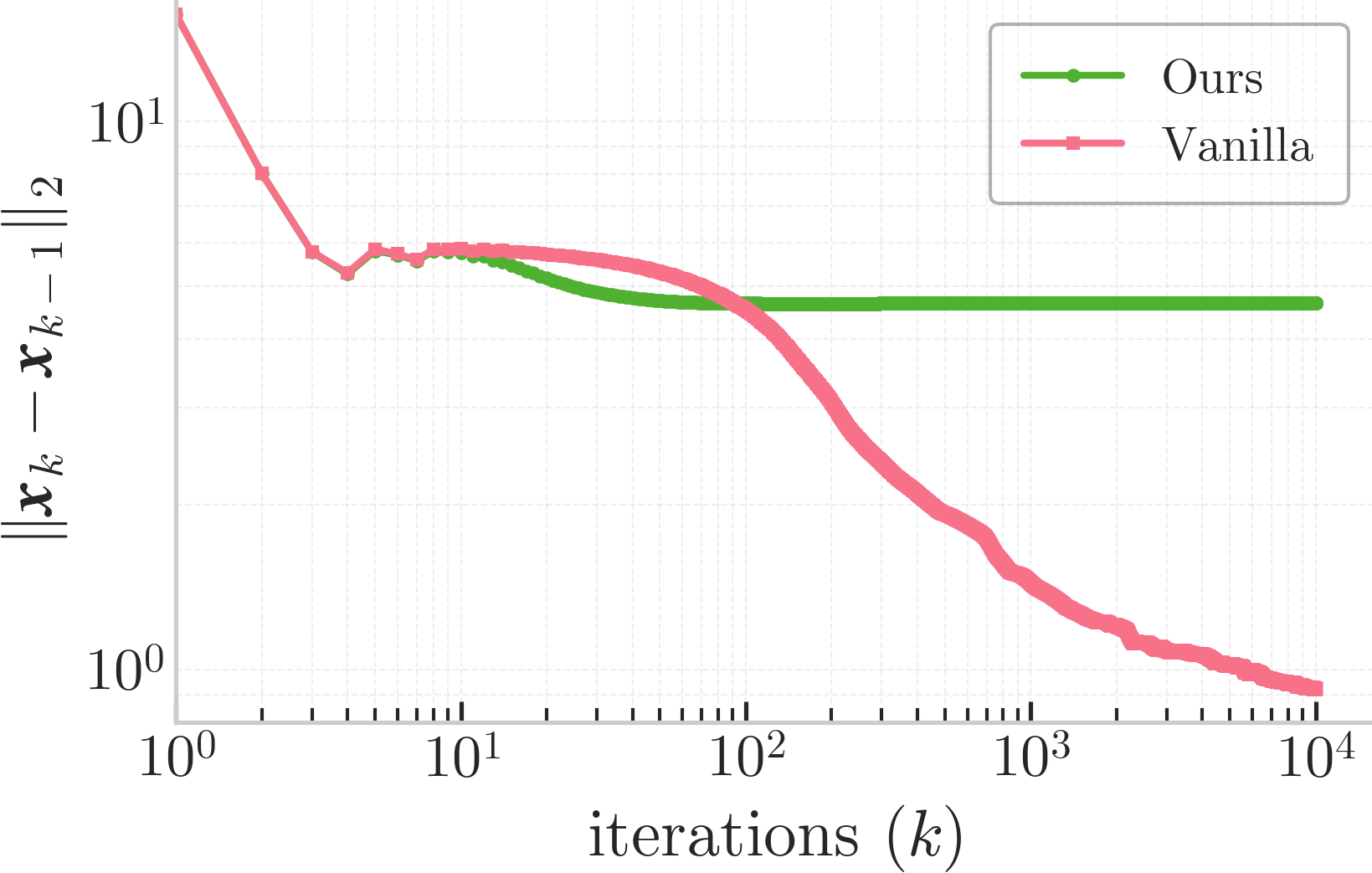}}\hfill
\subfloat[HQS+DRUNet]{\includegraphics[width=\RfourC]{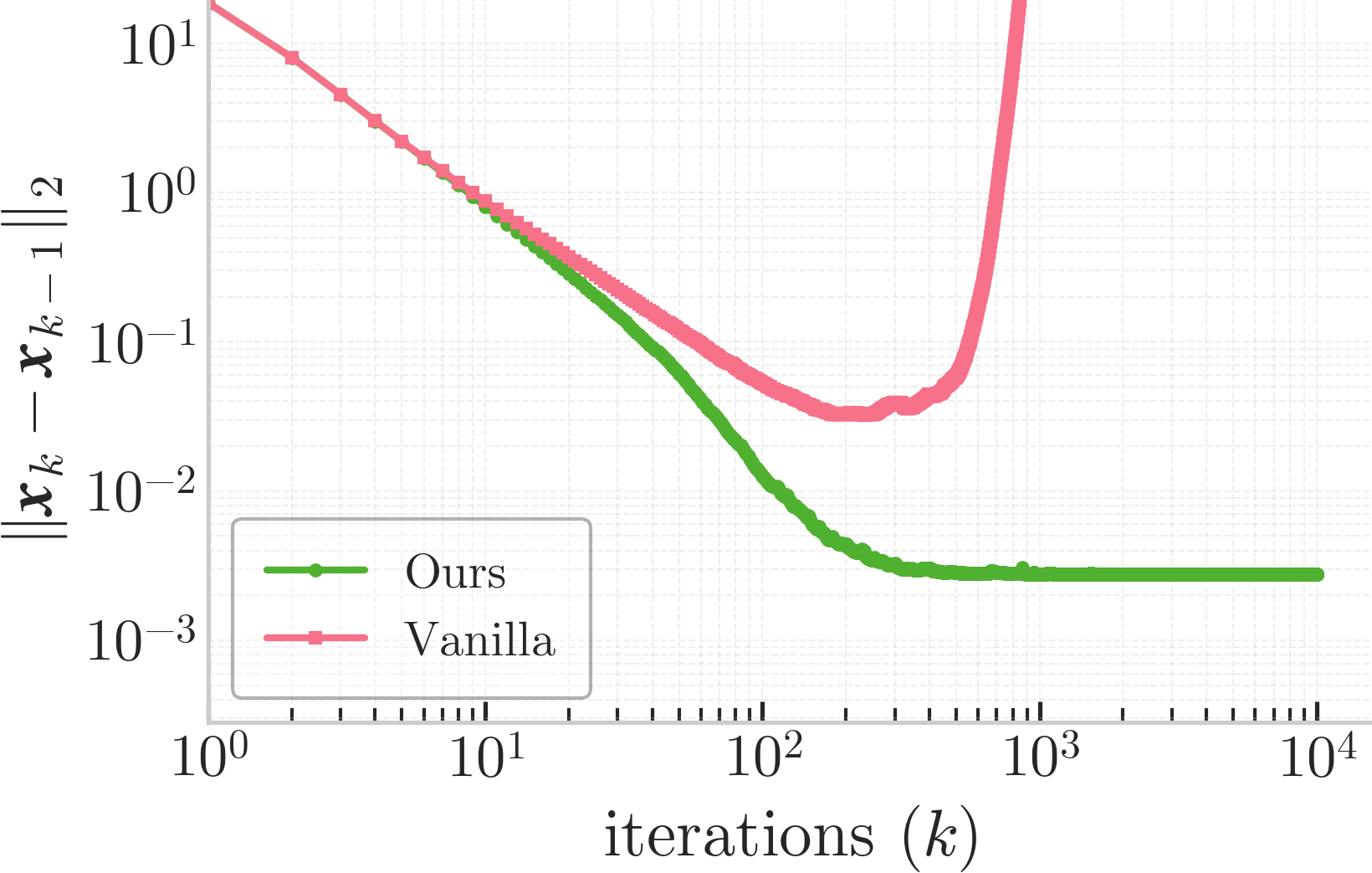}}\hfill
\subfloat[{RED-GD+GSDRUNet}]{\includegraphics[width=\RfourC]{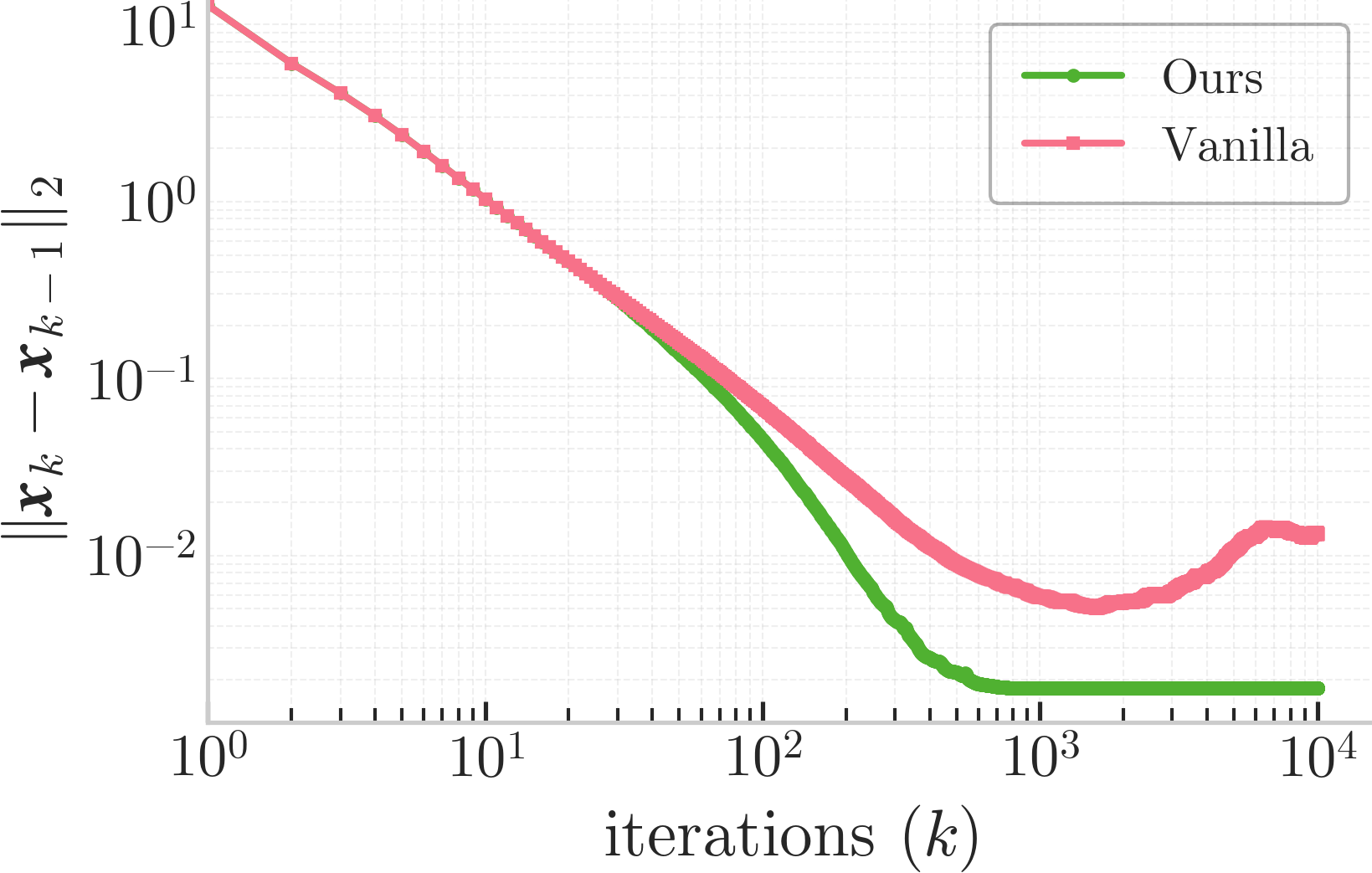}}\hfill
\caption{Evolution of the stability index $\etap$ and the weight $\theta_k$ obtained from \Cref{algo:vista} during deblurring. The first two rows show $\etap$ for the black-box operator $\cT$, $\etap$ for the stabilized operator $\cT_{\theta_k}$, and the corresponding weight $\theta_k$ over the iterations. The third and fourth rows report the PSNR and the difference norms, respectively. The curves are averaged over $9$ blur kernels. During the initial iterations, the stability index remains mostly below $1$, indicating stable behavior. When it exceeds $1$, $\theta_k$ becomes active and reduces $\etap$ of $\cT_{\theta_k}$ to maintain stability. Throughout the process, $\theta_k$ remains small, so the updates are still largely driven by the black-box operator $\cT$, thereby preserving reconstruction quality.}
\label{fig:theta-evolution}
\end{figure}

\section{Contractive Anchor}
\label{sec:conden}

In this section, we construct a contractive IR operator in the Euclidean norm and use it as the anchor $\cS$ in~\Cref{algo:vista}. The design is based on the following observation.
\begin{proposition}
\label{prop:ctr}
Let $f$ be convex and let $\cD$ be contractive. Then, for every $\rho>0$, the PnP-HQS operator $\Thqs$ defined in~\eqref{eq:pnphqs} is contractive.
\end{proposition}

This result shows that a contractive denoiser is sufficient to obtain a contractive reconstruction operator. Existing contractive denoisers often rely on patchwise processing or fixed transforms such as wavelets~\cite{nair2024averaged,nair_convergent_2024}. We instead develop a direct image-to-image model with trainable convolutional filters, avoiding costly patch aggregation and accommodating multiple noise levels. Our nonlinear parameterization is inspired by the spline-based convex regularizers in~\cite{goujon2023neural,goujon_learning_2024}. However, those models are designed for iterative denoising rather than as single-step contractive denoisers.

We model the denoiser as an operator $\cD:\bbR^n\to\bbR^n$ that remains contractive for any choice of convolutional filters. This allows the filters to be trained without explicit constraints while preserving contractivity both during training and when the denoiser is used within the reconstruction algorithm.

\paragraph{\normalfont\bfseries Architecture.}
The design is based on Lipschitz continuity and nonexpansiveness. An operator $\cT:\bbR^n\to\bbR^n$ is $\beta$-Lipschitz if
$\|\cT(\x_1)-\cT(\x_2)\| \leqslant \beta \|\x_1-\x_2\|$
for all $\x_1,\x_2\in\bbR^n$. It is said to be \emph{nonexpansive} when $\beta=1$.

Building on~\cite{cohen2021has,hurault2022gradient,goujon2023neural} and a classical result on the contractivity of gradient-step operators~\cite{nesterov_lectures_2018}, we define the denoiser as
\begin{equation}
\label{eq:gradstep_model}
\cD(\x) = \x - \gamma \nabla \! \Poten(\x),
\end{equation}
where $\Poten: \bbR^n \to \bbR$ is a trainable potential function and $\gamma >0$ is the step size. 

A standard result states that the gradient-step operator $\cD$ is contractive when $\Poten$ is smooth and strongly convex, provided the step size $\gamma$ is chosen sufficiently small~\cite{nesterov_lectures_2018}. Here, smoothness means that $\Poten$ is differentiable and $\nabla\Poten$ is Lipschitz continuous. We therefore seek a smooth, strongly convex potential whose gradient can be implemented by a one-layer convolutional network. Specifically, let $\W: \bbR^n \to \bbR^p$ be a convolution operator with $p$ filters, and define
\begin{equation}
\label{eq:potential}
    \Poten(\x)= \sum_{j=1}^{p} \qspline_j \big( (\W \x)_j \big)+ \frac{\tau}{2}\Vert\x\Vert^2,
\end{equation}
where $\tau > 0$ is a trainable parameter and each $\qspline_j$ is differentiable with derivative $\lspline_j = \qspline_j'$. The gradient of~\eqref{eq:potential} is 
\begin{equation}
\label{eq:gradmodel}
    \nabla \! \Poten(\x) = \W^\top \! \Lspline(\W \x) + \tau \, \x,
\end{equation}
where the activation $\Lspline: \bbR^p \to \bbR^p$ acts componentwise as
\begin{equation*}
\Lspline(\z) = \nabla\Qspline(z) = \big(\lspline_1(z_1), \ldots,\lspline_p(z_p) \big)  \qquad (\z \in \bbR^p).
\end{equation*}
Substituting~\eqref{eq:gradmodel} into~\eqref{eq:gradstep_model} gives
\begin{equation}
\label{eq:D}
\cD(\x) = (1- \gamma \tau)\, \x -  \gamma \W^\top \! \Lspline(\W \x).
\end{equation}

Since both $\W$ and its adjoint $\W^\top$ are convolutional operators, $\cD$ has the desired one-layer convolutional structure. It remains to impose conditions that guarantee smoothness, strong convexity, and hence contractivity. We use the following structural requirements:
\begin{enumerate}
\item We enforce nonexpansiveness of $\W$ using the reparameterization proposed in~\cite[Equation~4]{araujo_unified_2023}:
\begin{equation}
\W=\widetilde{\W}\R^{-1/2},
\qquad
\R
=
\operatorname{Diag} 
(\operatorname{Diag}(\q)^{-1}
|\widetilde{\W}^{\top}\widetilde{\W} |
\q),
\end{equation}
where $\q\in\Re^n$ is trainable and constrained to have strictly positive entries through through element wise exponentiation. For any vector $\v$, $\operatorname{Diag}(\v)$ denotes the diagonal matrix with diagonal $\v$ and zeros elsewhere. This construction guarantees that $\norm{\W}\leqslant 1$. We parameterize $\widetilde{\W}$ as a trainable $3\times3$ convolution with $64$ output channels.

\item  We parameterize each nonlinearity $\lspline_j$ as a trainable linear spline with slopes constrained to $[0,1]$, using $101$ equally spaced knots with spacing $0.02$. Following~\cite{goujon2023neural}, we construct $\lspline_j$ as the derivative of a convex function $\qspline_j$, that is, $\lspline_j=\qspline_j'$. This spline parameterization was found to be more expressive than a standard ReLU activation.
\end{enumerate}
With the above construction in place, we have the following result.
\begin{proposition}
\label{prop:ccd}
Let $\tau>0$ and $0 < \gamma \leqslant 2/(1+2\tau)$. Then the operator $\cD$ defined in~\eqref{eq:D} is a $\kappa$-contraction, where
\begin{equation*}
\kappa =
\left(
1-\frac{2\gamma\tau(1+\tau)}{1+2\tau}
\right)^{1/2}.
\end{equation*}
\end{proposition}

We fix the step size as $\gamma=1/(1+2\tau)$, for which the contraction factor becomes 
\begin{equation*}
\kappa
=
\frac{\sqrt{1+2\tau+2\tau^2}}{1+2\tau}.
\end{equation*}
We constrain $\tau\in[0.0102,0.135]$, corresponding to $\kappa\in[0.9,0.99]$. This range prevents the contraction from being either too strong or too weak, while retaining sufficient model expressivity.

\input{arch}

To handle multiple noise levels $\sigma_{\n}$ within a single model~\cite{zhang2021plug}, we incorporate noise conditioning through a trainable scaling function $\ScalFunc(\sigma_{\n})$~\cite{goujon_learning_2024}; details are provided in the Appendix E. The resulting noise-aware denoiser $\Dsig:\bbR^n\to\bbR^n$ is defined as
\begin{equation}
    \label{eq:Dsig}
    \Dsig(\x) = (1- \gamma \tau)\, \x -  \gamma \W^\top\! \left( \ScalFunc(\sigma)^{-1}\Lspline \left( \ScalFunc(\sigma)\W \x \right) \right), 
\end{equation}
where $\x=\bar{\x}+\n$ denotes the noisy input and $\sigma_{\n}$ is the corresponding noise level. The resulting architecture is shown in~\Cref{fig:arch}.

\begin{proposition}
\label{prop:dctr}
Let $\tau>0$ and $0 < \gamma \leqslant 2/(1+2\tau)$. Then, for every fixed $\sigma>0$, the noise-aware denoiser $\Dsig$ defined in~\eqref{eq:Dsig} is a $\kappa$-contraction, with the same contraction factor $\kappa$ as in~\Cref{prop:ccd}.
\end{proposition}

Based on the preceding construction, we define the IR anchor as
\begin{equation}
\label{eq:sctr}
\Sctr := \Dsig\circ \prox_{\rho f}.
\end{equation}
By \Cref{prop:dctr}, $\Dsig$ is contractive. Hence, \Cref{prop:ctr} implies that the IR operator $\Sctr$ is also contractive.

\paragraph{\normalfont\bfseries Training.}
We train $\Dsig$ as a Gaussian denoiser at noise levels $\sigma_{\n} \sim \cU[0,25/255]$ using standard training procedures~\cite{zhang2021plug,goujon_learning_2024}. The training set from~\cite{zhang2021plug} comprises BSD400~\cite{bsd500}, DIV2K~\cite{radu2017ntire}, the Waterloo Exploration Database~\cite{ma2017waterloo}, and Flickr2K~\cite{zhang2023ntire}. We use BSD32 for validation and BSD68 for testing.

We minimize the mean squared error using Adam~\cite{kingma2017adam} for $100{,}000$ epochs, with a batch size of $16$ randomly sampled $128\times128$ patches. The learning rate follows a cosine schedule and decays to $5\times10^{-3}$ of its initial value. Following~\cite{goujon_learning_2024}, we assign separate initial learning rates to different parameter groups. We use $5\times10^{-3}$ for $(\widetilde{\W},\q)$, $5\times10^{-4}$ for $\Lspline$, $5\times10^{-3}$ for $\ScalFunc$, and $5\times10^{-2}$ for $\tau$.

For completeness, the Appendix G reports the denoising performance of $\Dsig$ and the reconstruction performance of the corresponding IR operator $\Sctr$ in~\eqref{eq:sctr}. In~\Cref{fig:p_ablation}, we study how the quality of the anchor fixed point $\p$ affects stabilization by varying the number of epochs used to train $\Dsig$.

\begin{figure}[t]
    \centering
    \includegraphics[width=1.0\linewidth]{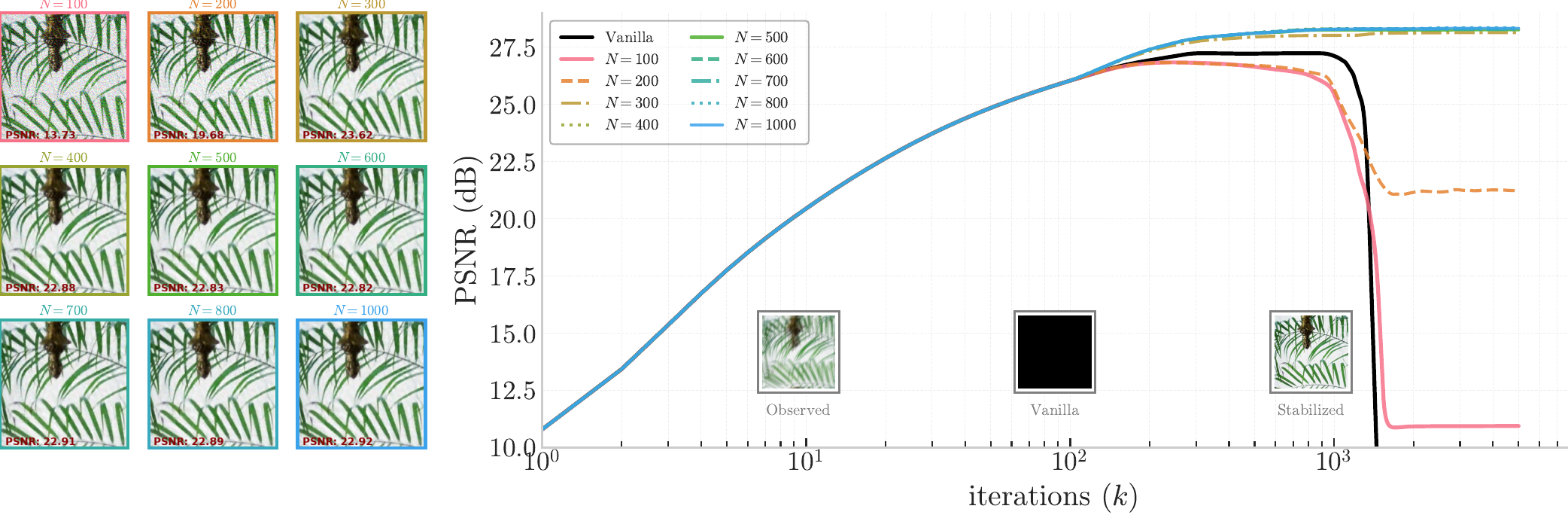}
    \caption{Experiment showing the effect of the reference point $\p$ on stabilization. We stabilize a divergent IR process with peak-and-collapse using the anchor $\cS$ in~\eqref{eq:sctr}. Different fixed points $\p$ are obtained from anchor weights at different training stages of the denoiser $\Dsig$ in~\eqref{eq:Dsig} (left). Early checkpoints ($N=100,200$) produce poor baselines, and the corresponding $\p$ does not prevent collapse. Once $\p$ reaches a reasonable baseline quality, already by $N=300$, stabilization succeeds. This shows that a reasonably accurate anchor is sufficient for reliable stabilization.}
    \label{fig:p_ablation}
\end{figure}

\section{Experiments}
\label{sec:exp}
\begin{figure}[t]
\centering
\captionsetup[subfloat]{labelformat=empty,labelsep=none}
\subfloat[ADMM+DiffUNet]{\includegraphics[width=\Rfour]{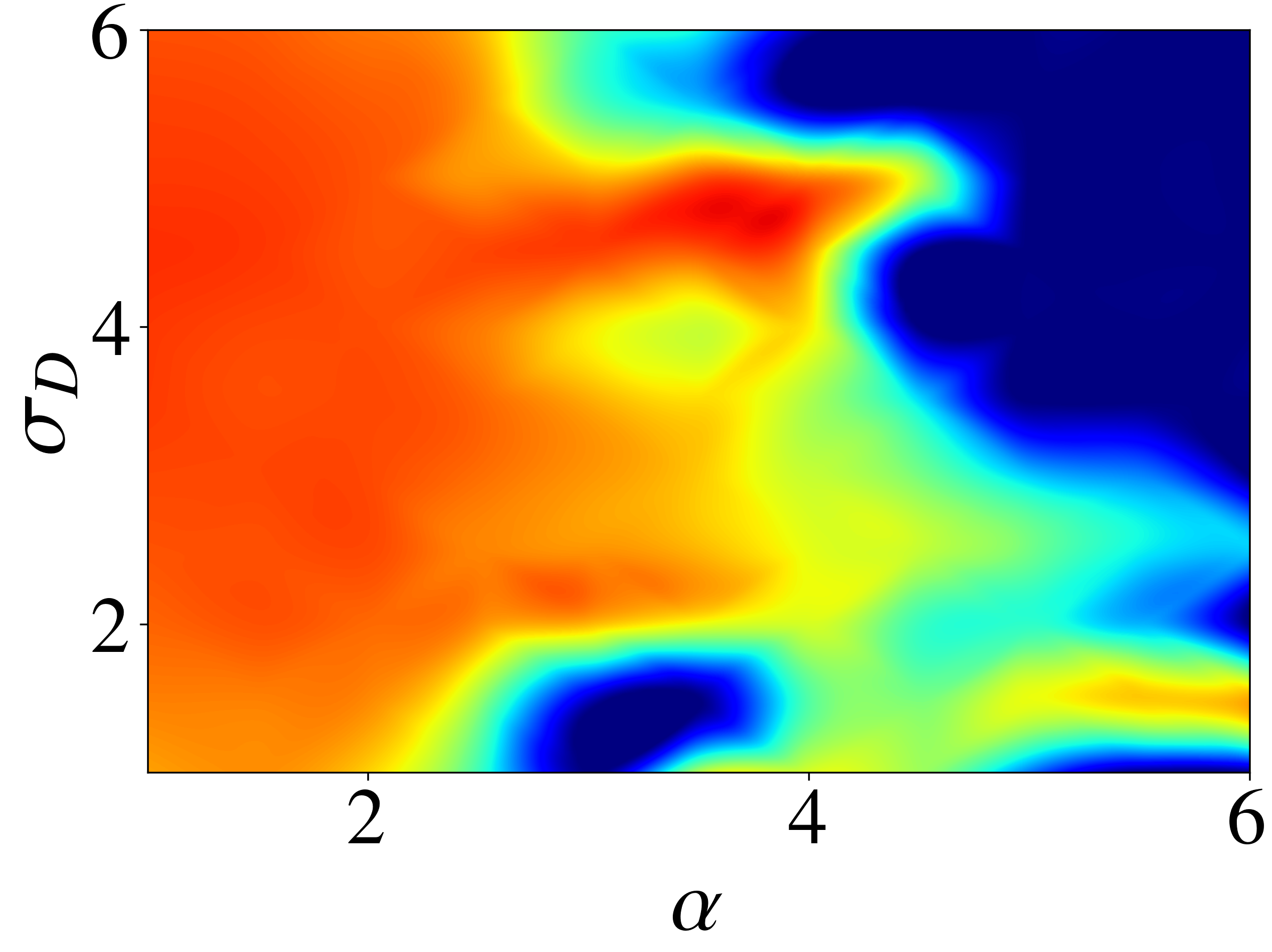}}\hfill
\subfloat[PGD+DRUNet]{\includegraphics[width=\Rfour]{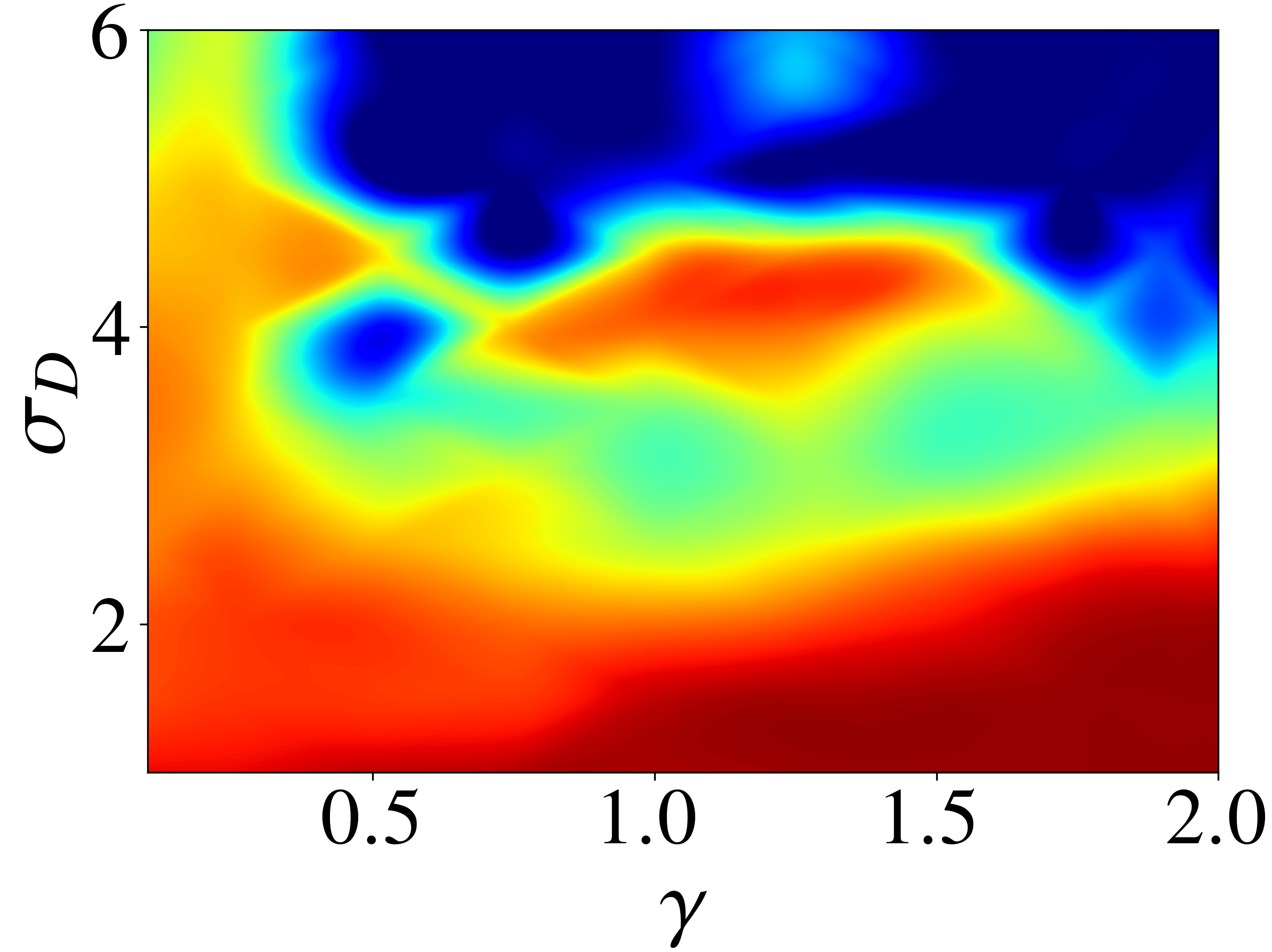}}\hfill
\subfloat[HQS+DRUNet]{\includegraphics[width=\Rfour]{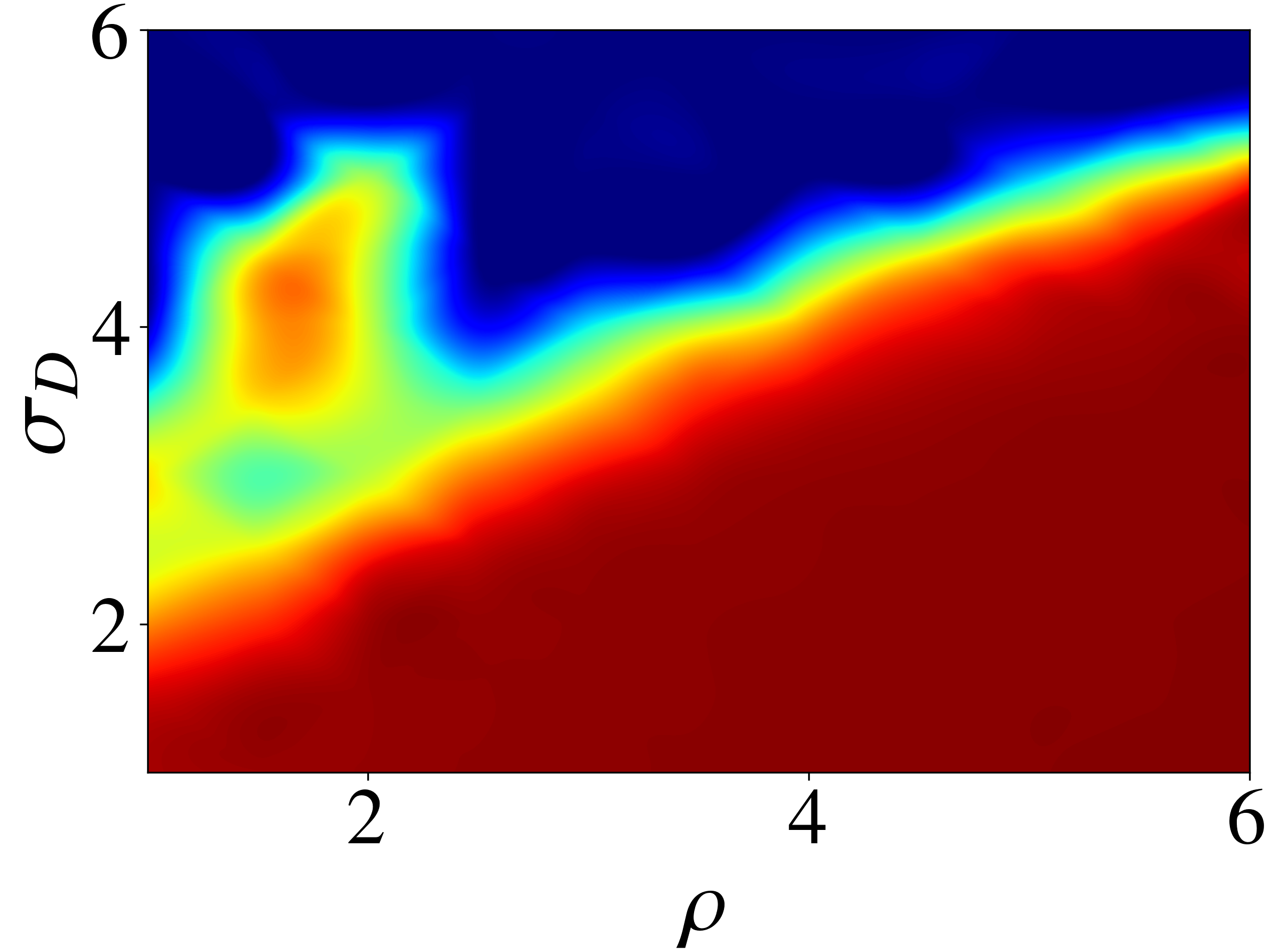}}
\subfloat[HQS+DnCNN]{\includegraphics[width=\Rfour]{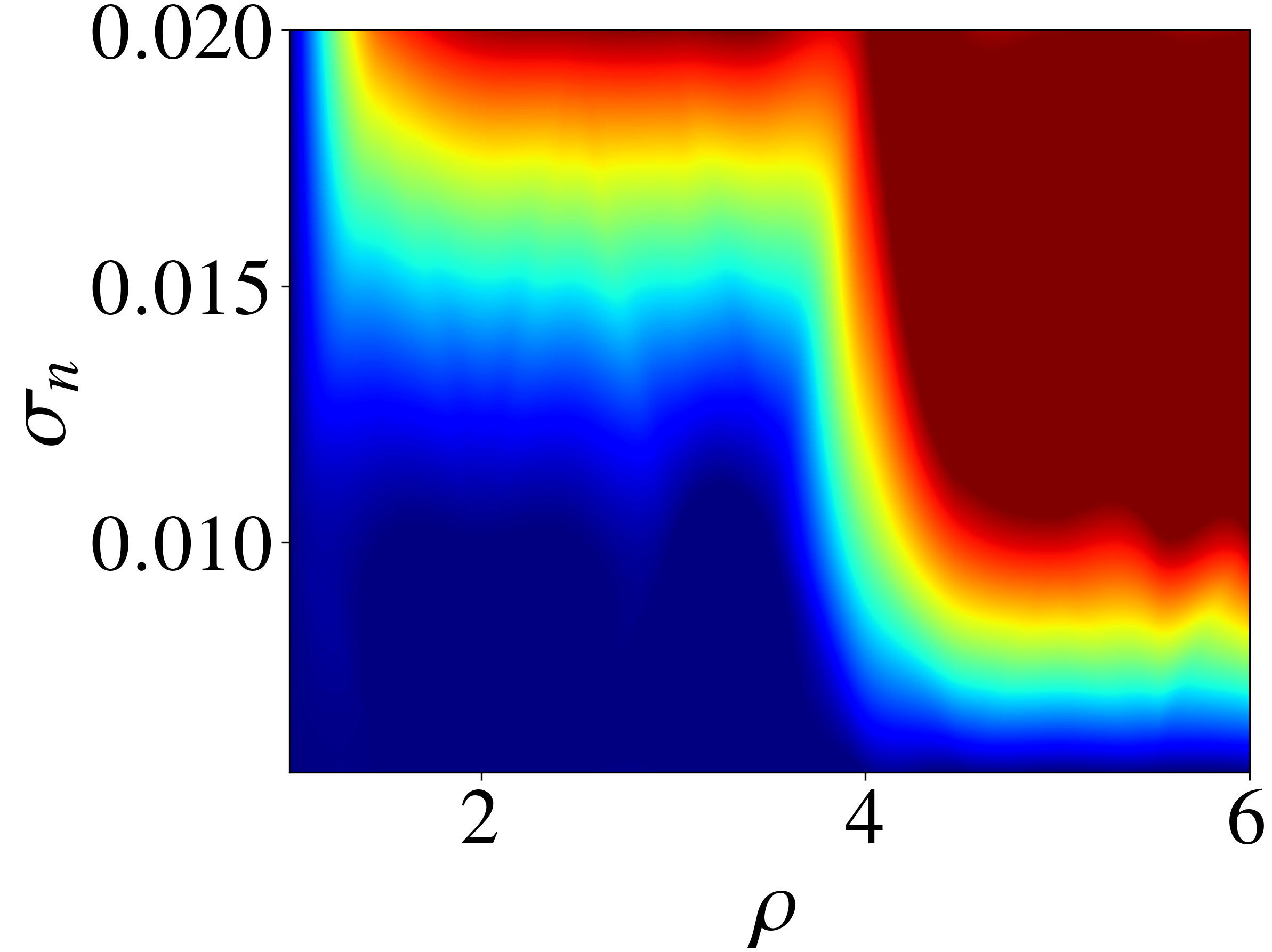
}}\hfill
\subfloat[ADMM+GSDRUNet]{\includegraphics[width=\Rfour]{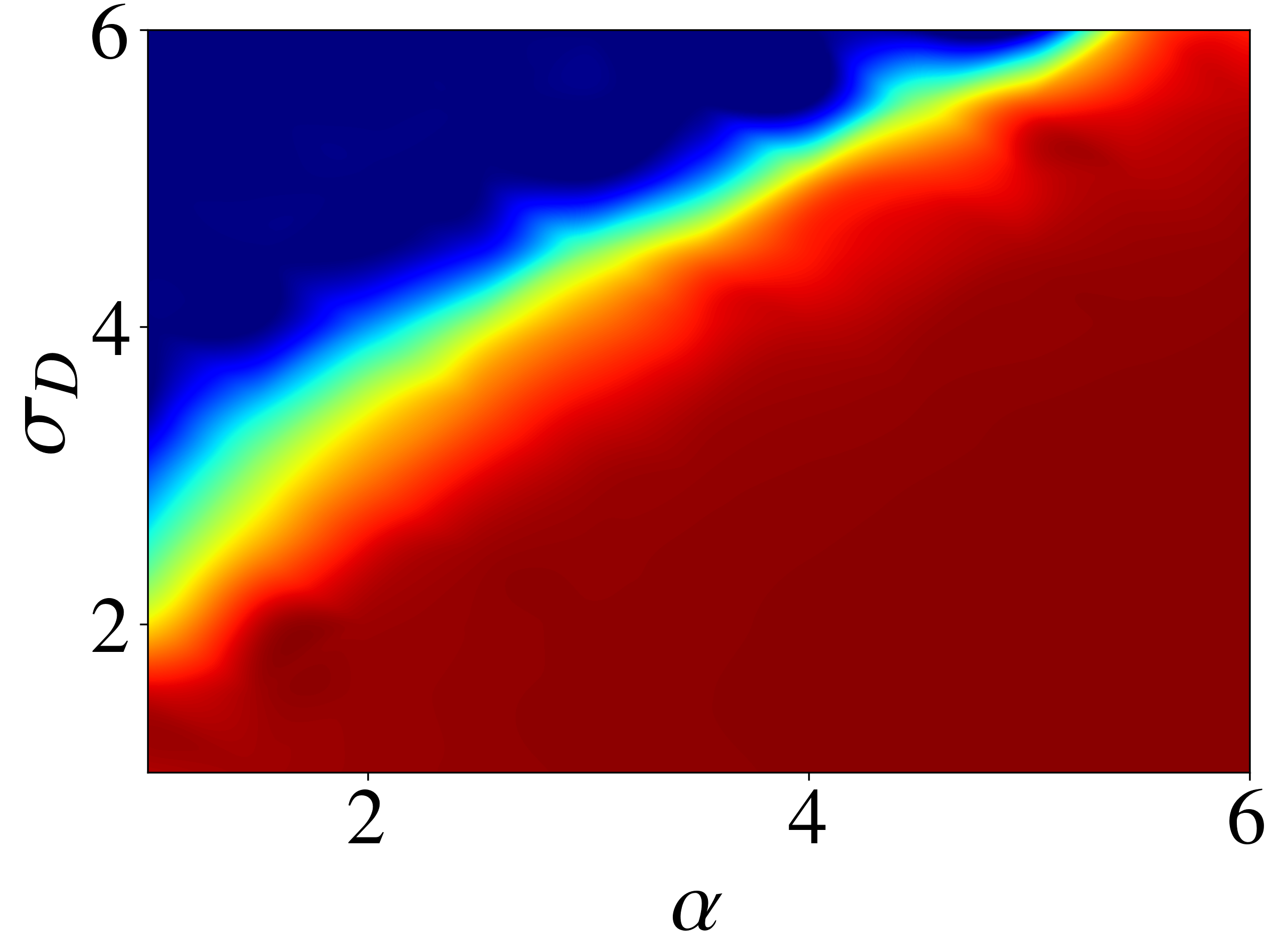}}
\subfloat[RED+GSDRUNet]{\includegraphics[width=\Rfour]{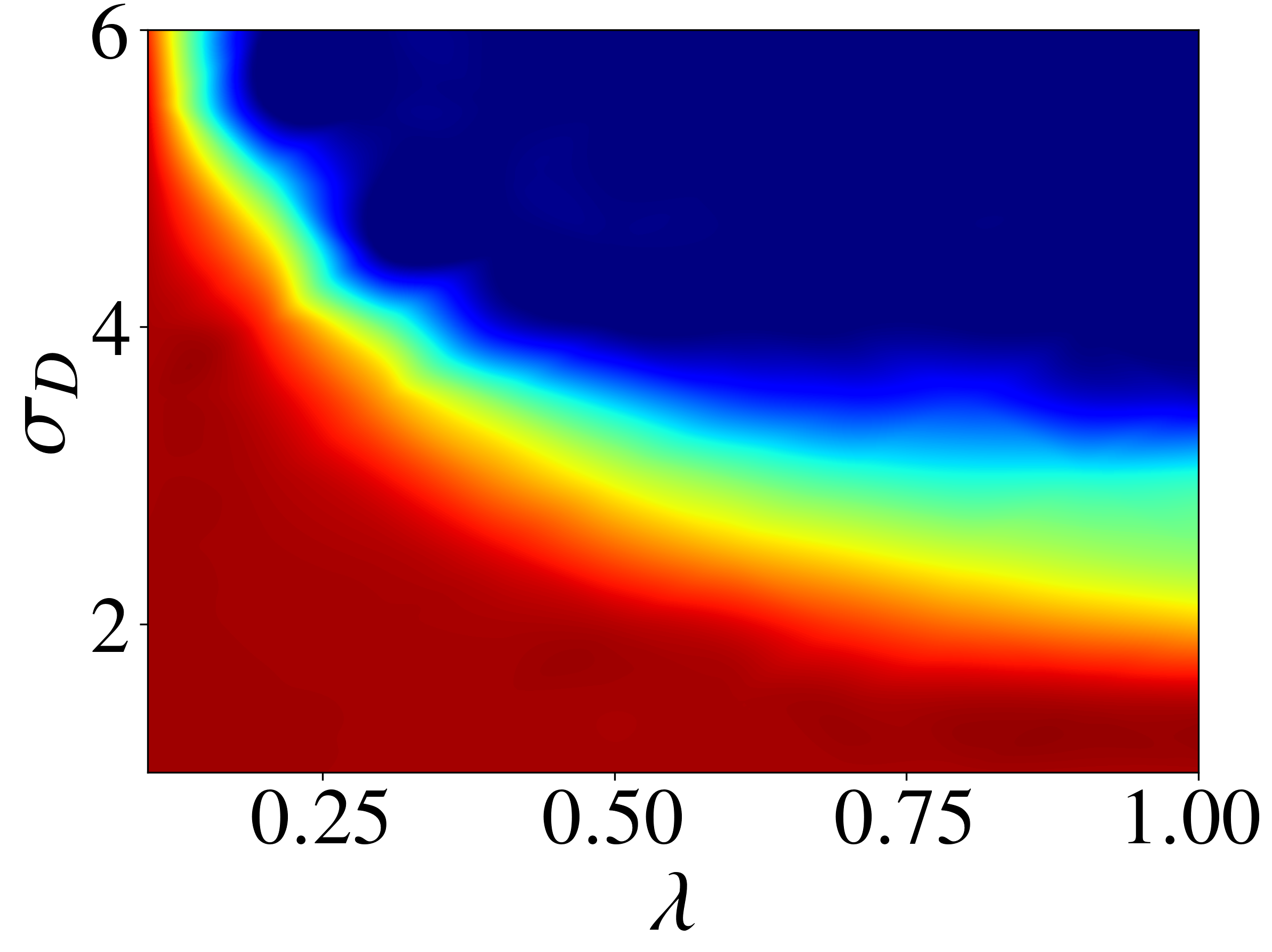}}
\subfloat[RED+DRUNet]{\includegraphics[width=\Rfour]{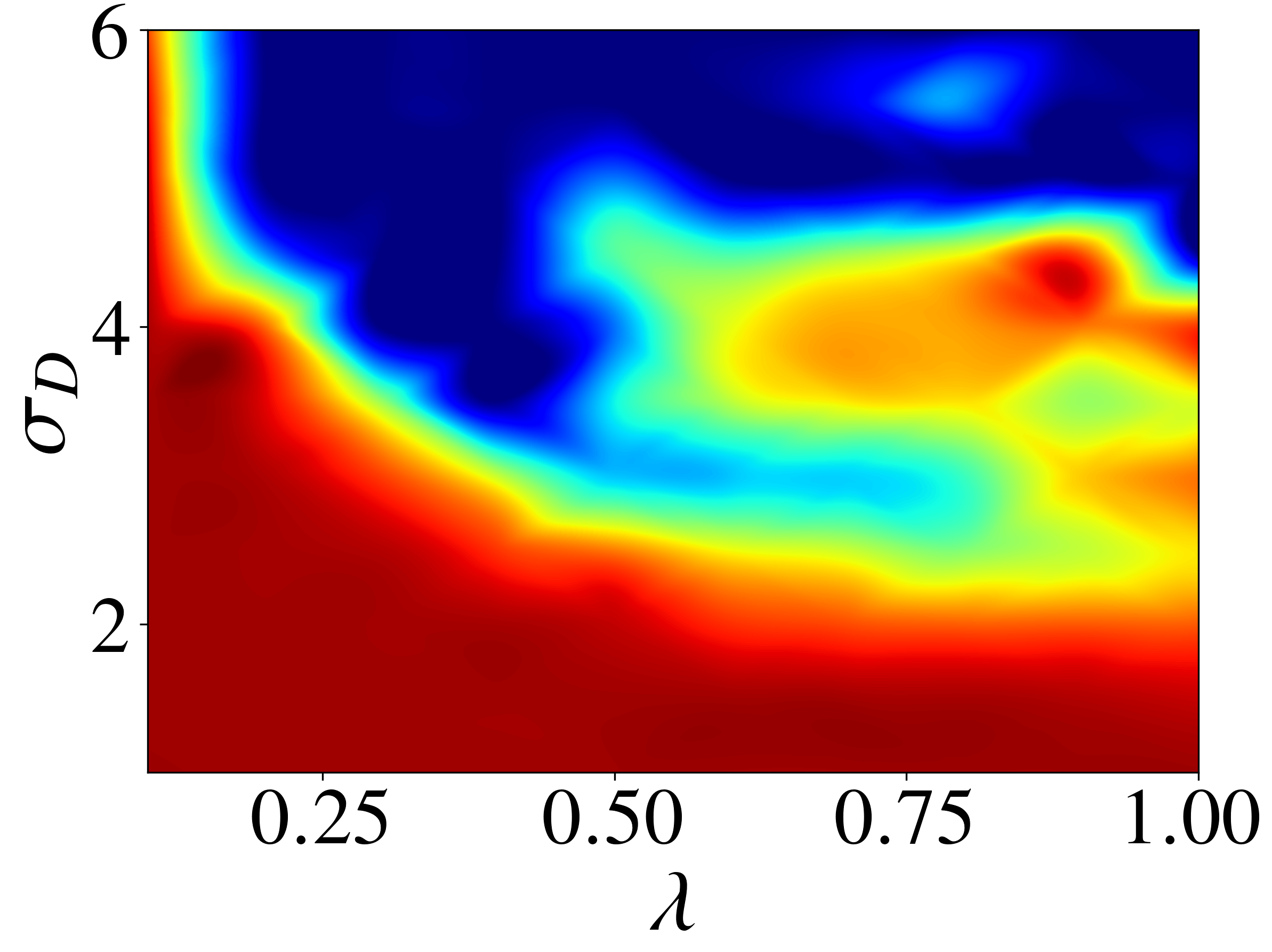}}
\subfloat[PGD+MMO]{\includegraphics[width=\Rfour]{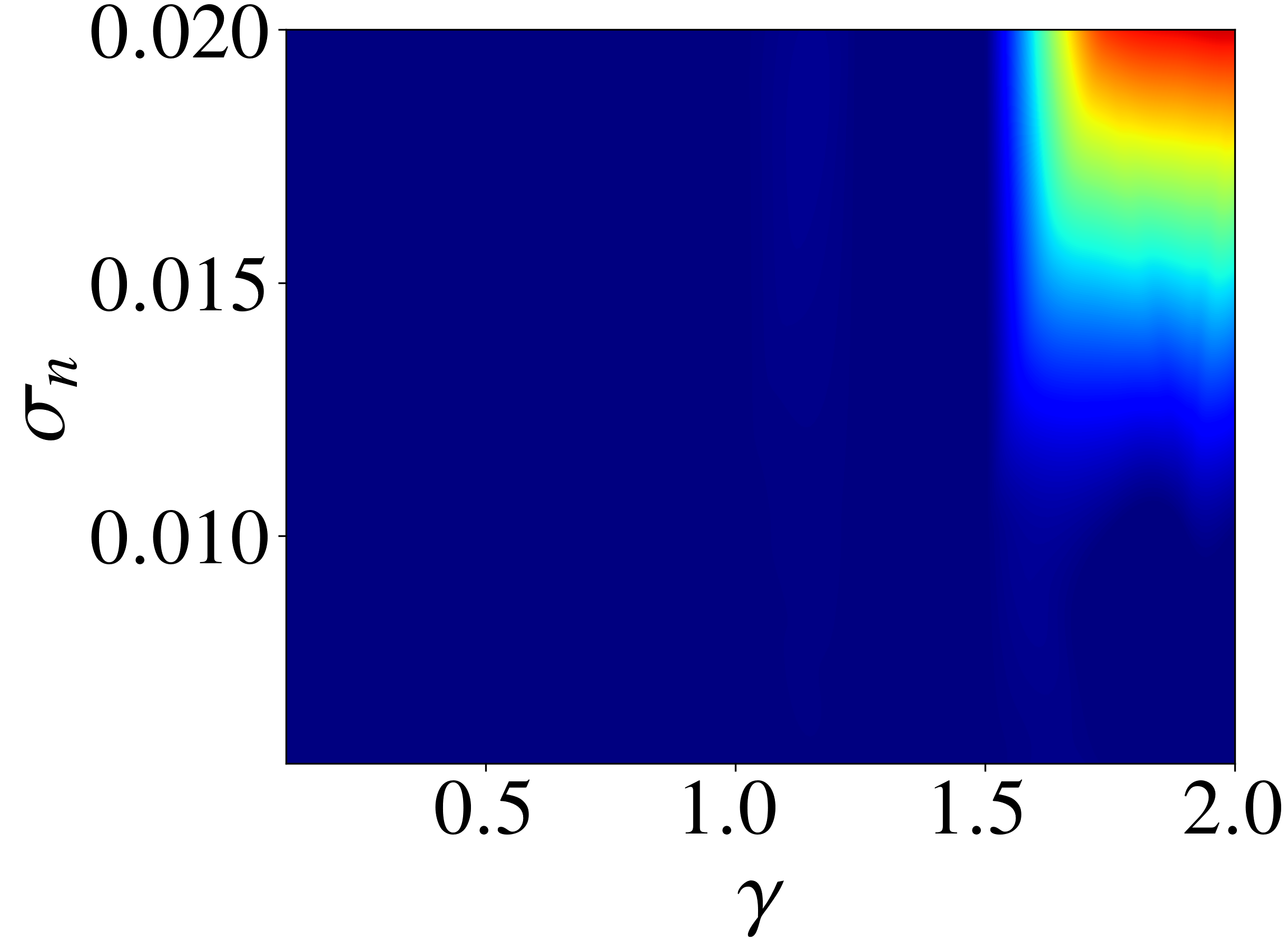
}}\hfill

\caption{(In)Stability regions for different denoiser-driven reconstruction operators. Each panel shows a 2D sweep over a pair of algorithm parameters. At each coordinate, the color encodes the ratio $r=T'/T$, where $T=1000$ is the total number of iterations and $T'$ is the first iteration at which the PSNR drops by $1$\,dB from its peak value. Larger $r$ indicates a more stable region (the peak is maintained for longer), while smaller $r$ indicates rapid post-peak degradation.}
\label{fig:rot}
\end{figure}

\begin{figure}[ht]
    \centering
    \captionsetup[subfloat]{labelformat=empty,labelsep=none,justification=centering}
    \subfloat[blurry]{
    \includegraphics[width=\Rthree]{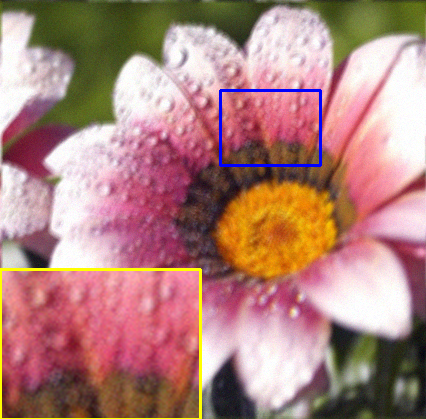}}\hfill
    \subfloat[Vanilla \\ (iter $=1000$)]{
    \includegraphics[width=\Rthree]{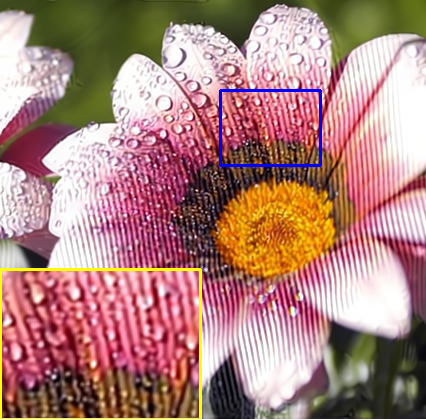}}\hfill
    \subfloat[Ours]{
    \includegraphics[width=\Rthree]{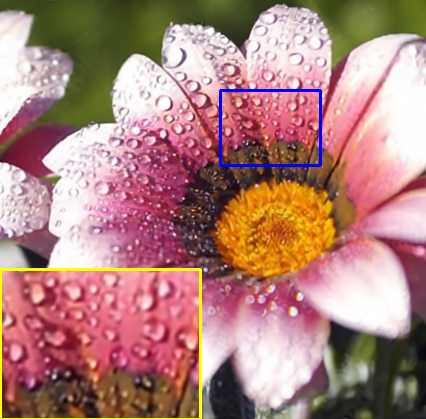}}\hfill
    \subfloat[Vanilla (peak)]{
    \includegraphics[width=\Rthree]{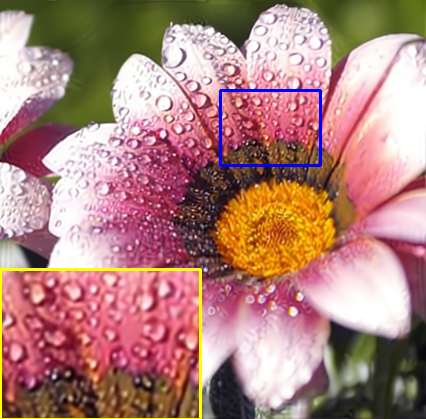}}\hfill
    \subfloat[$\p$]{
    \includegraphics[width=\Rthree]{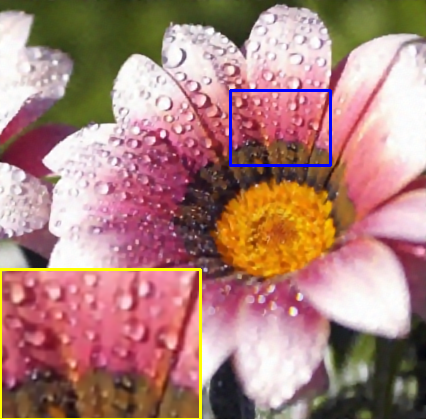}}\hfill
     \subfloat[clean]{
    \includegraphics[width=\Rthree]{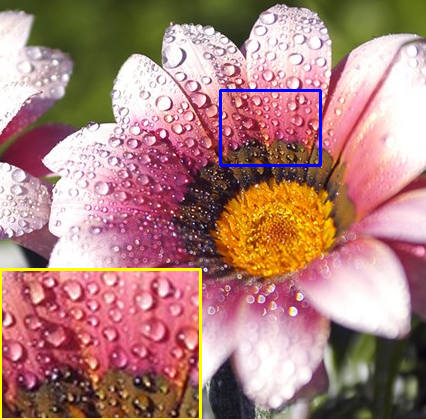}}\hfill
    \caption{PnP-HQS reconstruction results for motion deblurring on \emph{flower}\cite{general100} with kernel3~\cite{levin_kernel_2009} and noise level $\sigma_{\n}=0.02$. Vanilla-PnP introduces artifacts, while our algorithm provides the best reconstruction. The DRUNet denoiser is used in all cases. The PSNR(dB) values are: (a) $21.80$, (b) $20.83$, (c) $29.25$, (d) $29.10$ and (e) $27.58$.}
    \label{fig:lily_deblur_3}
\end{figure}

\begin{figure}[ht]
    \centering
    \captionsetup[subfloat]{labelformat=empty,labelsep=none,justification=centering}
    \subfloat[bicubic]{
    \includegraphics[width=\Rthree]{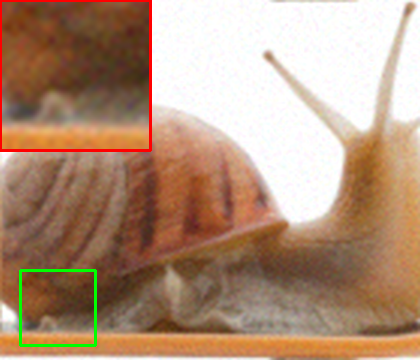}}\hfill
    \subfloat[Vanilla \\ (iter $=3000$)]{
    \includegraphics[width=\Rthree]{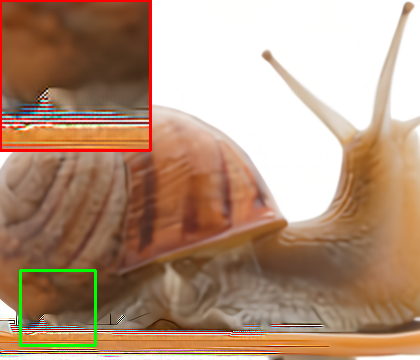}}\hfill
    \subfloat[Ours]{
    \includegraphics[width=\Rthree]{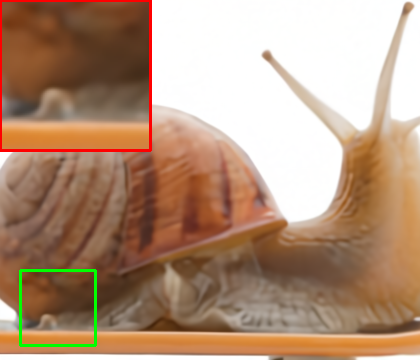}}\hfill
    \subfloat[Vanilla (peak)]{
    \includegraphics[width=\Rthree]{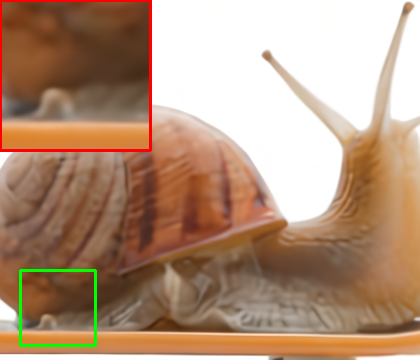}}\hfill
    \subfloat[$\p$]{
    \includegraphics[width=\Rthree]{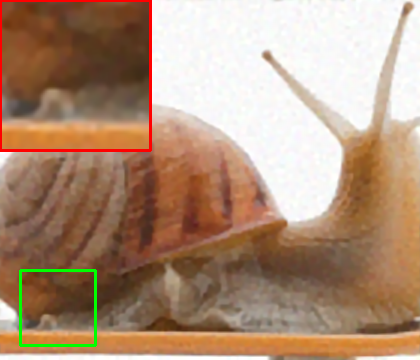}}\hfill
     \subfloat[clean]{
    \includegraphics[width=\Rthree]{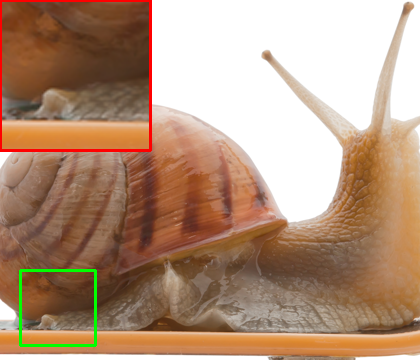}}\hfill
    \caption{PnP-HQS reconstruction results for $3\times$ superresolution on \emph{snail}\cite{general100} with kernel3 and noise level $\sigma_{\n}=0.02$; see \Cref{tab:cbsd10_superresolution_vista}. Vanilla-PnP introduces artifacts, while our algorithm provides a stable reconstruction. The GSDRUNet denoiser is used in all cases. The PSNR(dB) values are: (a) $21.16$, (b) $21.01$, (c) $33.06$, (d) $33.08$ and (e) $29.86$.}
    \label{fig:snail_3xSR}
\end{figure}

\begin{figure}[ht]
    \centering
    \captionsetup[subfloat]{labelformat=empty,labelsep=none,justification=centering}
    \subfloat[bicubic]{
    \includegraphics[width=\Rthree]{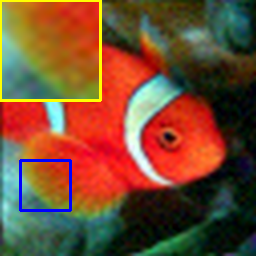}}\hfill
    \subfloat[Vanilla (peak)]{
    \includegraphics[width=\Rthree]{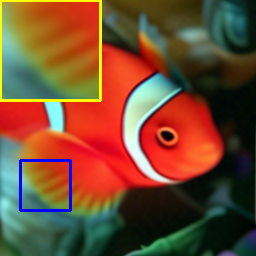}}\hfill
    \subfloat[Ours]{
    \includegraphics[width=\Rthree]{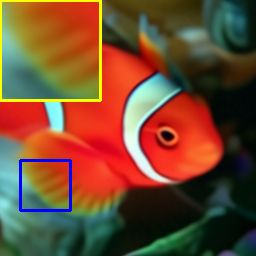}}\hfill
    \subfloat[Equiv (peak)]{
    \includegraphics[width=\Rthree]{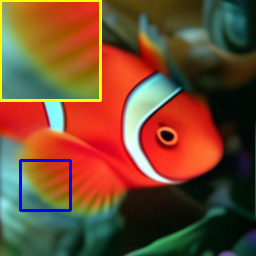}}\hfill
    \subfloat[$\p$]{
    \includegraphics[width=\Rthree]{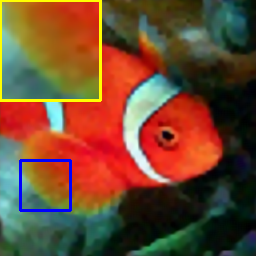}}\hfill
    \subfloat[clean]{
    \includegraphics[width=\Rthree]{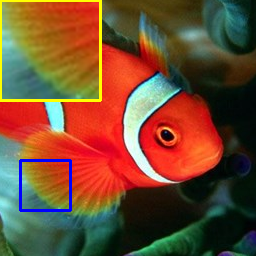}}\hfill
    \caption{PnP-ADMM reconstruction results for $4\times$ superresolution on \textit{fish}~\cite{general100} with noise level $\sigma_{\n}=0.03$. Our method with the DRUNet denoiser suppresses the artifacts and post-peak degradation observed in Vanilla-PnP while maintaining competitive reconstruction quality. The PSNR(dB) values are: (a) $22.70$, (b) $28.93$, (c) $29.34$, (d) $28.67$ and (e) $27.11$.}
    \label{fig:fish_4xSR}
\end{figure}

\begin{figure}[ht]
    \centering
    \captionsetup[subfloat]{labelformat=empty,labelsep=none,justification=centering}
    \subfloat[PSNR plots for \Cref{fig:lily_deblur_3}.]{
    \label{fig:lily-psnr}
    \includegraphics[width=0.31\columnwidth]{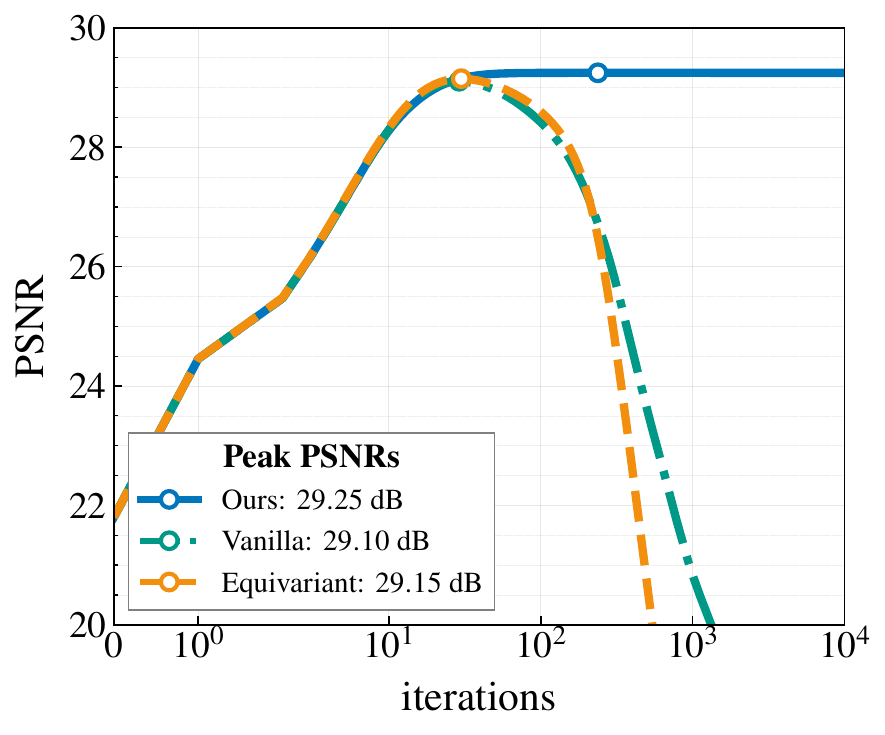}}\hfill
    \subfloat[PSNR plots for \Cref{fig:snail_3xSR}.]{
    \label{fig:snail-psnr}
    \includegraphics[width=0.31\columnwidth]{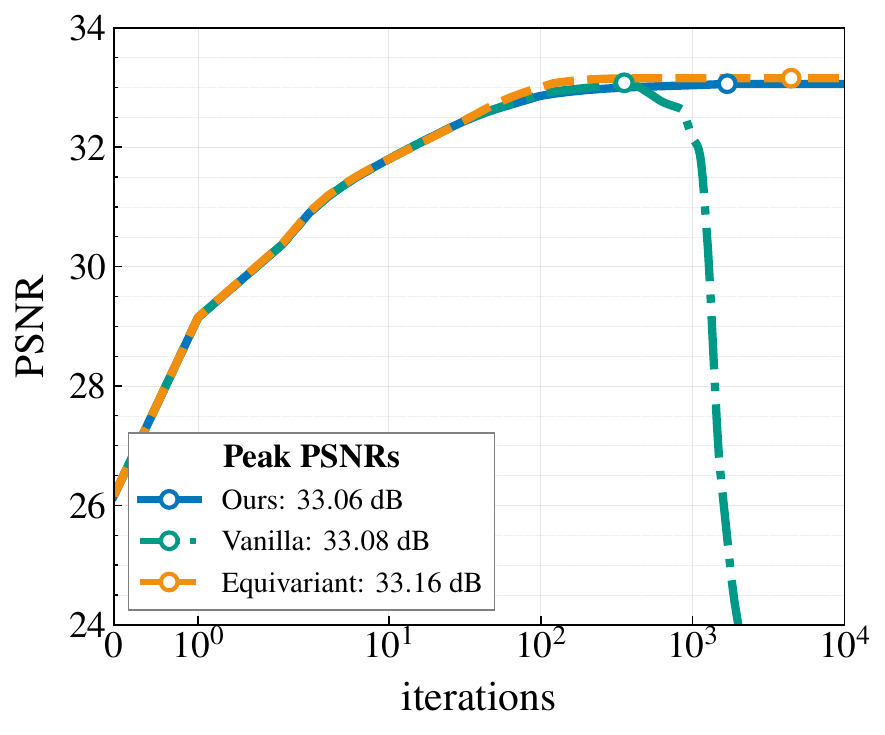}}
    \subfloat[PSNR plots for \Cref{fig:fish_4xSR}.]{
    \label{fig:fish-psnr}
    \includegraphics[width=0.31\columnwidth]{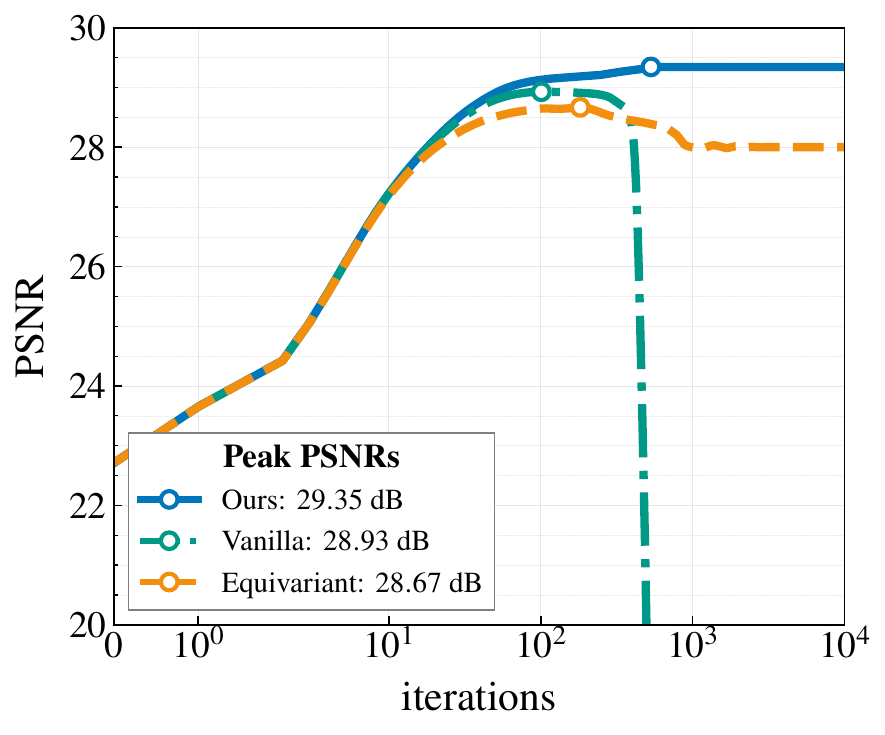}}
    \caption{Experiments showing that Vanilla-PnP and Equivariant-PnP can attain high peak PSNR but may collapse at later iterations, whereas our algorithm remains stable while maintaining high reconstruction quality.}
    \label{fig:psnr_plots_exps}
\end{figure}

We demonstrate the effectiveness of our stabilization framework on two standard inverse problems: \emph{deblurring} and \emph{superresolution}~\cite{zhang2021plug,terris2024equivariant}. The purpose of these experiments\footnote{code available at \url{https://github.com/trishitmg/costa}} is to demonstrate that our method prevents post-peak collapse and yields \emph{reliable} reconstructions across PnP frameworks, denoiser backbones, noise levels, and inverse problems.

\paragraph{\normalfont\bfseries Experimental setup.}
We use the kernels from~\cite{levin_kernel_2009} for motion deblurring, while a $25\times25$ Gaussian kernel with standard deviation $1.6$ is used for Gaussian deblurring. For superresolution, we apply Gaussian blur with standard deviation $0.7$, $1.2$, $1.6$, or $2.0$, followed by downsampling by a factor of $2$, $3$, or $4$.
In all experiments, the noise is additive Gaussian with standard deviation in $(0,0.03]$.
The test images are taken from set3c, CBSD68~\cite{bsd500}, and General100~\cite{general100}.
We initialize deblurring with the observed blurry image and superresolution with bicubic upsampling.

All pretrained models are obtained through DeepInverse~\cite{tachella2023deepinverse}.
We evaluate PnP-HQS~\eqref{eq:pnphqs} with pretrained denoisers DnCNN~\cite{zhang2017beyond}, DRUNet~\cite{zhang2021plug}, DiffUNet~\cite{choi2021conditioning}, GSDRUNet~\cite{hurault2022gradient}, and MMO~\cite{pesquet_learning_2021}.
We use $\rho=7$ and $\sigma=\sigma_{\n}$ for the anchor $\Sctr$ in~\eqref{eq:sctr}. 

To illustrate the stability landscape of denoiser-driven reconstruction operators, we visualize \emph{stability regions} obtained by sweeping pairs of algorithm parameters and measuring how long the PSNR remains close to its peak (see \Cref{fig:rot}). The resulting stability regions show that many PnP settings operate near instability. The stable regions are often narrow and depend strongly on the solver, while small parameter changes can lead to sharp post-peak degradation. This sensitivity provides further motivation for our stabilization method. The computational overhead of our method is around $20\%$ compared to Vanilla-PnP; see Appendix J.2 for details.

\paragraph{\normalfont\bfseries Quantitative results.}
Since stabilization is our primary objective, we use Vanilla PnP and Equivariant PnP~\cite{terris2024equivariant} as the main baselines and compare how reliably they avoid post-peak collapse. For completeness, we also compare against the high-performing reconstruction pipeline DPIR\cite{zhang2021plug} and the convergent methods GSPnP~\cite{hurault2022gradient} and DEAL~\cite{pourya2025dealing}, using the recommended settings from their respective codebases.
We report both \textbf{peak PSNR} (maximum over iterations) and \textbf{final PSNR} 
(PSNR after $10,\!000$ iterations) to represent stability and long-term behavior.

Tables~\ref{tab:cbsd10_deblurring_vista} and~\ref{tab:cbsd10_superresolution_vista} report \emph{peak} and \emph{final} PSNR (mean $\pm$ std.\ dev.) on CBSD10~\cite{hurault2022gradient} at $\sigma_{\n}=0.02$ across multiple PnP frameworks and denoisers. 
Within each column, \BestCol{teal} and \runupCol{orange} indicate the best and second-best results, while \textbf{bold} and \underline{underlined} entries indicate the overall best and second-best values.

Across settings, \textbf{Vanilla-PnP} and often \textbf{Equivariant-PnP} achieve strong \emph{peak} PSNR but frequently diverge (\xmark) or degrade substantially afterward. In contrast, our stabilization algorithm consistently keeps the \emph{final} PSNR close to the \emph{peak}, demonstrating stable performance without early stopping.


Our results show that stabilization through \Cref{algo:vista} alone can recover much of the performance typically attributed to carefully tuned reconstruction pipelines such as \textbf{DPIR}, \textbf{GSPnP}, and \textbf{DEAL}, while simultaneously ensuring stable PnP trajectories across different frameworks, denoisers, and inverse problems. For example, in the deblurring experiments, our method with HQS+GS-DRUNet achieves \textbf{30.70} dB for the motion blur case, equivalent to DPIR (\underline{30.69} dB). 
Similarly, for $2\times$ superresolution, our method attains \textbf{27.99} dB, whereas DPIR attains \underline{27.93} dB. 
Across the remaining settings, including Gaussian deblurring and $3\times$ and $4\times$ superresolution, the stabilized variants remain competitive and closely match the best-performing methods.

\paragraph{\normalfont\bfseries Qualitative results.}
\Cref{fig:lily_deblur_3,fig:snail_3xSR,fig:fish_4xSR} show the same effect visually. For motion deblurring (\Cref{fig:lily_deblur_3}) and $3\times$ superresolution (\Cref{fig:snail_3xSR}), Vanilla-PnP introduces artifacts and deteriorates at later iterations, while our method produces cleaner reconstructions and avoids post-peak collapse. For $4\times$ superresolution with PnP-ADMM (\Cref{fig:fish_4xSR}), our stabilized update suppresses the post-peak degradation of Vanilla-PnP while maintaining competitive perceptual quality. Finally, \Cref{fig:psnr_plots_exps} shows representative PSNR trajectories. Vanilla-PnP and Equivariant-PnP may attain slightly higher peaks but can collapse, whereas our method remains stable over long iterations.

\begin{table*}[p]
\centering
\caption{PSNR results (mean $\pm$ std.\ dev.) for deblurring using various PnP frameworks on CBSD10 ($\sigma_{\n}=0.02$). \xmark~indicates divergence of the iterates.}
\label{tab:cbsd10_deblurring_vista}
\setlength\tabcolsep{5pt}
\scriptsize
\begin{tabular}{llcccc}
\toprule
\multirow{3}{*}{\textbf{Framework}} & \multirow{3}{*}{\textbf{Method}} &  \multicolumn{2}{c}{\multirow{2}{*}{
\raisebox{\height}{\includegraphics[height=0.08\textwidth,keepaspectratio]{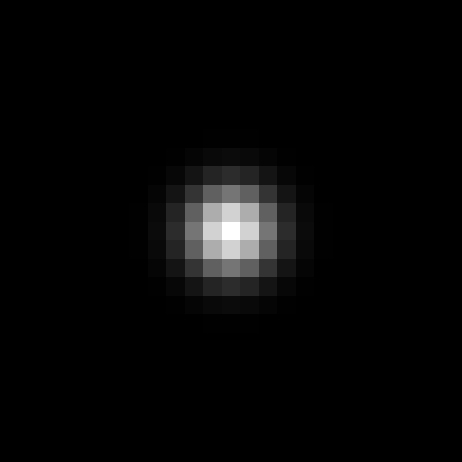}}}} &  
\multicolumn{2}{c}{
\includegraphics[width=\lenBlur]{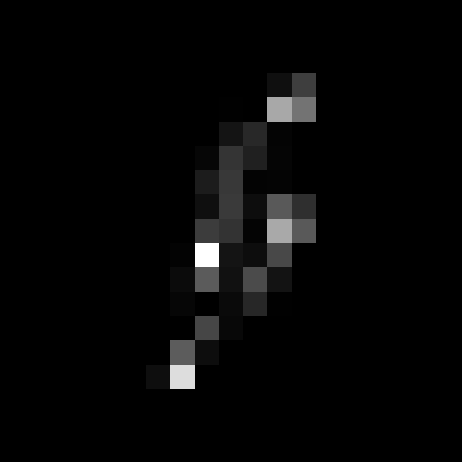}
\includegraphics[width=\lenBlur]{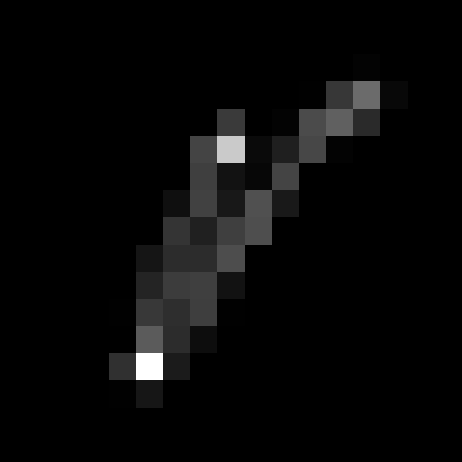}
\includegraphics[width=\lenBlur]{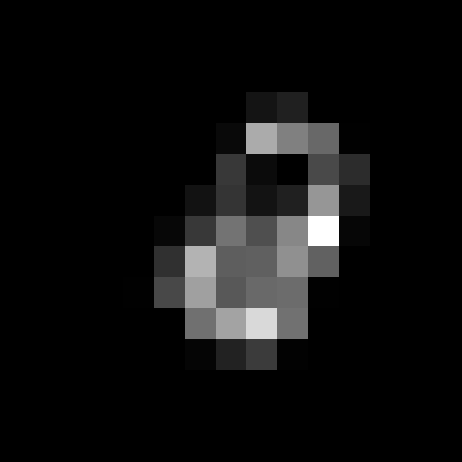}
\includegraphics[width=\lenBlur]
{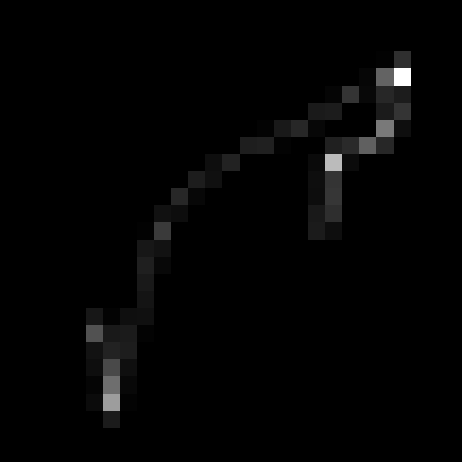}
}\\
& & & & 
\multicolumn{2}{c}{
\includegraphics[width=\lenBlur]{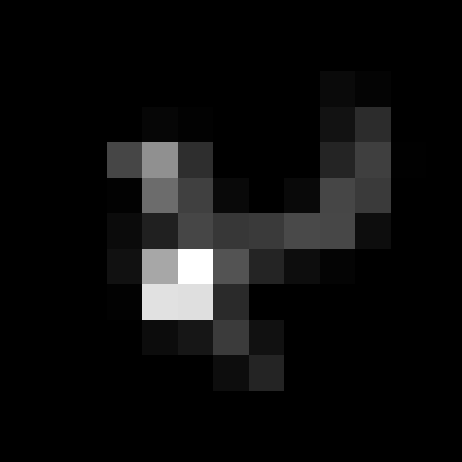}
\includegraphics[width=\lenBlur]{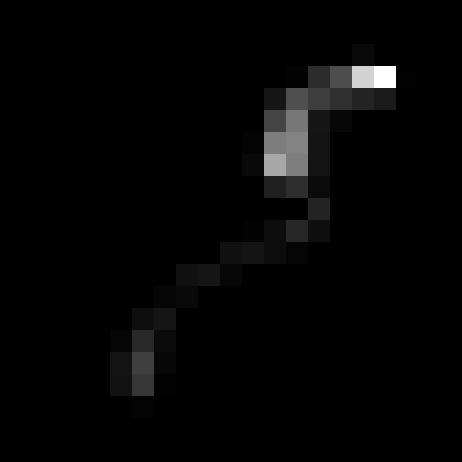}
\includegraphics[width=\lenBlur]{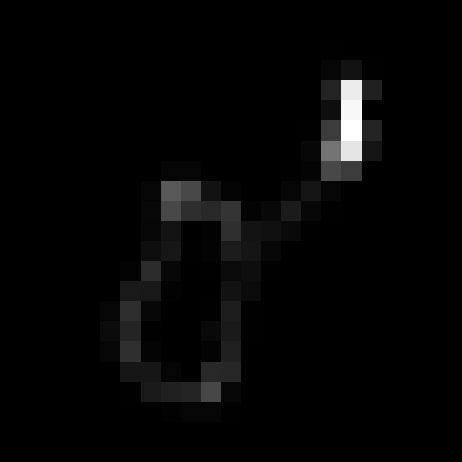}
\includegraphics[width=\lenBlur]{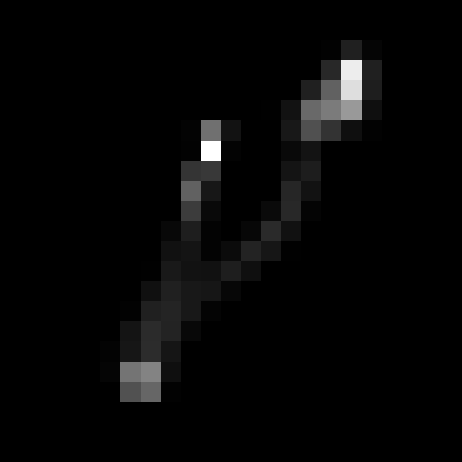}
}\\
\cmidrule(lr){3-4} \cmidrule(lr){5-6}
 &  & Peak & Final & Peak & Final \\
\midrule
\multicolumn{2}{c}{\textbf{Observed}} & \multicolumn{2}{c}{$24.15\pm3.19$} & \multicolumn{2}{c}{$20.38\pm3.16$} \\
\midrule
$\Sctr$ & & \multicolumn{2}{c}{$27.55\pm3.96$} & \multicolumn{2}{c}{$28.71\pm0.64$}\\
\addlinespace[1.5pt]
\hline
\addlinespace[1.5pt]
\multirow{3}{*}{\parbox{1.5cm}{\textit{PnP-PGD} +\newline \textbf{DnCNN}~\cite{zhang2017beyond}}} & \textbf{Ours} & \runupCol{$27.96\pm4.38$} & \BestCol{$27.95\pm4.38$} & \BestCol{$28.78\pm3.02$} & \BestCol{$28.77\pm3.02$} \\
 & Vanilla & $27.78\pm4.13$ & \xmark & $27.53\pm2.97$ & \xmark \\
 & Equiv. & \BestCol{$27.97\pm4.37$} & \xmark & \runupCol{$27.58\pm3.01$} & \xmark \\
\addlinespace[1.5pt]
\hline
\addlinespace[1.5pt]
\multirow{3}{*}{\parbox{1.5cm}{\textit{PnP-HQS} +\newline \textbf{DRUNet}~\cite{zhang2021plug}}} & \textbf{Ours} & \BestCol{$27.86\pm4.42$} & \BestCol{$27.86\pm4.42$} & \BestCol{$30.02\pm4.03$} & \BestCol{$30.02\pm4.03$} \\
 & Vanilla & $27.60\pm4.28$ & \runupCol{$26.90\pm4.36$} & $29.78\pm3.89$ & \xmark \\
 & Equiv. & \runupCol{$27.70\pm4.42$} & \xmark & \runupCol{$29.83\pm3.94$} & \xmark \\
\addlinespace[1.5pt] \hline \addlinespace[1.5pt]
\multirow{3}{*}{\parbox{2cm}{\textit{PnP-HQS} +\newline \textbf{DiffUNet}~\cite{choi2021conditioning}}} & \textbf{Ours} & \BestCol{$27.84\pm3.98$} & \BestCol{$27.82\pm3.96$} & \BestCol{$29.62\pm3.49$} & \BestCol{$29.61\pm3.48$} \\
 & Vanilla & $27.21\pm3.52$ & \runupCol{$18.90\pm0.92$} & $29.20\pm3.22$ & \runupCol{$23.12\pm3.17$} \\
 & Equiv. & \runupCol{$27.40\pm3.72$} & $18.71\pm0.72$ & \runupCol{$29.23\pm3.25$} & $23.00\pm3.14$ \\
\addlinespace[1.5pt] \hline \addlinespace[1.5pt]
\multirow{3}{*}{\parbox{2cm}{\textit{PnP-HQS} +\newline \textbf{GS-DRUNet}~\cite{hurault2022gradient}}} & \textbf{Ours} & \BestCol{$28.29\pm4.39$} & \BestCol{$28.28\pm4.39$} & \BestCol{$\mathbf{30.70\pm4.12}$} & \BestCol{$\mathbf{30.70\pm4.12}$} \\
 & Vanilla & $28.18\pm4.26$ & $23.86\pm8.18$ & $30.56\pm4.13$ & $25.72\pm8.23$ \\
 & Equiv. & \runupCol{$28.27\pm4.42$} & \runupCol{$27.86\pm4.05$} & \runupCol{$30.58\pm4.16$} & \runupCol{$26.10\pm8.37$} \\
\addlinespace[1.5pt] \hline \addlinespace[1.5pt]
\multirow{3}{*}{\parbox{2cm}{\textit{PnP-ADMM} +\newline \textbf{CoCo-DRUNet}~\cite{wei2025learning}}} & \textbf{Ours} & \BestCol{$27.85\pm4.30$} & \BestCol{$27.65\pm4.39$} & \BestCol{$29.26\pm3.01$} & \BestCol{$29.19\pm3.05$} \\
 & Vanilla & $27.59\pm4.14$ & \xmark & $28.28\pm2.92$ & \xmark \\
 & Equiv. & \runupCol{$27.65\pm4.20$} & \xmark & \runupCol{$28.31\pm2.93$} & \xmark \\
\addlinespace[1.5pt] \hline \addlinespace[1.5pt]
\multirow{3}{*}{\parbox{2cm}{\textit{PnP-ADMM} +\newline \textbf{MMO}~\cite{pesquet_learning_2021}}} & \textbf{Ours} & $27.90\pm4.06$ & $27.90\pm4.06$ & \BestCol{$29.09\pm2.84$} & \BestCol{$29.09\pm2.84$} \\
 & Vanilla & \runupCol{$27.93\pm3.99$} & \runupCol{$27.93\pm3.99$} & \runupCol{$27.68\pm2.78$} & $22.43\pm2.31$ \\
 & Equiv. & \BestCol{$27.94\pm4.01$} & \BestCol{$27.94\pm4.01$} & $27.68\pm2.79$ & \runupCol{$22.57\pm2.44$} \\
\addlinespace[1.5pt] \hline \addlinespace[1.5pt]
\multirow{3}{*}{\parbox{2cm}{\textit{RED-GD} +\newline \textbf{DRUNet}~\cite{zhang2021plug}}} & \textbf{Ours} & \BestCol{$27.74\pm4.19$} & \BestCol{$27.74\pm4.19$} & \BestCol{$29.48\pm3.18$} & \BestCol{$29.48\pm3.18$} \\
 & Vanilla & $27.49\pm4.08$ & \runupCol{$27.00\pm4.44$} & $29.17\pm3.04$ & \runupCol{$25.93\pm6.13$} \\
 & Equiv. & \runupCol{$27.59\pm4.18$} & \xmark & \runupCol{$29.21\pm3.06$} & \xmark \\
\addlinespace[1.5pt] \hline \addlinespace[1.5pt]
\textbf{DPIR}~\cite{zhang2021plug} & & \multicolumn{2}{c}{\underline{$28.33\pm4.65$}} & \multicolumn{2}{c}{\underline{$30.69\pm0.52$}} \\
\textbf{GSPnP}~\cite{hurault2022gradient} & & \multicolumn{2}{c}{$28.21\pm4.23$} & \multicolumn{2}{c}{$30.58\pm0.59$} \\
\textbf{DEAL}~\cite{pourya2025dealing} & & \multicolumn{2}{c}{$\mathbf{28.44\pm4.49}$} & \multicolumn{2}{c}{$30.45\pm0.44$} \\
\bottomrule
\end{tabular}
\end{table*}

\begin{table*}[htpb]
\centering
\caption{PSNR results (mean $\pm$ std.\ dev.) for superresolution using various PnP frameworks on CBSD10 ($\sigma_{\n}=0.02$). \xmark~indicates divergence of the iterates.}
\label{tab:cbsd10_superresolution_vista}
\setlength\tabcolsep{2pt}
\tiny
\begin{tabular}{lllcccccc}
\toprule
 & \textbf{Framework} & \textbf{Method} & \multicolumn{2}{c}{$s=2$} & \multicolumn{2}{c}{$s=3$} & \multicolumn{2}{c}{$s=4$} \\
\cmidrule(lr){4-5} \cmidrule(lr){6-7} \cmidrule(lr){8-9}
 & &  & Peak & Final & Peak & Final & Peak & Final \\
\midrule
& \multicolumn{2}{c}{\textbf{Bicubic}} & \multicolumn{2}{c}{$24.28\pm3.37$} & \multicolumn{2}{c}{$22.68\pm3.46$} & \multicolumn{2}{c}{$21.40\pm3.48$} \\
\midrule
\multirow{18}{*}{
\parbox{\lenSR}{
\includegraphics[width=\lenSR]{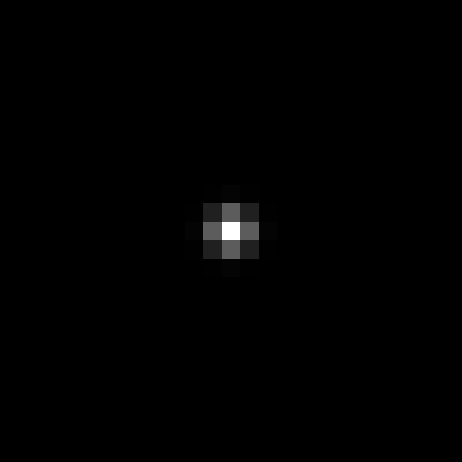}\\
\includegraphics[width=\lenSR]{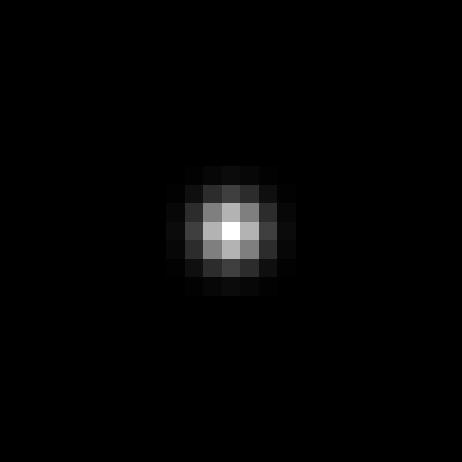}\\
\includegraphics[width=\lenSR]{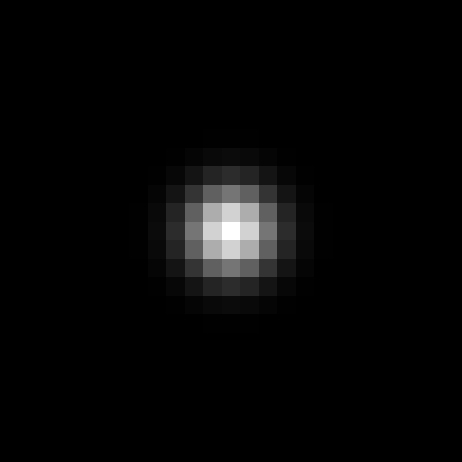}\\
\includegraphics[width=\lenSR]{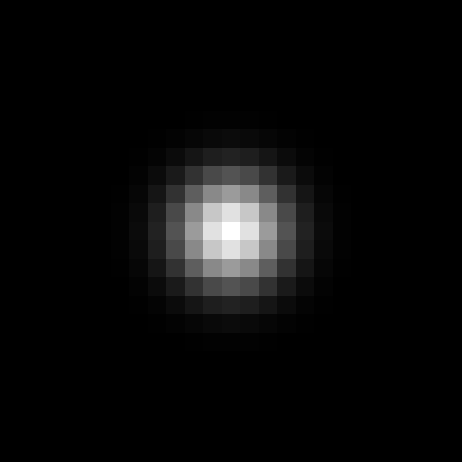}
}
} & $\Sctr$ & & \multicolumn{2}{c}{$26.84\pm0.99$} & \multicolumn{2}{c}{$24.97\pm0.51$} & \multicolumn{2}{c}{$22.77\pm1.87$} \\
\addlinespace[1.5pt] \cline{2-9} \addlinespace[1.5pt]
& \multirow{3}{*}{\parbox{1.5cm}{\textit{PnP-HQS} +\newline \textbf{DRUNet}~\cite{zhang2021plug}}} & \textbf{Ours} & \BestCol{$27.48\pm4.25$} & \BestCol{$27.48\pm4.24$} & \BestCol{$25.85\pm4.40$} & \BestCol{$25.83\pm4.43$} & \BestCol{$24.29\pm4.46$} & \BestCol{$23.96\pm4.80$} \\
& & Vanilla & $27.29\pm4.25$ & \xmark & $25.64\pm4.35$ & \xmark & $24.04\pm4.32$ & \xmark \\
& & Equiv. & \runupCol{$27.37\pm4.32$} & \xmark & \runupCol{$25.72\pm4.35$} & \xmark & \runupCol{$24.15\pm4.31$} & \xmark \\
\addlinespace[1.5pt] \cline{2-9} \addlinespace[1.5pt]
& \multirow{3}{*}{\parbox{1.5cm}{\textit{PnP-HQS} +\newline \textbf{GS-DRUNet}~\cite{hurault2022gradient}}} & \textbf{Ours} & \BestCol{$\mathbf{27.99\pm4.52}$} & \BestCol{$\mathbf{27.95\pm4.47}$} & \runupCol{$26.30\pm4.52$} & \BestCol{$26.29\pm4.53$} & \BestCol{$24.74\pm4.41$} & \BestCol{$24.62\pm4.37$} \\
& & Vanilla & $27.89\pm4.37$ & $23.84\pm7.91$ & $26.24\pm4.48$ & $24.08\pm5.78$ & $24.65\pm4.37$ & $21.02\pm6.71$ \\
& & Equiv. & \runupCol{$27.97\pm4.49$} & \runupCol{$25.24\pm7.23$} & \BestCol{$26.32\pm4.57$} & \runupCol{$24.22\pm6.41$} & \runupCol{$24.69\pm4.42$} & \runupCol{$21.51\pm6.33$} \\
\addlinespace[1.5pt] \cline{2-9} \addlinespace[1.5pt]
& \multirow{3}{*}{\parbox{1.5cm}{\textit{PnP-IHQS} +\newline \textbf{SPC-DRUNet}~\cite{wei2024learning}}} & \textbf{Ours} & \BestCol{$27.26\pm4.04$} & \BestCol{$27.22\pm4.07$} & \BestCol{$25.61\pm4.15$} & \BestCol{$25.55\pm4.22$} & \BestCol{$24.03\pm4.12$} & \BestCol{$23.75\pm4.29$} \\
& & Vanilla & \runupCol{$27.14\pm3.96$} & \xmark & $25.46\pm4.05$ & \xmark & $23.88\pm3.99$ & \xmark \\
& & Equiv. & $27.12\pm3.98$ & \xmark & \runupCol{$25.47\pm4.07$} & \xmark & \runupCol{$23.88\pm4.01$} & \xmark \\
\addlinespace[1.5pt] \cline{2-9} \addlinespace[1.5pt]
& \multirow{3}{*}{\parbox{1.5cm}{\textit{RED-GD} +\newline \textbf{GS-DRUNet}~\cite{hurault2022gradient}}} & \textbf{Ours} & $27.54\pm4.18$ & \BestCol{$27.53\pm4.17$} & $26.08\pm4.39$ & \BestCol{$26.07\pm4.39$} & \BestCol{$24.66\pm4.34$} & \BestCol{$24.66\pm4.35$} \\
& & Vanilla & \runupCol{$27.58\pm4.02$} & $25.48\pm5.04$ & \runupCol{$26.11\pm4.33$} & $24.84\pm4.33$ & $24.61\pm4.33$ & $23.24\pm4.42$ \\
& & Equiv. & \BestCol{$27.73\pm4.17$} & \runupCol{$26.60\pm5.03$} & \BestCol{$26.22\pm4.44$} & \runupCol{$25.35\pm4.81$} & \runupCol{$24.65\pm4.36$} & \runupCol{$23.38\pm4.62$} \\
\addlinespace[1.5pt] \cline{2-9} \addlinespace[1.5pt]
 & \textbf{DPIR}~\cite{zhang2021plug} & & \multicolumn{2}{c}{\underline{$27.93\pm1.16$}} & \multicolumn{2}{c}{\underline{$26.32\pm0.35$}} & \multicolumn{2}{c}{\underline{$24.76\pm0.36$}} \\
 & \textbf{GSPnP}~\cite{hurault2022gradient} & & \multicolumn{2}{c}{$27.47\pm0.78$} & \multicolumn{2}{c}{$26.13\pm0.28$} & \multicolumn{2}{c}{$24.66\pm0.54$}  \\
& \textbf{DEAL}~\cite{pourya2025dealing} & & \multicolumn{2}{c}{$27.87\pm1.05$} & \multicolumn{2}{c}{$\mathbf{26.40\pm0.37}$} & \multicolumn{2}{c}{$\mathbf{25.02\pm0.25}$} \\
\bottomrule
\end{tabular}
\end{table*}

\section{Conclusion}
We presented a stabilization framework for denoiser-driven black-box reconstruction systems that applies across proximal solvers, pretrained denoisers, and measurement models, without retraining or modifying the underlying black box. Our approach characterizes instability via a stability index $\eta$ and adaptively controls it, while keeping the updates largely driven by the original operator to preserve reconstruction quality. The resulting method is easy to deploy as a drop-in wrapper, requires no additional parameter tuning, and uses only input-output evaluations of the black-box IR system. 

Extensive experiments show that the proposed framework consistently prevents peak-and-collapse behavior and produces reliable reconstructions across a broad range of settings. In particular, it enables the stable use of standard IR pipelines without the careful modifications that are otherwise needed to avoid collapse. Future work includes incorporating this stabilization principle into training to obtain inherently reliable IR systems and developing a deeper understanding of peak-and-collapse in modern denoising models.

\section*{Acknowledgements}
A.~Sinha was supported by the Government of India through the Prime Minister’s Research Fellowship (PMRF) under grant TF/PMRF-22-5534, and by the Qualcomm Innovation Fellowship India under grant 4300074105. 

K.~N.~Chaudhury and T.~Mukherjee were supported by the Government of India through grant ANRF/ARG/2025/00696/ENS. The authors thank the Kotak IISc AI-ML Center at the Indian Institute of Science for providing GPU resources.

%
%
\bibliographystyle{splncs04}
\bibliography{ref}

\input{appendix}

\end{document}

%% file: arch.tex
%
\definecolor{convfill}{HTML}{F2EDE4}      
\definecolor{convstroke}{HTML}{B0A48E}     
\definecolor{normfill}{HTML}{E4EDF2}       
\definecolor{normstroke}{HTML}{7A9EB5}     
\definecolor{scalefill}{HTML}{E4F2EA}      
\definecolor{scalestroke}{HTML}{7AB59A}    
\definecolor{actfill}{HTML}{F2E4EC}        
\definecolor{actstroke}{HTML}{B57A96}      
\definecolor{taufill}{HTML}{EAE4F2}        
\definecolor{taustroke}{HTML}{8E7AB5}      
\definecolor{stepfill}{HTML}{F5F4C4}       
\definecolor{stepstroke}{HTML}{B5967A}     
\definecolor{plusfill}{HTML}{7AB5A8}       
\definecolor{minusfill}{HTML}{B57A7A}      
\definecolor{arrowgray}{HTML}{5A5A5A}     
\definecolor{sigmaline}{HTML}{7AB59A}      
\begin{figure}[t]
    \centering
    \resizebox{\linewidth}{!}{
    \begin{tikzpicture}[
        convbox/.style={
            rectangle, rounded corners=6pt,
            fill=convfill, draw=convstroke, line width=0.8pt,
            minimum width=1.5cm, minimum height=2.4cm,
            align=center, font=\large},
        normbox/.style={
            rectangle, rounded corners=6pt,
            fill=normfill, draw=normstroke, line width=0.8pt,
            minimum width=1.4cm, minimum height=2.4cm,
            align=center, font=\large},
        scalebox/.style={
            rectangle, rounded corners=6pt,
            fill=scalefill, draw=scalestroke, line width=0.8pt,
            minimum width=1.3cm, minimum height=2.4cm,
            align=center, font=\large},
        actbox/.style={
            rectangle, rounded corners=6pt,
            fill=actfill, draw=actstroke, line width=0.8pt,
            minimum width=1.2cm, minimum height=2.4cm,
            align=center, font=\large},
        stepbox/.style={
            rectangle, rounded corners=5pt,
            fill=stepfill, draw=stepstroke, line width=0.8pt,
            minimum width=0.85cm, minimum height=0.85cm,
            align=center, font=\LARGE},
        taubox/.style={
            rectangle, rounded corners=4pt,
            fill=taufill, draw=taustroke, line width=0.8pt,
            minimum width=0.75cm, minimum height=0.75cm,
            align=center, font=\LARGE},
        pluscirc/.style={
            circle, draw=plusfill!80!black, line width=0.8pt,
            fill=plusfill, text=white, inner sep=2pt,
            font=\Large\bfseries},
        minuscirc/.style={
            circle, draw=minusfill!80!black, line width=0.8pt,
            fill=minusfill, text=white, inner sep=2pt,
            font=\Large\bfseries},
        arr/.style={
            -Stealth, line width=0.7pt, color=arrowgray,
            rounded corners=2pt},
        sigmaarr/.style={
            -Stealth, line width=0.7pt, color=sigmaline!40!black,
            rounded corners=2pt, dashed},
    ]

    \node[convbox] (Wfwd) at (-5.0, 0) {$\widehat{\W}_{\mathrm{conv}}$};

    \node[normbox] (T1) at (-3.1, 0) {$\R^{-1/2}$};

    \node[scalebox] (scaleup) at (-1.3, 0) {$\ScalFunc(\sigma_{\n})$};

    \node[actbox] (act) at (0.4, 0) {$\Lspline$};

    \node[scalebox] (scaledn) at (2.1, 0) {$\dfrac{1}{\ScalFunc(\sigma_{\n})}$};

    \node[normbox] (T2) at (3.9, 0) {$\R^{-1/2}$};

    \node[convbox] (Wtrn) at (5.8, 0) {$\widehat{\W}_{\mathrm{conv}}^{\top}$};

    \node[pluscirc]  (plus)  at (7.4, 0) {$+$};
    \node[stepbox]   (gamma) at (8.6, 0) {$\gamma$};
    \node[minuscirc] (minus) at (9.8, 0) {$-$};

    \node[taubox] (tau) at (0.4, -2.2) {$\tau$};

    \node[inner sep=2pt]
        (inimg) at (-8.5, 0)
        {\includegraphics[width=2.5cm]{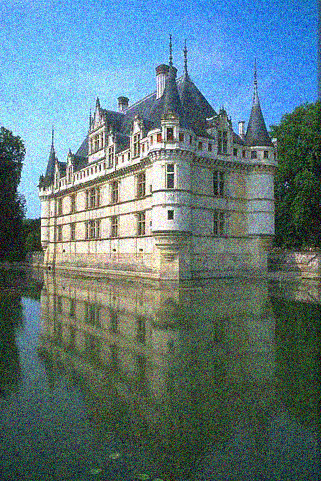}};
    \node[below=2pt, font=\large] at (inimg.south) {noisy $\bar{\x}+\n$};

    \node[inner sep=2pt]
        (outimg) at (12.5, 0)
        {\includegraphics[width=2.5cm]{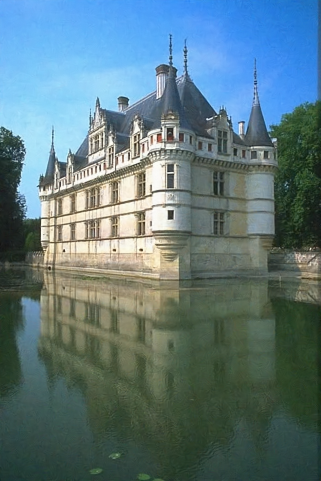}};
    \node[below=2pt, font=\large] at (outimg.south) {denoised $\hat{\x}$};


    \draw[arr] (inimg.east)  -- (Wfwd.west);
    \draw[arr] (minus.east)  -- (outimg.west);

    \draw[arr] (Wfwd.east)    -- (T1.west);
    \draw[arr] (T1.east)      -- (scaleup.west);
    \draw[arr] (scaleup.east) -- (act.west);
    \draw[arr] (act.east)     -- (scaledn.west);
    \draw[arr] (scaledn.east) -- (T2.west);
    \draw[arr] (T2.east)      -- (Wtrn.west);
    \draw[arr] (Wtrn.east)    -- (plus.west);
    \draw[arr] (plus.east)    -- (gamma.west);
    \draw[arr] (gamma.east)   -- (minus.west);

    \draw[arr] (-6.8, 0) -- (-6.8, -2.2) -- (tau.west);
    \draw[arr] (tau.east)  -- (7.4, -2.2) -- (plus.south);

    \draw[arr] (-6.5, 0) -- (-6.5, 1.8) -- (9.8, 1.8) -- (minus.north);

    %
    \pgfmathsetmacro{\sigmaRailY}{2.2}
    %
    \coordinate (sigmaJunction) at (0.4, \sigmaRailY);
    %
    \draw[sigmaarr, -{}, rounded corners=4pt]
        (inimg.north) -- ++(0, 1.)       
        -| (sigmaJunction);               
    %
    \node[above=-1pt, font=\large, color=sigmaline!40!black]
        at (-8.0, \sigmaRailY) {$\sigma_{\n}$};
    %
    \draw[sigmaarr, rounded corners=4pt]
        (sigmaJunction) -- (scaleup.north);
    %
    \draw[sigmaarr, rounded corners=4pt]
        (sigmaJunction) -- (scaledn.north);

    \end{tikzpicture}
    }
    \caption{Architecture of the noise-aware denoiser $\Dsig$ defined in~\eqref{eq:Dsig}.}
    \label{fig:arch}
\end{figure}

%% file: appendix.tex
\appendix

\newcommand{\xsol}{\x^\ast}
\newcommand{\xfp}{\x_\ast}

\section*{Appendix}

\noindent\textbf{Organization}
\begin{itemize}
    \item \textbf{Appendices A--F (Proofs):} Detailed proofs of the mathematical results stated in the main text.

    \item \textbf{Appendix G (Standalone performance of $\Dsig$):} Additional experiments evaluating the contractive denoiser $\Dsig$ and its induced IR operators in isolation. These results provide further insight into the behaviour of the proposed contractive components outside the stabilization framework.

    \item \textbf{Appendix H (Classical early stopping methods)} Comparisons with fixed-iteration stopping, a discrepancy principle, and related classical criteria. These experiments show that the appropriate stopping iteration can depend strongly on the input image and that, even with the exact noise level, the discrepancy principle may fail to stop near the PSNR peak.

    \item \textbf{Appendix I (Ablation studies):} Architectural ablations of $\Dsig$ and comparisons of the resulting anchors used in \Cref{algo:vista}. Since the anchor is the only component that varies in the algorithm, ablations are naturally performed by studying different anchor constructions. These experiments justify the design choices made in the main paper and explain the selection of $\Sctr$ as the anchor operator.

    \item \textbf{Appendix J (Experiments):} Stability landscapes for Equivariant-PnP~\cite{terris2024equivariant}, discussion on computational overhead and additional visual results across multiple inverse problems, noise levels, and modern deep denoisers. These examples illustrate that instability or the \emph{peak-and-collapse} (\emph{PC}) behaviour are common across tasks and architectures, and demonstrate the robustness of the proposed stabilization framework in preventing such failures.
\end{itemize}
\section{Proof of \Cref{lem:eta-reduces}}
\begin{proof}
We begin by expanding the definition of $\cT_{\theta}$:
\begin{align*}
    \|\cT_{\theta}(\x) - \p\| 
    &= \|(1-\theta)\cT(\x) + \theta S(\x) - \p\|\\
    &= \|(1-\theta)(\cT(\x) - \p) + \theta(S(\x) - \p)\|.
\end{align*}
Since $\p$ is a fixed point of $\cS$, we have $\cS(\p) = \p$. Substituting this we have:
\begin{align*}
    \|\cT_{\theta}(\x) - \p\| 
    &= \|(1-\theta)(\cT(\x) - \p) + \theta(S(\x) - \cS(\p))\|\\
    &\le (1-\theta)\|\cT(\x) - \p\| + \theta\|S(\x) - \cS(\p)\|,
\end{align*}
Dividing both sides by $\|\x - \p\|$ and applying the contraction property of $\cS$, we obtain:
\begin{equation*}
    \etap(\x, \cT_{\theta}) \le (1-\theta)\etap(\x, \cT) + \theta\kappa.
\end{equation*}
Therefore, for any $\theta \in (0,1]$, whenever $\etap(\x, \cT) > \kappa$, the desired result follows.
\qed
\end{proof}

\section{Proof of \Cref{thm:theta-threshold}}
\begin{proof}
Let $\kappa < \xi < \etap(\x, T)$. First we show the existence as follows.\\
Define $\psi(\theta) := \etap(\x, T_\theta)$. Note that, $\psi$ is continuous. \\
Observe that $\psi(0) = \etap(\x, T) > \xi$ and $\psi(1) \le \kappa < \xi$. \\
By the Intermediate Value Theorem\cite{rudin1976principles}, there exists $\ttheta \in (0,1)$ such that $\psi(\ttheta) = \xi$.\\
We now show that for all $\zeta \in [\ttheta, 1]$, we have $\etap(\x, T_\zeta) \le \xi$. \\
For any $\zeta \in [\ttheta, 1]$, define $\rho := (\zeta - \ttheta)/(1 - \ttheta) \in [0,1]$ (valid since $\ttheta \in (0,1)$). \\
Note that $T_\zeta = (1-\zeta)T + \zeta S$ can be rewritten as:
\begin{equation*}
    T_\zeta = (1-\rho)T_{\ttheta} + \rho S.
\end{equation*}
Since $\etap(\x, T_{\ttheta}) = \xi > \kappa$, we apply \Cref{lem:eta-reduces} to conclude:
\begin{equation*}
    \etap(\x, T_\zeta) = \etap(\x,(1-\rho)T_{\ttheta} + \rho S) \le \etap(\x,T_{\ttheta}) = \xi
\end{equation*}
This completes the proof.
\qed
\end{proof}


\section{Proof of \Cref{prop:ctr}}

\begin{proof} As the loss $f$ in~\eqref{eq:VarProb} is convex, its proximal operator $\prox_{\rho f}$ is nonexpansive for any $\rho>0$~\cite[Proposition~12.28]{bauschke2011convex}. Consequently, being the composition of the contractive operator $D$ and the nonexpansive operator $\prox_{\rho f}$, the operator $\Sctr$ is itself contractive.
\qed
\end{proof}

\section{Proof of \Cref{prop:ccd}}

\begin{proof}

The linear spline $\Lspline$ is nonexpansive as its slopes are in $[0,1]$. Since $\W$ is also nonexpansive, it follows from~\eqref{eq:gradmodel} that $\nabla\!\Poten$ is $\beta$-Lipschitz, where $\beta=1+\tau$. Moreover, each $\psi_j$ in \eqref{eq:potential} convex, so $\Poten$ is $\mu$-strongly convex with $\mu=\tau$. Hence, if we can show that the gradient step operator \eqref{eq:gradstep_model} of a strongly convex function is contractive, we are done. We invoke the following Lemma to complete the proof. 

\begin{lemma}
\label{lemma:grad_sc}
Let $h : \bbR^n \mapsto \bbR$ be $\beta$-smooth and $\mu$-strongly convex. Then its graident step operator, $G_{\gamma,h} := \cI - \gamma\nabla h$ is contractive $\forall \gamma \in \left(0,2/(\beta+\mu)\right]$ with contraction factor $\kappa = \sqrt{1 - 2\gamma\mu\beta/(\beta + \mu)}$.
\end{lemma}

The proof can also be found in~\cite[Proposition 3]{sabach_first_2017}, but we include a brief argument for completeness.

\begin{proof} Let $\u$ and $\v \in \bbR^n$. 
    \begin{align*}
        G_{\gamma, h}(\u) - G_{\gamma, h}(\v) = (\u-\v) - \gamma(\nabla h(\u) - \nabla h(\v))
    \end{align*}
    Taking norms, squaring, and expanding on both sides, we have
    \begin{align*}
        \norm[]{G_{\gamma, h}(\u) - G_{\gamma, h}(\v)}^2 &= \norm[]{\u - \v}^2 + \gamma^2 \norm[]{\nabla h(\u) - \nabla h(\v)}^2  \\
        &- 2\gamma \langle \u - \v, \nabla h(\u) - \nabla h(\v) \rangle
    \end{align*}
    By \cite[Theorem~2.1.12]{nesterov_lectures_2018}
    \begin{equation*}
        \langle \nabla h(\u) - \nabla h(\v), \u - \v \rangle \geqslant \frac{1}{\beta + \mu} \norm[]{\nabla h(\u) - \nabla h(\v)}^2 + \frac{\beta\mu}{\beta + \mu} \norm[]{\u - \v}^2
    \end{equation*}
    Substituting, we get
    \begin{align*}
    \norm[]{G_{\gamma, h}(\u) - G_{\gamma, h}(\v)}^2 &\leqslant \left(1 - \frac{2\gamma\mu\beta}{\beta + \mu} \right) \norm[]{\u - \v}^2
    \\ &+ \left(\gamma^2 - \frac{2\gamma}{\beta + \mu} \right) \norm[]{\nabla h(\u) - \nabla h(\v)}^2
    \end{align*}
    For $G_{\gamma,h}$ to be contractive, we need
    \begin{align*}
        \gamma \leqslant \frac{\beta+\mu}{2\mu\beta} \text{ and } \gamma \leqslant \frac{2}{\beta+\mu} \Rightarrow \gamma \leqslant \frac{2}{\beta+\mu} \quad\because \forall \beta,\mu\in\bbR : \frac{2}{\beta+\mu} \leqslant \frac{\beta+\mu}{2\mu\beta}
    \end{align*}
    \begin{align*}
        \therefore \forall \gamma \in \left(0,\frac{2}{\beta+\mu}\right] : \norm[]{G_{\gamma, h}(\u) - G_{\gamma, h}(\v)} \leqslant \sqrt{1 - \frac{2\gamma\mu\beta}{\beta + \mu}} \norm[]{\u - \v}
    \end{align*}
    Thus, the mapping $G_{\gamma, h}$ is $\sqrt{1 - 2\gamma\mu\beta/(\beta + \mu)}$ -- contractive.
    \qed
\end{proof}

Thus, by \Cref{lemma:grad_sc}, the gradient-step operator \eqref{eq:gradstep_model}, and hence \eqref{eq:D}, is contractive with contraction factor 
\begin{equation}
\label{eq:ctr_fac}
    \kappa = \sqrt{1 - \frac{2\gamma\tau(1+\tau)}{(1 + 2\tau)}}
\end{equation}
for any $\tau > 0$ and $0<\gamma \leqslant2/(1+2\tau)$.
\qed
\end{proof}

\section{Multi Noise Level Denoiser}
\label{app:arch_det}
To enhance the expressivity of $\Dctr$ in \eqref{eq:D} and induce noise-level awareness, inspired from~\cite{goujon_learning_2024}, we incorporate a trainable noise-dependent scaling function $\ScalFunc$ applied to the spline activations $\Lspline$. With slight abuse of notation, we define the scaled profile, activation, potential, gradient functions and the noise-level aware gradient step denoiser, respectively  to be
\begin{align}
    \label{eq:scal_profile}
    \Qspline_{\sigma} (\x) &= \ScalFunc(\sigma)^{-2}\Qspline\,( \ScalFunc(\sigma) \,\x),\\
    \label{eq:scal_act}
    \Lspline_{\sigma} (\x) =\nabla\!\Qspline_{\sigma}(\x) &= \ScalFunc(\sigma)^{-1}\Lspline\,( \ScalFunc(\sigma) \,\x), \\
    \label{eq:scal_potential}
    \varphi_{\sigma}(\x) &= \sum_{j=1}^{p} \qspline_{{\sigma}_j} \big( (\W \x)_j \big)+ \frac{\tau}{2}\Vert\x\Vert^2,\\
    \label{eq:scal_gradmodel}
    \nabla \! \PotenScal(\x) &= \Lspline_{\sigma} (\x) + \tau\x,\\
    \label{eq:scal_gradstepmodel}
    \Dsig(\x) &= \x - \gamma\nabla \! \PotenScal(\x).
\end{align}

\noindent Following~\cite{goujon_learning_2024}
the scaling function is parameterized as 
\begin{equation}
    \label{eq:scal_func}
    \ScalFunc = (\nu_1\,,\, \ldots \,,\, \nu_p); \quad \text{where } \nu_{i}(\sigma) = \exp(\lsplineScal[i](\sigma))/(\sigma+10^{-5})
\end{equation}
and $\LsplineScal = (\lsplineScal[1] \,,\, \ldots \,,\,\lsplineScal[p])$ are implemented using unconstrained learnable linear splines with 11 equidistant knots in $[\sigma_{\min},\sigma_{\max}]$.

From the definitions of the composite profile and activation functions in \Cref{sec:conden}, together with \eqref{eq:scal_func}, the scaled versions take the form
\begin{align}
    \Qspline_{\sigma}(\x) &= \big(\qspline_{\sigma_1}(x_1),\, \ldots,\, \qspline_{\sigma_p}(x_p)\big),
    \qquad 
    \qspline_{\sigma_j}(x_j) = \nu_j(\sigma)^{-2}\qspline_j\big(\nu_j(\sigma)x_j\big);\\
    \Lspline_{\sigma}(\x) &= \big(\lspline_{\sigma_1}(x_1),\, \ldots,\, \lspline_{\sigma_p}(x_p)\big),
    \qquad 
    \lspline_{\sigma_j}(x_j) = \nu_j(\sigma)^{-1}\lspline_j\big(\nu_j(\sigma)x_j\big).
\end{align}

Since scaling only rescales the inputs and outputs to the nonlinearities, it does not alter the strong convexity of $\PotenScal$ or the contractivity of $\Dsig$ (see the proof of \Cref{prop:dctr} below).

\section{Proof of \Cref{prop:dctr}}

\begin{proof}
It suffices to show that $\varphi_{\sigma}$ is $(1+\tau)$-smooth and $\tau$-strongly convex, since the result then follows exactly as in \Cref{prop:ccd}. This, in turn, reduces to proving that each $\qspline_{\sigma_j}$ is convex and each $\lspline_{\sigma_j}$ is $1$-Lipschitz.

\begin{lemma}
\label{lem:cvx_trans}
Let $Q:\bbR\to\bbR$ be convex. For any $t>0$ and $m,n\in\bbR$, define
\[
Q_t(\cdot):=t^{-m}Q(t^n\;\cdot).
\]
Then $Q_t$ is convex.
\end{lemma}

\begin{proof}
Let $u,v\in\bbR$ and $\theta\in[0,1]$. Then
\begin{align*}
Q_t\big(\theta u+(1-\theta)v\big)
&= t^{-m}Q\Big(t^n\big(\theta u+(1-\theta)v\big)\Big)\\
&= t^{-m}Q\big(\theta t^n u+(1-\theta)t^n v\big)\\
&\le t^{-m}\big(\theta Q(t^n u)+(1-\theta)Q(t^n v)\big)\\
&= \theta Q_t(u)+(1-\theta)Q_t(v).
\end{align*}
\end{proof}

Since each $\qspline_j$ is convex and $\nu_j(\sigma)>0$ for all $\sigma\ge0$ and $j\in[p]$, applying \Cref{lem:cvx_trans} with $m=2$, $n=1$, and $t=\nu_j(\sigma)$ shows that each $\qspline_{\sigma_j}$ and hence $\Qspline_\sigma$ are convex. Therefore $\PotenScal$ is $\tau$-strongly convex.

\begin{lemma}
\label{lem:lips_trans}
Let $H:\bbR\to\bbR$ be $\beta$-Lipschitz. For any $t\neq0$, define
\[
H_t(\cdot):=t^{-1}H(t\;\cdot).
\]
Then $H_t$ is also $\beta$-Lipschitz.
\end{lemma}

\begin{proof}
Let $u,v\in\bbR$. Then
\begin{align*}
|H_t(u)-H_t(v)|
&= \big|t^{-1}\big(H(tu)-H(tv)\big)\big|\\
&= |t|^{-1}\,|H(tu)-H(tv)|\\
&\le \beta |t|^{-1}|t(u-v)|\\
&= \beta |u-v|.
\end{align*}
\end{proof}

Since each $\lspline_j$ has slopes in $[0,1]$, it is $1$-Lipschitz. As $\nu_j(\sigma)>0$ for all $\sigma\ge0$ and $j\in[p]$, \Cref{lem:lips_trans} with $\beta=1$ and $t=\nu_j(\sigma)$ implies that each $\lspline_{\sigma_j}$ is also $1$-Lipschitz. Hence $\Lspline_\sigma$ is nonexpansive.

Therefore, $\nabla \! \PotenScal$ is $(1+\tau)$-Lipschitz and $\PotenScal$ is $\tau$-strongly convex. The conclusion now follows exactly as in \Cref{prop:ccd}. Thus, by \Cref{lemma:grad_sc}, the gradient-step operator \eqref{eq:scal_gradstepmodel}, and hence \eqref{eq:Dsig}, is contractive with $\kappa$ in \eqref{eq:ctr_fac}.  \qed
\end{proof}

\section{Performance of the contractive denoiser}

\subsection{Denoising}
We compare our denoiser with classical denoisers such as TV~\cite{rudin1992nonlinear}, DSG-NLM~\cite{sreehari2016plug}, and BM3D~\cite{dabov2007image} in \Cref{fig:coral_awgn}. Despite using only a single layer, $\Dsig$ preserves fine details and even outperforms BM3D, which is widely regarded as the state-of-the-art among classical denoisers.

\begin{table}[ht]
\centering
\caption{PSNR/SSIM performance on AWGN denoising for CBSD68~\cite{bsd500}, comparing classical methods and deep CNN-based denoisers against the proposed shallow contractive denoiser $\Dsig$ over different noise levels (\(255\times\sigma_{\n}\)).}
\label{tab:awgn-comp}
\setlength\tabcolsep{4pt}
\begin{tabular}{l|c|c|c|c|c}
\toprule
\textbf{Denoiser} & 5 & 10 & 15 & 20 & 25 \\
\midrule
\textbf{TV}~\cite{rudin1992nonlinear} & 36.28/.9531 & 32.03/.8895 & 29.66/.8253 & 28.30/.7912 & 27.13/.7428 \\
\textbf{DSG-NLM}~\cite{sreehari2016plug} & 36.84/.9554 & 32.32/.8896 & 29.88/.8262 & 28.22/.7684 & 26.98/.7170 \\
\textbf{BM3D}~\cite{dabov2007image} & 37.48/.9655 & 33.27/.9195 & 31.01/.8738 & 29.46/.8320 & 28.30/.7943 \\
$\Dsig$ & 37.90/.9747 & 33.89/.9354 & 31.43/.8922 & 29.67/.8485 & 28.28/.8044 \\
\textbf{DnCNN}~\cite{zhang2017beyond}&  39.80/- & 35.82/- & 33.55/- & 32.02/- & 30.87/-  \\
\textbf{DRUNet}~\cite{zhang2021plug} & 40.54/.9815  & 36.34/.9582 & 33.96/.9341 & 32.30/.9098 & 31.00/.8858 \\
\textbf{GSDRUNet}~\cite{hurault2022gradient} & 40.46/.9812 & 36.27/.9577 & 33.91/.9337 & 32.26/.9100 & 30.97/.8863 \\
\bottomrule
\end{tabular}
\end{table}

\begin{figure}[ht]
    \centering
    \subfloat[noisy]{
    \includegraphics[width=\Rthree]{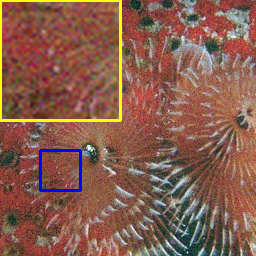}}\hfill
    \subfloat[TV]{
    \includegraphics[width=\Rthree]{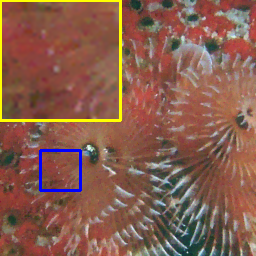}}\hfill
    \subfloat[DSG-NLM]{
    \includegraphics[width=\Rthree]{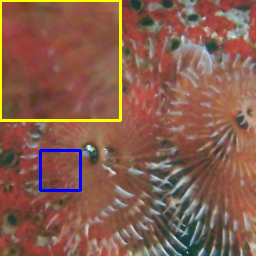}}\hfill
    \subfloat[BM3D]{
    \includegraphics[width=\Rthree]{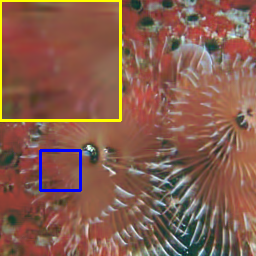}}\hfill
    \subfloat[$\Dsig$]{
    \includegraphics[width=\Rthree]{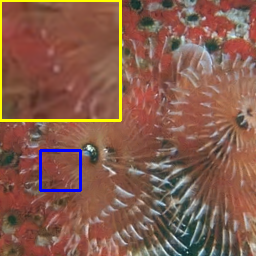}}\hfill
    \subfloat[clean]{
    \includegraphics[width=\Rthree]{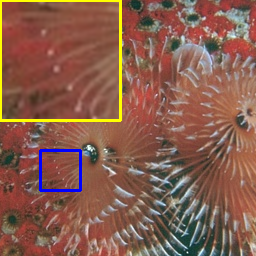}}\hfill
    \caption{Comparison of denoising on \textit{coral} image corrupted with AWGN at  $\sigma_{\n}=15$. The proposed contractive denoiser $\Dsig$ outperforms TV-denoiser, DSG-NLM, and BM3D, the latter regarded as the state of the art among classical denoisers. The magnified regions show that  $\Dsig$ better preserves finer details. The PSNR(dB)/SSIM values are: (a) $ 24.62/0.6823$, (b) $29.20, 0.8446$, (c) $ 29.16,0.8414$, (d) $ 30.34, 0.8718$ and (e) $ 31.44, 0.9048$.}
    \label{fig:coral_awgn}
\end{figure}

\subsection{Inverse problems}

We can also use $\Dsig$ to develop globally convergent reconstruction operator $\Rctr$. 
Instead of the standard PnP-HQS operator~\eqref{eq:pnphqs}, where the proximal step is applied before the denoiser, we adopt the $T_{\mathrm{GS\text{-}PnP}}$~\cite{hurault2022gradient} ordering and apply the denoiser first, followed by the proximal step. This reordering fits the updates into a proximal-gradient framework and simplifies the convergence analysis. Formally, plugging our contractive denoiser $\Dsig$ into this scheme yields the 
contractive reconstruction operator
\begin{equation}
\label{eq:Rctr}
\Rctr := \prox_{\rho f} \circ\, \Dsig,
\end{equation}
where the parameter $\rho>0$ is the penalty used in the proximal update. 
The central result for the contractive PnP update is stated below. This is a standard result; e.g., see~\cite{nair_convergent_2024,hurault2022gradient}.

\begin{proposition}
\label{prop:sc_obj}
The operator $\Rctr$ in~\eqref{eq:sctr} is contractive for any $\rho>0$. For any initialization $\x_0 \in \bbR^n$, the sequence $\{\x_k\}$ generated by the iteration
\begin{equation}
\label{eq:ctr_pnp}
\x_{k+1} = \Rctr(\x_k)
\end{equation}
converges to the unique minimizer of \eqref{eq:VarProb} with $g=\lambda\PotenScal$, where $\lambda>0$.
\end{proposition}
\begin{proof}
Since $\Dsig$ is contractive by \Cref{prop:dctr}, it follows from \Cref{prop:ctr} that the IR operator $\Rctr$ is contractive. Since the sequence $\{\x_k\}$ in \eqref{eq:ctr_pnp} is generated using $\Rctr$, its convergence follows from the contraction mapping theorem~\cite[Theorem~1.50]{bauschke2011convex}.

Consider the composite objective
\begin{equation}
\label{eq:sc_obj}
\min_{x\in\bbR^n}\; J(\x)\;:=\; f(\x)+\lambda\PotenScal(\x).
\end{equation}
Since $f$ is convex and $\PotenScal$ is $\tau$-strongly convex, $J$ is $\lambda\tau$-strongly convex and therefore by~\cite[Corollary~11.17]{bauschke2011convex} admits a unique minimizer,
denoted by $\xsol$.

Recall that $\Dsig$ is of (explicit) gradient-step form
\(
\x-\gamma \nabla \! \PotenScal(\x),
\)
and our PnP iteration is
\[
\x_{k+1}=\Rctr\,(\x_k) \;=\; \prox_{\rho f} \circ\,\Dsig \,(\x_k)
=
\prox_{\rho f}\big(\x_k-\rho\lambda\nabla \! \PotenScal(\x_k)\big),
\]
where $\gamma=\rho\lambda$ with $\lambda>0$ such that $\rho\lambda\leqslant2/(1+2\tau)$. 

Let $\xfp\in\bbR^n$ be the \emph{unique} fixed point of $\Rctr$, i.e.
\[
\xfp \;=\; \Rctr(\xfp) = \prox_{\rho f}\big(\Dsig(\xfp)\big).
\]
The fixed-point relation becomes
\begin{equation}\label{eq:fp_rel}
\xfp \;=\; \prox_{\rho f}\big(\xfp - \rho\lambda\nabla \! \PotenScal(\xfp)\big).
\end{equation}
On the other hand, by the proximal-point characterization~\cite[Proposition~16.44]{bauschke2011convex}
\begin{equation}\label{eq:prox_char}
\u=\prox_{\rho f}(\v) \iff \v-\u\in \partial \rho f(\u).
\end{equation}
Substituting \eqref{eq:fp_rel} into \eqref{eq:prox_char}, i.e., $\u=\x_{\star}$ and $\v=\x_{\star}-\rho\lambda\nabla \! \PotenScal(\xfp)$,  yields
\[
\frac{\xfp - \rho\lambda\nabla \! \PotenScal(\xfp)-\xfp}{\rho}
=
\frac{-\rho\lambda\nabla \! \PotenScal(\xfp)}{\rho}
=
-\lambda\nabla \! \PotenScal(\xfp)
\;\in\; \partial f(\xfp),
\]
or,
\[
0\in \partial f(\xfp)+\lambda\nabla \! \PotenScal(\xfp).
\]
Since $\PotenScal$ is differentiable, $\partial(\lambda\PotenScal)=\{\lambda\nabla \! \PotenScal\}$, and hence by~\cite[Theorem 23.8]{rockafellar1970convex}
\[
0\in \partial\big(f+\lambda\PotenScal\big)(\xfp).
\]
This is exactly the first-order optimality condition for minimizing $J$. Therefore, $\xfp$ is also a minimizer of
$f+\lambda\PotenScal$, and by strong convexity it must equal the unique minimizer, i.e., 
\( \xfp = \xsol.\)
In particular, the unique fixed point $\xfp$ of $\Rctr$ maps to the unique minimizer $\xsol$ of~\eqref{eq:sc_obj}.
\qed
\end{proof}

\input{gc}

As shown in \Cref{fig:watch_recon}, the IR operators $\Sctr$~\eqref{eq:sctr} and $\Rctr$~\eqref{eq:Rctr} recovers sharper edges and finer textures than the TV~\cite{rudin1992nonlinear}, DSG-NLM~\cite{sreehari2016plug} and  BM3D~\cite{dabov2007image} denoisers. Thus, our IR operators using the contractive denoiser $\Dsig$~\eqref{eq:Dsig} beats classical baselines in perceptual quality and metrics. 

Although our single-layer denoiser outperforms classical denoisers, it cannot match the performance of modern deep denoisers based on UNet~\cite{zhang2021plug,choi2021conditioning} or Transformer~\cite{zhang2023practical,zamir2022restormer} architectures. This further motivates the use of the stabilization framework, \Cref{algo:vista}, to combine the stability benefits of our contractive module with the quality of the high-performance prior reconstruction.

\begin{figure}[ht]
  \centering
  \captionsetup[subfloat]{labelformat=empty,labelsep=none,justification=centering}
  \subfloat[blurry]{
  \includegraphics[width=\Rfour]{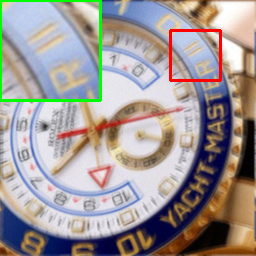}}\hfill
  \subfloat[PGD+TV-reg]{
  \includegraphics[width=\Rfour]{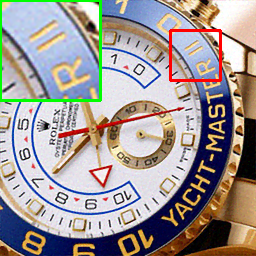}}\hfill
  \subfloat[ADMM + DSG-NLM]{
  \includegraphics[width=\Rfour]{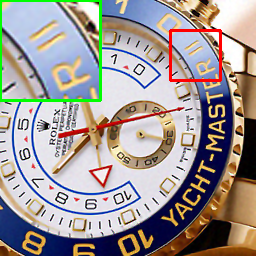}}\hfill
  \subfloat[HQS+BM3D]
  {\includegraphics[width=\Rfour]{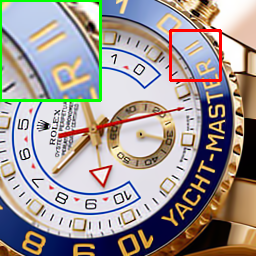}}\hfill
  \subfloat[$\Sctr$ in \eqref{eq:sctr}]{
  \includegraphics[width=\Rfour]{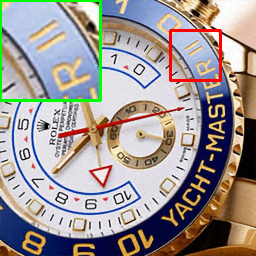}}\hfill
  \subfloat[$\Rctr$ in \eqref{eq:Rctr}]{
  \includegraphics[width=\Rfour]{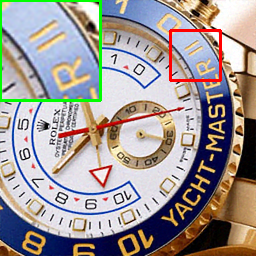}}\hfill
  \subfloat[FBS+DnCNN]
  {\includegraphics[width=\Rfour]{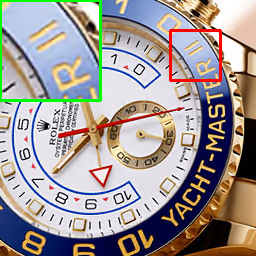}}\hfill
  \subfloat[ADMM+MMO]
  {\includegraphics[width=\Rfour]{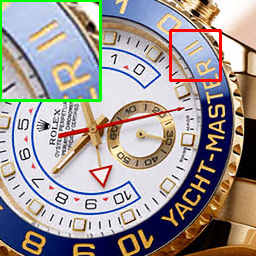}}\hfill
  \subfloat[HQS+DRUNet]
  {\includegraphics[width=\Rfour]{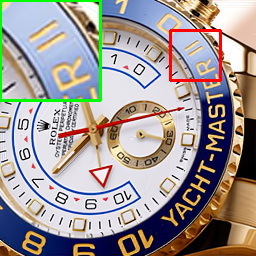}}\hfill
  \subfloat[HQS+DiffUNet]
  {\includegraphics[width=\Rfour]{figures/watch/DRUNet.png}}\hfill
  \subfloat[HQS+GSDRUNet]
  {\includegraphics[width=\Rfour]{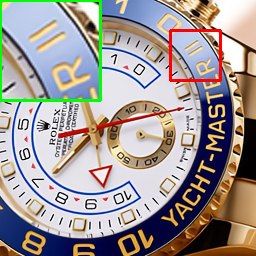}}\hfill
  \subfloat[clean]{
  \includegraphics[width=\Rfour]{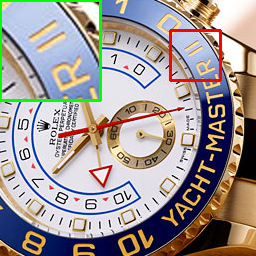}}\hfill
  \caption{Motion deblurring experiment on \textit{watch} image from General100~\cite{general100} with kernel 5~\cite{levin_kernel_2009} and additive noise $\sigma_{\n}=0.01$. The results are using the mentioned PnP framworks and denoisers. The PSNR(dB) values are: (a) $15.36$, (b) $26.01$, (c) $26.47$, (d) $26.98$, (e) $27.44$, (f) $27.78$, (g) $28.57$, (h) $28.19$, (i) $29.90$, (j) $29.12$ and (k) $29.71$.}
  \label{fig:watch_recon}
\end{figure}


\section{Classical early stopping methods}

Early-stopping rules such as the L-curve and discrepancy principle provide natural alternatives for mitigating post-peak degradation. However, these methods only select an iterate from the original trajectory, whereas our approach modifies the dynamics and may also improve the attained peak. We therefore use Equivariant PnP as a primary baseline, since it likewise alters the iterative trajectory.

\begin{figure}[t]
    \centering    \subfloat{\includegraphics[width=0.66\linewidth]{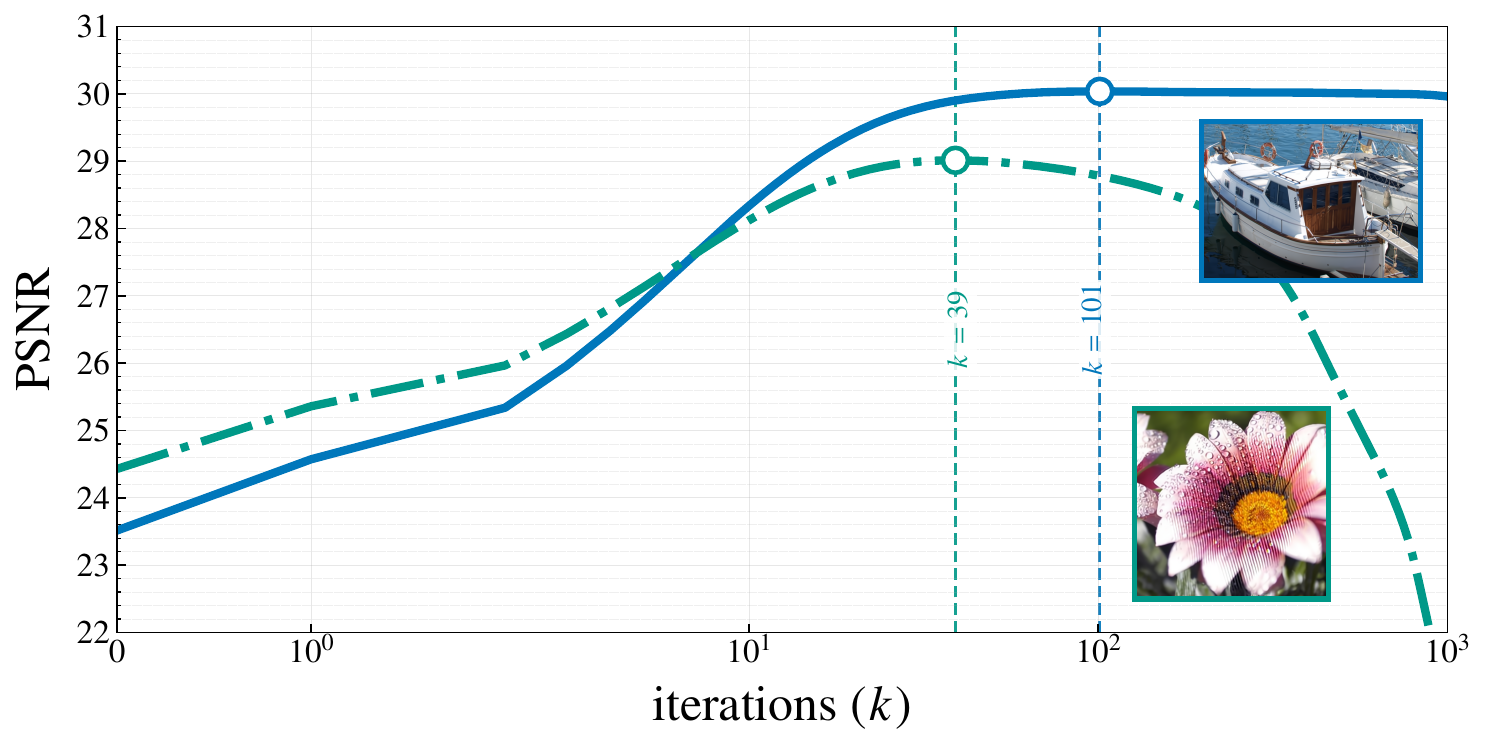}}\hfill   \subfloat{\includegraphics[width=0.32\linewidth]{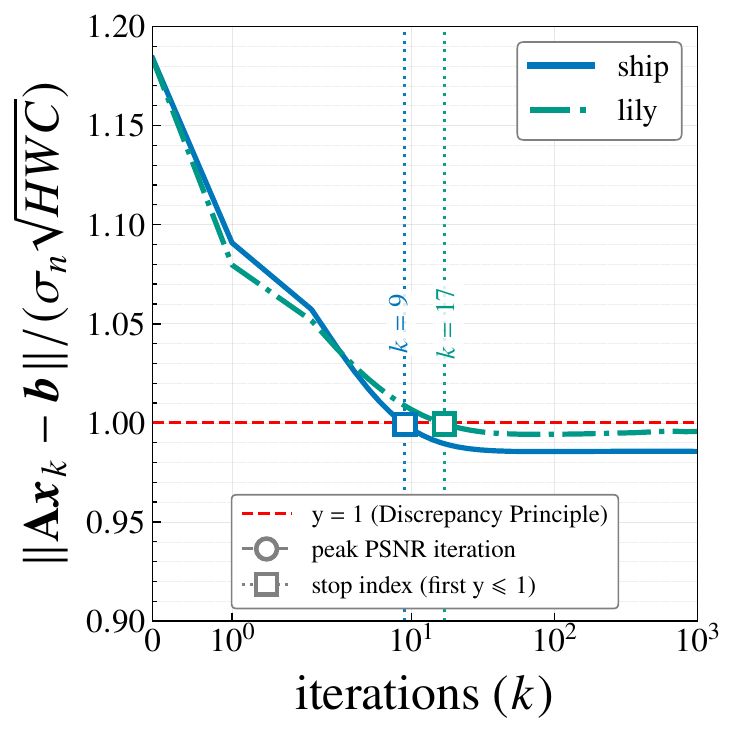}}\hfill
    \caption{Identical PnP setup on two images. The PSNR curves show that the image can strongly change the iterate dynamics. The discrepancy rule triggers too early and does not find the peak. $H, W, C$ denote the height, width and channels of the image respectively.}
    \label{fig:stab_same_sett}
\end{figure}

\begin{table}[t]
\centering
\caption{Peak PSNR (dB) and discrepancy-principle (DP) PSNR (dB) with the attaining iteration in brackets, for deblurring and superresolution over set3c dataset.}
\label{tab:discrepancy}
\setlength{\tabcolsep}{5pt}
\renewcommand{\arraystretch}{1.15}
\begin{tabular}{l l c c c c}
\toprule
& & \multicolumn{2}{c}{\includegraphics[width=\lenBlur]{figures/kernels/blur/8.png}} & 
\multicolumn{2}{c}{\includegraphics[width=\lenBlur]{figures/kernels/SR/1.png}}\\
\multirow{2}{*}{Denoiser} & \multirow{2}{*}{Image}
  & \multicolumn{2}{c}{$s=1$}
  & \multicolumn{2}{c}{$s=2$} \\
\cmidrule(lr){3-4} \cmidrule(lr){5-6}
 & & Peak (it.) & DP (it.) & Peak (it.) & DP (it.) \\
\midrule
\multirow{3}{*}{DnCNN}
  & Butterfly & 25.41 (63) & 22.49 (14) & 28.02 (335) & 25.43 (2) \\
  & Leaves    & 24.29 (73) & 21.34 (18) & 26.76 (173) & 24.86 (3) \\
  & Starfish  & 26.92 (39) & 24.85 (9) & 29.32 (197) & 27.62 (2) \\
\cmidrule(lr){1-6}
\multirow{3}{*}{MMO}
  & Butterfly & 26.45 (113) & 22.80 (15) & 27.67 (999) & 25.45 (2) \\
  & Leaves    & 25.74 (159) & 21.58 (19) & 27.10 (945) & 24.90 (3) \\
  & Starfish  & 27.61 (55) & 25.29 (10) & 29.45 (998) & 27.67 (2) \\
\cmidrule(lr){1-6}
\multirow{3}{*}{DRUNet}
  & Butterfly & 29.54 (992) & 25.04 (12) & 28.90 (177) & 26.23 (2) \\
  & Leaves    & 30.93 (397) & 24.22 (15) & 29.09 (99) & 25.80 (3) \\
  & Starfish  & 29.67 (56) & 27.06 (8) & 30.31 (49) & 28.58 (2) \\
\cmidrule(lr){1-6}
\multirow{3}{*}{GSDRUNet}
  & Butterfly & 30.32 (104) & 24.80 (10) & 30.24 (225) & 26.32 (2) \\
  & Leaves    & 31.33 (105) & 23.87 (12) & 30.47 (163) & 25.98 (3) \\
  & Starfish  & 31.13 (74) & 26.98 (7) & 31.83 (199) & 28.63 (2) \\
\cmidrule(lr){1-6}
\multirow{3}{*}{DiffUNet}
  & Butterfly & 27.82 (39) & 25.31 (12) & 28.65 (57) & 26.17 (2) \\
  & Leaves    & 28.07 (55) & 24.30 (13) & 27.59 (41) & 25.79 (3) \\
  & Starfish  & 28.88 (23) & 27.18 (8) & 30.48 (48) & 28.46 (2) \\
\bottomrule
\end{tabular}
\end{table}

Classical stopping rules can also be difficult to apply in generic black-box PnP settings. The L-curve requires a meaningful regularity measure, while PnP generally has no explicit regularizer or canonical regularity metric. The discrepancy principle, $\norm{\mat A\x_k - \b} \leq \sigma_{\n} \sqrt{HWC}$, requires a reliable noise estimate and, even when the exact noise level $\sigma_{\n}$ is available, need not stop near the PSNR peak; see \Cref{fig:stab_same_sett} (right) and the aggregate results in \Cref{tab:discrepancy}. Likewise, a fixed point iteration, PnP/RED algorithm with fixed step size(s) and denoiser noise level $\sigma_{\cD}$, selected on one validation image may not transfer to another, since the trajectory depends on the input image, as illustrated in \Cref{fig:stab_same_sett} (left).

\section{Ablation Studies}

In this section, we study the contractive denoiser $\Dsig$ and the induced contractive reconstruction operators. 
Since the only component that varies in \Cref{algo:vista} is the contractive anchor, and this anchor is fully determined by the choice of $\Dsig$ through \eqref{eq:sctr} or \eqref{eq:Rctr}, the natural ablation axis for \Cref{algo:vista} is the architecture of $\Dsig$. 
We therefore first examine several architectural variants of $\Dsig$, summarized in \Cref{tab:dsig-varaints}, and then study the resulting reconstruction operators both in standalone mode and as anchors in \Cref{algo:vista}. In \Cref{tab:dsig-varaints},


\begin{itemize}
    \item $\gamma$ denotes the gradient step size;
    \item \emph{MC} indicates the use of multi-convolution blocks~\cite{goujon2023neural}, 3 conv layers of output channels 12, 24 and 128 channels were used;
    \item \emph{LLS} denotes learnable linear splines~\cite{goujon2023neural}, else ReLU;
    \item \emph{NLS} denotes noise-level scaling~\cite{goujon_learning_2024};
    \item \emph{NM} denotes the use of a noise map~\cite{zhang2021plug} -- appending an additional constant channel $\sigma_{\n} \mat 1$ to the input, where $\mat 1 \in \bbR^n$ is an image of unit intensity;
    \item \emph{SR} denotes single parameter rescaling (applicable only to the \emph{MC} setting);
    \item \emph{P} is the number of trainable parameters. 
\end{itemize}
For these variants, we report comprehensive denoising results to assess the contribution of the individual design choices.

\begin{table}[htpb]
\centering
\caption{Ablation variants of the contractive denoiser $\Dsig$. We vary the gradient step size $\gamma$, the use of multi-convolution (MC), learnable linear splines (LLS), noise-level scaling (NLS), noise map (NM), and single rescaling (SR); $P$ reports the number of trainable parameters. }
\label{tab:dsig-varaints}
\setlength\tabcolsep{5pt}
\begin{tabular}{l|c|c|c|c|c|c|c}
\toprule
\textbf{Variant} & $\gamma$ & MC & LLS & NLS & NM & SR & P\\
\midrule
V0A & $1$ & \xmark & \xmark & \xmark & \cmark & --- & 2.3\,k\\
V0B & $1$ & \xmark & \xmark & \cmark & \xmark & --- & 2.4\,k\\
\addlinespace[1.5pt] \hline \addlinespace[1.5pt]
V1A & $(1+2\tau)^{-1}$ & \xmark & \cmark & \cmark & \cmark & --- & 9.5\,k\\
V1B & $(1+2\tau)^{-1}$ & \xmark & \cmark & \xmark & \cmark & --- & 8.8\,k\\
V1C & $(1+2\tau)^{-1}$ & \xmark & \cmark & \cmark & \xmark & --- & 8.9\,k\\
\addlinespace[1.5pt] \hline \addlinespace[1.5pt]
V1D & $1$ & \xmark & \cmark & \cmark & \cmark & --- & 9.5\,k\\
V1E & $1$ & \xmark & \cmark & \xmark & \cmark & --- & 8.8\,k\\
V1F & $1$ & \xmark & \cmark & \cmark & \xmark & --- & 8.9\,k\\
\addlinespace[1.5pt] \hline \addlinespace[1.5pt]
V2A & $1$ & \cmark & \cmark & \cmark & \cmark & \xmark & 98\,k\\
V2B & $1$ & \cmark & \cmark & \cmark & \xmark & \xmark & 126\,k\\
\addlinespace[1.5pt] \hline \addlinespace[1.5pt]
V3A & $1$ & \cmark & \cmark & \cmark & \cmark & \cmark & 98\,k\\
V3B & $(1+2\tau)^{-1}$ & \cmark & \cmark & \cmark & \cmark & \cmark & 98\,k\\
V3C & $1$ & \cmark & \cmark & \cmark & \xmark & \cmark & 125.7\,k\\
\bottomrule
\end{tabular}
\end{table}

We then study the corresponding contractive IR operators $\Sctr$~\eqref{eq:sctr} and $\Rctr$~\eqref{eq:Rctr} on inverse problems, namely deblurring and superresolution. 
We first compare their standalone reconstruction performance, and then evaluate them as anchors within the proposed stabilization framework across different PnP algorithms and pretrained denoisers, justifying our choice.

\subsection{Contractive denoiser $\Dsig$}

\Cref{tab:awgn-abl} reports denoising PSNR(dB)/SSIM on CBSD68 for the variants listed in \Cref{tab:dsig-varaints}. 
Replacing ReLU with learnable linear splines already yields a clear gain over the minimal baselines V0A/V0B, while moving from a single convolution to the multi-convolution design gives a further improvement, with V2/V3 variants consistently outperforming the V1 family across all noise levels. 
Multi layer rescaling when used with noise map provide smaller but consistent gains, with the V2 family generally matching or slightly improving over V3.



\begin{table}[htpb]
\centering
\caption{PSNR(dB)/SSIM performance on AWGN denoising for CBSD68~\cite{bsd500}, reporting ablation studies across different variants of the proposed shallow contractive denoiser $\Dsig$ over different noise levels (\(255\times\sigma_{\n}\)).}
\label{tab:awgn-abl}
\setlength\tabcolsep{5pt}
\begin{tabular}{l|c|c|c|c|c}
\toprule
\textbf{Variant} & 5 & 10 & 15 & 20 & 25 \\
\midrule
V0A & 33.96/.9561 & 31.33/.8859 & 28.77/.7967 & 26.60/.7093 & 24.81/.6322 \\
V0B & 33.58/.9506 & 30.69/.8656 & 27.97/.7638 & 25.71/.6688 & 23.86/.5876 \\
\addlinespace[1.5pt] \hline \addlinespace[1.5pt]
V1A & 37.72/.9741 & 33.89/.9351 & 31.48/.8908 & 29.72/.8444 & 28.32/.7983\\
V1B & 35.47/.9615 & 33.19/.9260 & 31.17/.8819 & 29.54/.8361 & 28.01/.7791\\
V1C & 37.90/.9747 & 33.88/.9354 & 31.43/.8922 & 29.67/.8484 & 28.28/.8043 \\
\addlinespace[1.5pt] \hline \addlinespace[1.5pt]
V1D & 37.74/.9741 & 33.90/.9351 & 31.49/.8905 & 29.72/.8439 & 28.31/.7975\\
V1E & 35.47/.9615 & 33.19/.9260 & 31.17/.8818 & 29.54/.8362 & 28.02/.7796 \\
V1F & 37.90/.9747 & 33.89/.9354 & 31.43/.8922 & 29.67/.8485 & 28.28/.8044 \\
\addlinespace[1.5pt] \hline \addlinespace[1.5pt]
V2A & 38.04/.9760 & 34.28/.9430 & 31.97/.9073 & 30.30/.8703 & 28.99/.8332\\
V2B & 37.85/.9762 & 34.11/.9435 & 31.85/.9080 & 30.22/.8711 & 28.93/.8339\\
\addlinespace[1.5pt] \hline \addlinespace[1.5pt]
V3A & 37.84/.9761 & 34.10/.9434 & 31.84/.9079 & 30.21/.8711 & 28.93/.8341\\
V3B & 37.85/.9761 & 34.11/.9434 & 31.85/.9080 & 30.22/.8712 & 28.93/.8342 \\
V3C & 37.85/.9761 & 34.11/.9437 &  31.85/.9082 & 30.22/.8716 &  28.94/.8347\\
\bottomrule
\end{tabular}
\end{table}

\subsection{Contractive IR operators $\Sctr$ vs $\Rctr$}

In the main paper, we use the lightweight \emph{V1C} configuration for $\Dsig$. Although its standalone denoising and reconstruction performance is lower than that of the heavier V2/V3 variants, it is more efficient and, more importantly, performs strongly inside the stabilization framework. In particular, as shown in the next subsection, V1C often matches or even surpasses the larger variants when used as an anchor in \Cref{algo:vista}. This motivates our choice of V1C in the main experiments.

\subsubsection{Standalone.}

\Cref{tab:cbsd10_dsig_standalone} compares the standalone performance of $\Sctr$ and $\Rctr$ using variants V1C, V2A, and V3C. 
Overall, $\Sctr$ is stronger for both deblurring tasks and for $2\times$ superresolution, while $\Rctr$ is slightly better for $3\times$ and $4\times$ superresolution. 
Among all variants, V2A in $\Sctr$ gives the best standalone PSNR for Gaussian deblurring, motion deblurring, and $2\times$ superresolution, whereas V2A in $\Rctr$ is best for $3\times$ and $4\times$ superresolution. 
The next subsection shows that standalone performance alone is not a sufficient criterion for selecting the anchor in \Cref{algo:vista}.

\begin{table*}[htpb]
\centering
\caption{PSNR(dB) results for deblurring and superresolution using variants noted in \Cref{tab:dsig-varaints} using IR operators $\Sctr$ and $\Rctr$ on CBSD10 ($\sigma_{\n}=0.02$).}
\label{tab:cbsd10_dsig_standalone}
\setlength\tabcolsep{5pt}
\begin{tabular}{lccccc}
\toprule
\multirow{3}{*}{\textbf{IR Operator}} &
\multirow{2}{*}{
{\includegraphics[width=\lenSR]{figures/kernels/blur/9.png}}} &  
\includegraphics[width=\lenSR]{figures/kernels/blur/1.png}
\includegraphics[width=\lenSR]{figures/kernels/blur/2.png}
\includegraphics[width=\lenSR]{figures/kernels/blur/3.png}
\includegraphics[width=\lenSR]
{figures/kernels/blur/4.png} 
&
\multicolumn{3}{c}{\multirow{2}{*}{
\raisebox{\height}{
\includegraphics[width=\lenSR]{figures/kernels/SR/1.png}
\includegraphics[width=\lenSR]{figures/kernels/SR/2.png}
\includegraphics[width=\lenSR]{figures/kernels/SR/3.png}
\includegraphics[width=\lenSR]{figures/kernels/SR/4.png}
}}}\\
& & 
\includegraphics[width=\lenSR]{figures/kernels/blur/5.png}
\includegraphics[width=\lenSR]{figures/kernels/blur/6.png}
\includegraphics[width=\lenSR]{figures/kernels/blur/7.png}
\includegraphics[width=\lenSR]{figures/kernels/blur/8.png}
\\
\addlinespace[1.5pt] \hline \addlinespace[1.5pt]
& \multicolumn{2}{c}{$s=1$} &  $s=2$ & $s=3$ & $s=4$\\
\addlinespace[1.5pt] \hline \addlinespace[1.5pt]
Start & $24.15$ & $20.38$ & $24.28$ & $22.68$ & $21.40$\\
\midrule
V1C in $\Sctr$ & \underline{27.55} & 28.71 & 26.84 & 24.97 & 22.77\\
V2A in $\Sctr$ & \textbf{27.56} & \textbf{29.02} & \textbf{26.99} & \underline{25.23} & \underline{23.41}\\
V3C in $\Sctr$  & 27.50 & \underline{28.98} & 26.89 & 25.05 & 22.98 \\
\addlinespace[1.5pt] \hline \addlinespace[1.5pt]
V1C in $\Rctr$ & 27.46 & 27.78 & 26.75 & 25.06 & 22.90\\
V2A in $\Rctr$ & 27.50 & 28.22 & \underline{26.92} & \textbf{25.33} & \textbf{23.54}\\
V3C in $\Rctr$ & 27.49 & 28.22 & 26.89 & 25.21 & 23.17 \\
\bottomrule
\end{tabular}
\end{table*}

\subsubsection{Anchor in \Cref{algo:vista}.}

\Cref{tab:set3c_vista_abl} shows that, unlike the standalone setting, the choice among V1C, V2A, and V3C has only a minor effect once the corresponding operators are used as anchors in \Cref{algo:vista}; all variants perform very similarly. Thus, the stronger standalone performance of the heavier multi-convolution variants does not translate into a clear advantage inside the stabilization framework. This supports our use of the lightweight V1C configuration in the main paper. We do not use $\Dsig$ alone as the anchor because, although it stabilizes the iterates, its performance remains close to the vanilla peak. In contrast, $\Sctr$ and $\Rctr$ yield better reconstructions, which we attribute to their additional data-fidelity step.
We use $\Sctr$ rather than $\Rctr$ as the anchor mainly for simplicity of motivation and presentation: although $\Rctr$ is tied to the minimization of~\eqref{eq:sc_obj}, that objective-level interpretation is not needed for \Cref{algo:vista}, whereas the standard HQS-based operator $\Sctr$ provides a simpler and equally effective anchor.

\begin{table*}[htpb]
\centering
\caption{PSNR results (mean $\pm$ std.\ dev.) on set3c ($\sigma_{\n}=0.02$). All expriments were conducted using PnP-PGD~\eqref{eq:pnppgd} + DRUNet~\cite{zhang2021plug} denoiser. \xmark~indicates divergence of the iterates.}
\label{tab:set3c_vista_abl}
\setlength\tabcolsep{4pt}
\tiny
\begin{tabular}{p{0.05cm}ccccccc}
\toprule
& & \multicolumn{2}{c}{\includegraphics[width=\lenBlur]{figures/kernels/blur/3.png}} & \multicolumn{2}{c}{\includegraphics[width=\lenBlur]{figures/kernels/blur/9.png}} &
\multicolumn{2}{c}{\includegraphics[width=\lenBlur]{figures/kernels/SR/3.png}}\\
& & \multicolumn{4}{c}{$s=1$} & \multicolumn{2}{c}{$s=2$} \\
\cmidrule(lr){3-4} \cmidrule(lr){5-6} \cmidrule(lr){7-8}
& \textbf{Start} & \multicolumn{2}{c}{$19.00\pm1.97$} & \multicolumn{2}{c}{$21.75\pm1.78$} &\multicolumn{2}{c}{$21.19\pm1.85$} \\
\midrule
 & & Peak & Final & Peak & Final & Peak & Final \\
\midrule
& \textbf{Vanilla} & $28.99\pm0.32$ & $28.37\pm1.04$ & $26.60\pm1.01$ & \xmark & $25.54\pm1.09$ & \xmark\\
\addlinespace[1.5pt] \hline \addlinespace[1.5pt]
\multirow{9}{*}{\rotatebox[origin=c]{90}{\Cref{algo:vista}}} & \textbf{$\Dsig$ with V1C} & $28.35\pm0.48$ & $28.35\pm0.48$ & $26.68\pm0.98$ & $26.68\pm0.98$ & $25.55\pm1.08$ & $25.54\pm1.11$\\
\addlinespace[1.5pt] \cline{2-8} \addlinespace[1.5pt]
 & \textbf{$\Sctr$ with V1C} & $\mathbf{29.18\pm0.22}$ & $\mathbf{29.18\pm0.22}$ & $\mathbf{26.85\pm0.89}$ & $\mathbf{26.85\pm0.89}$ & \underline{$25.74\pm1.05$} & \underline{$25.74\pm1.05$}\\
\addlinespace[2pt]
 & \textbf{$\Sctr$ with V2A} & $29.15\pm0.25$ & $29.14\pm0.24$ & $26.84\pm0.93$ & $26.84\pm0.93$ & $25.69\pm1.08$ & $25.69\pm1.09$\\
\addlinespace[2pt]
& \textbf{$\Sctr$ with V3C} & $29.14\pm0.26$ & $29.13\pm0.24$ & $26.83\pm0.93$ & $26.83\pm0.94$ & $25.67\pm1.09$ & $25.67\pm1.10$\\
\addlinespace[1.5pt] \cline{2-8} \addlinespace[1.5pt]
& \textbf{$\Rctr$ with V1C} & \underline{$29.17\pm0.25$} & \underline{$29.17\pm0.24$} & \underline{$26.85\pm0.90$} & \underline{$26.85\pm0.90$} & $\mathbf{25.77\pm1.05}$ & $\mathbf{25.77\pm1.05}$\\
\addlinespace[2pt]
& \textbf{$\Rctr$ with V2A} & $29.17\pm0.26$ & $29.16\pm0.25$ & $26.84\pm0.92$ & $26.84\pm0.92$ & $25.72\pm1.07$ & $25.72\pm1.07$ \\
\addlinespace[2pt]
& \textbf{$\Rctr$ with V3C} & $29.16\pm0.26$ & $29.16\pm0.25$ & $26.83\pm0.93$ & $26.83\pm0.93$ & $25.70\pm1.08$ & $25.70\pm1.09$\\
\bottomrule
\end{tabular}
\end{table*}

\section{Additional results}

\subsection{Stability landscapes for Equivariant-PnP}

To complement \Cref{fig:rot} in the main paper, we visualize the corresponding stability landscapes for Equivariant-PnP~\cite{terris2024equivariant} in \Cref{fig:rot_equiv}. 
These plots show that, although equivariant averaging can enlarge the stable region in some settings, it does not eliminate peak-and-collapse behaviour altogether. 
Thus, Equivariant-PnP can delay instability, but remains sensitive to parameter choices and does not provide the same reliability as the proposed stabilization framework.

\begin{figure}[ht]
\centering
\captionsetup[subfloat]{labelformat=empty,labelsep=none}
\subfloat[ADMM+DiffUNet]{\includegraphics[width=\Rthree]{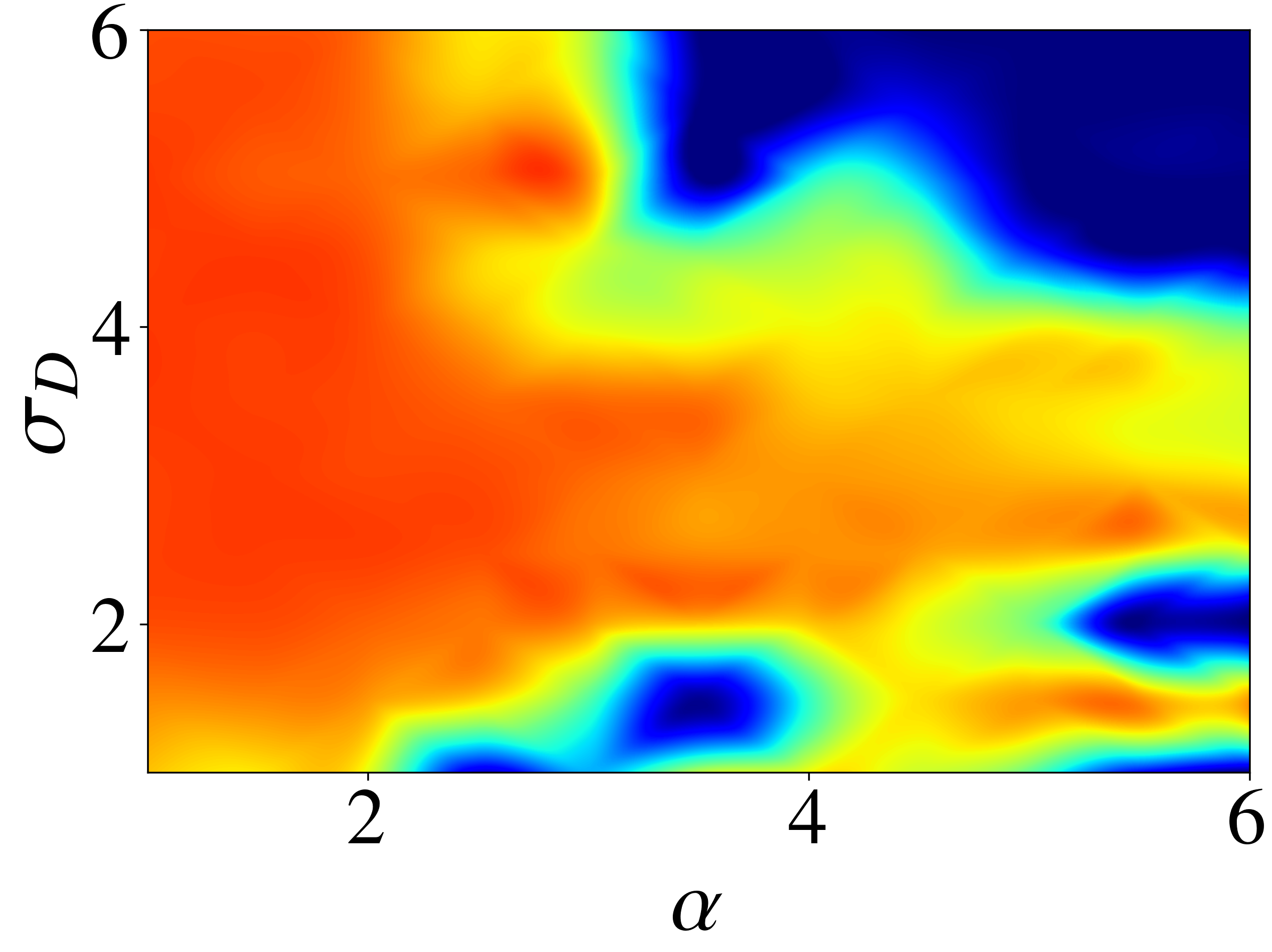}}\hfill
\subfloat[PGD+DRUNet]{\includegraphics[width=\Rthree]{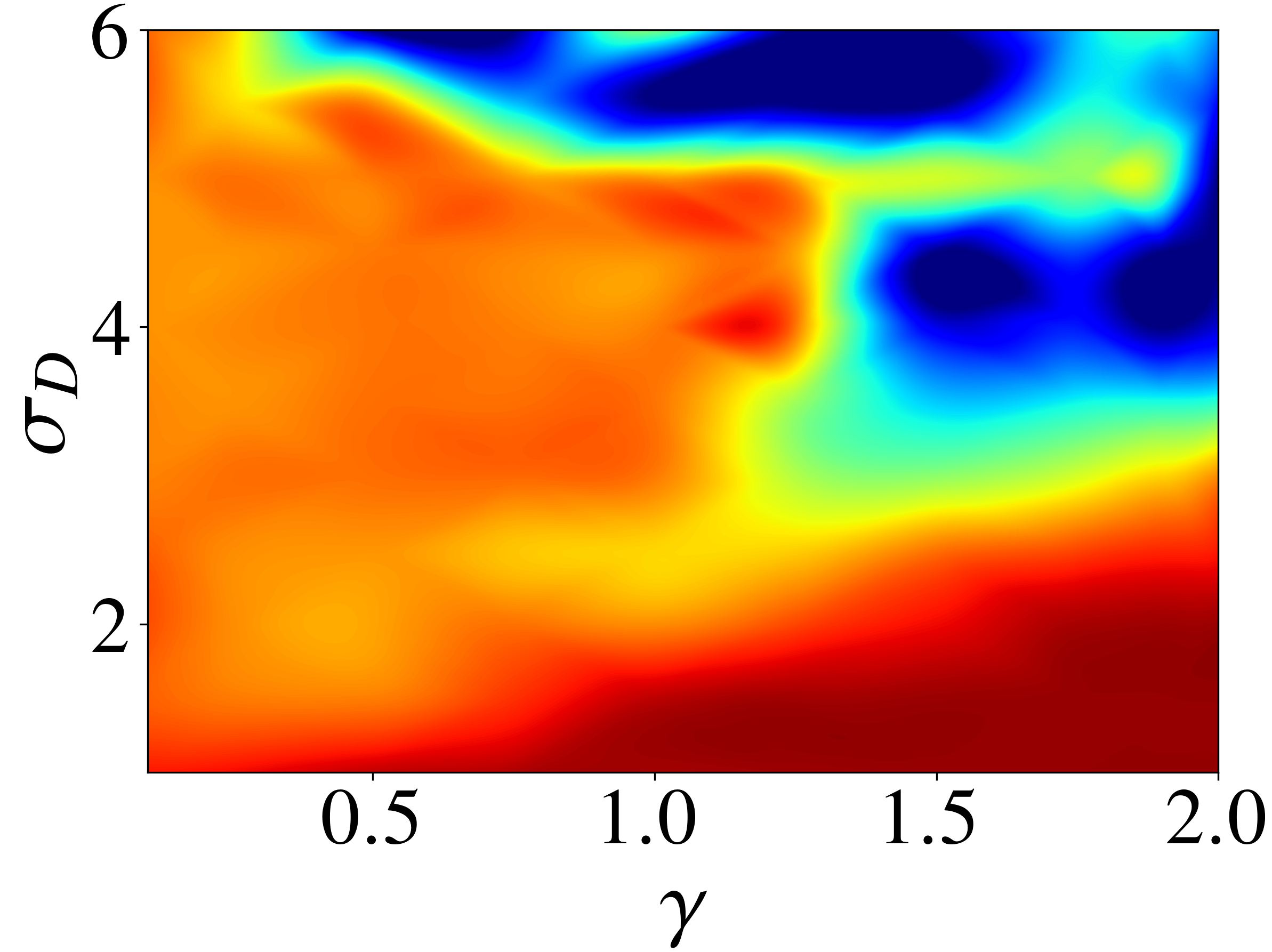}}\hfill
\subfloat[HQS+DnCNN]{\includegraphics[width=\Rthree]{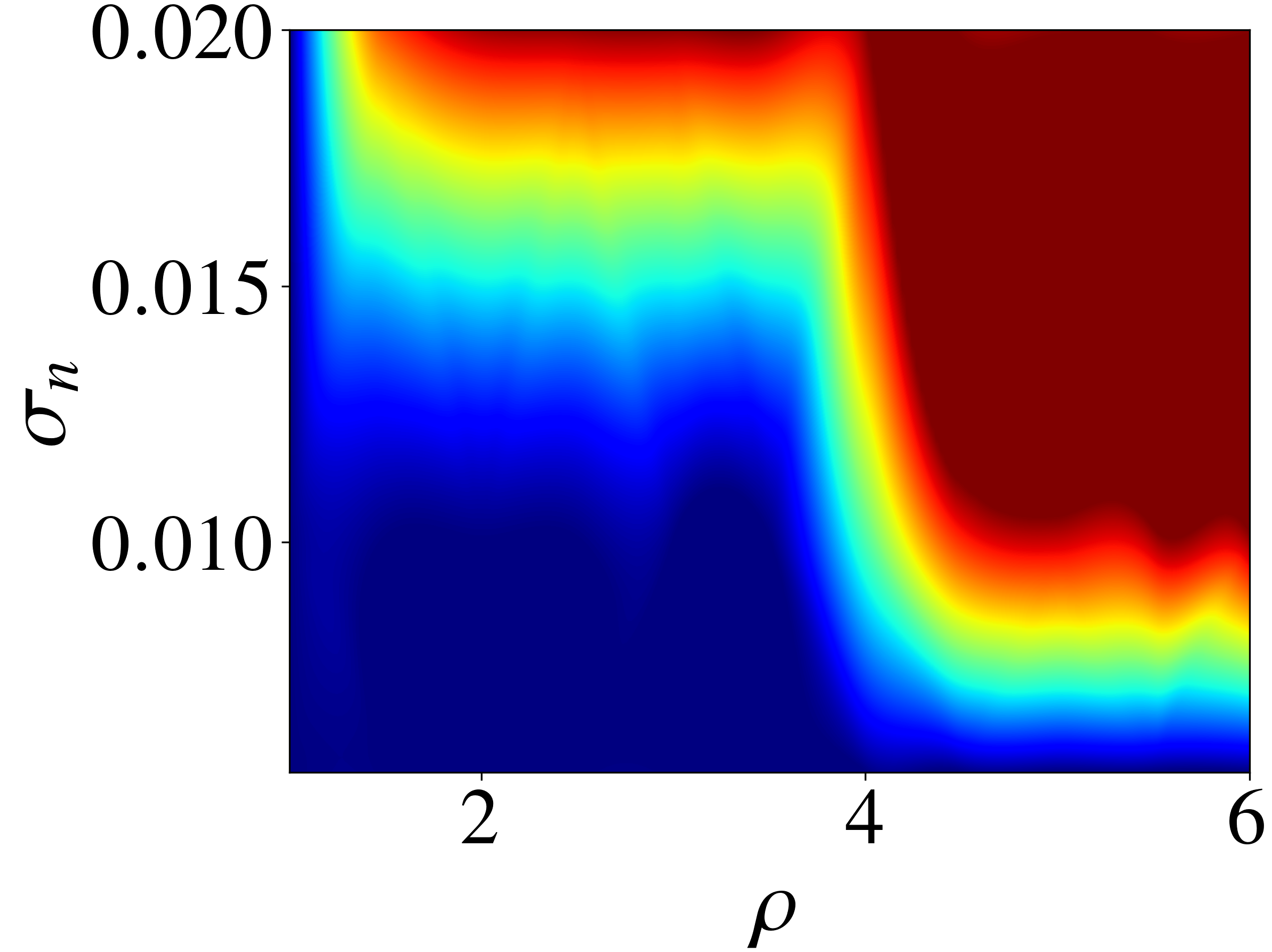
}}\hfill
\subfloat[ADMM+GSDRUNet]{\includegraphics[width=\Rthree]{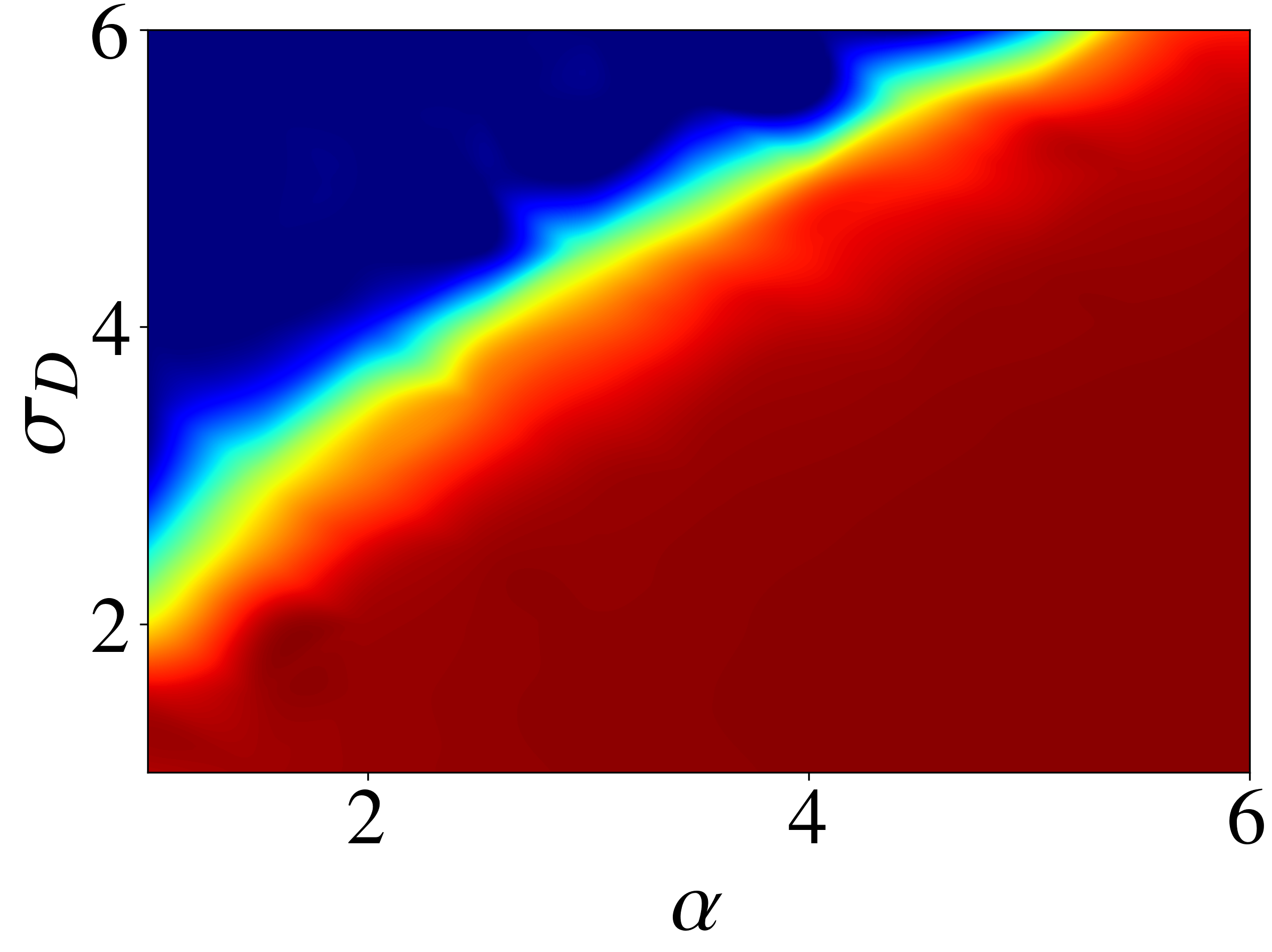}}\hfill
\subfloat[HQS+DRUNet]{\includegraphics[width=\Rthree]{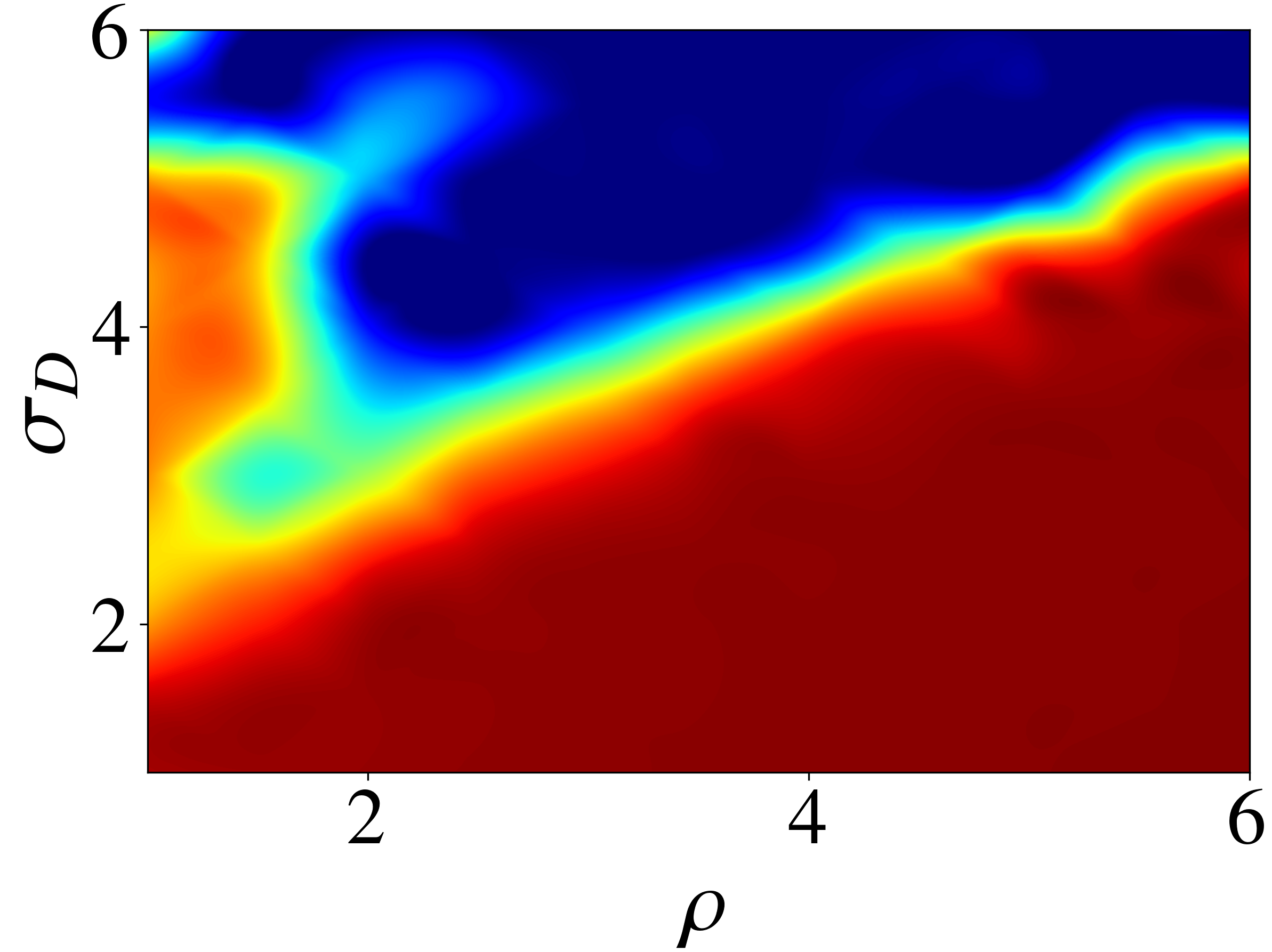}}\hfill
\subfloat[PGD+MMO]{\includegraphics[width=\Rthree]{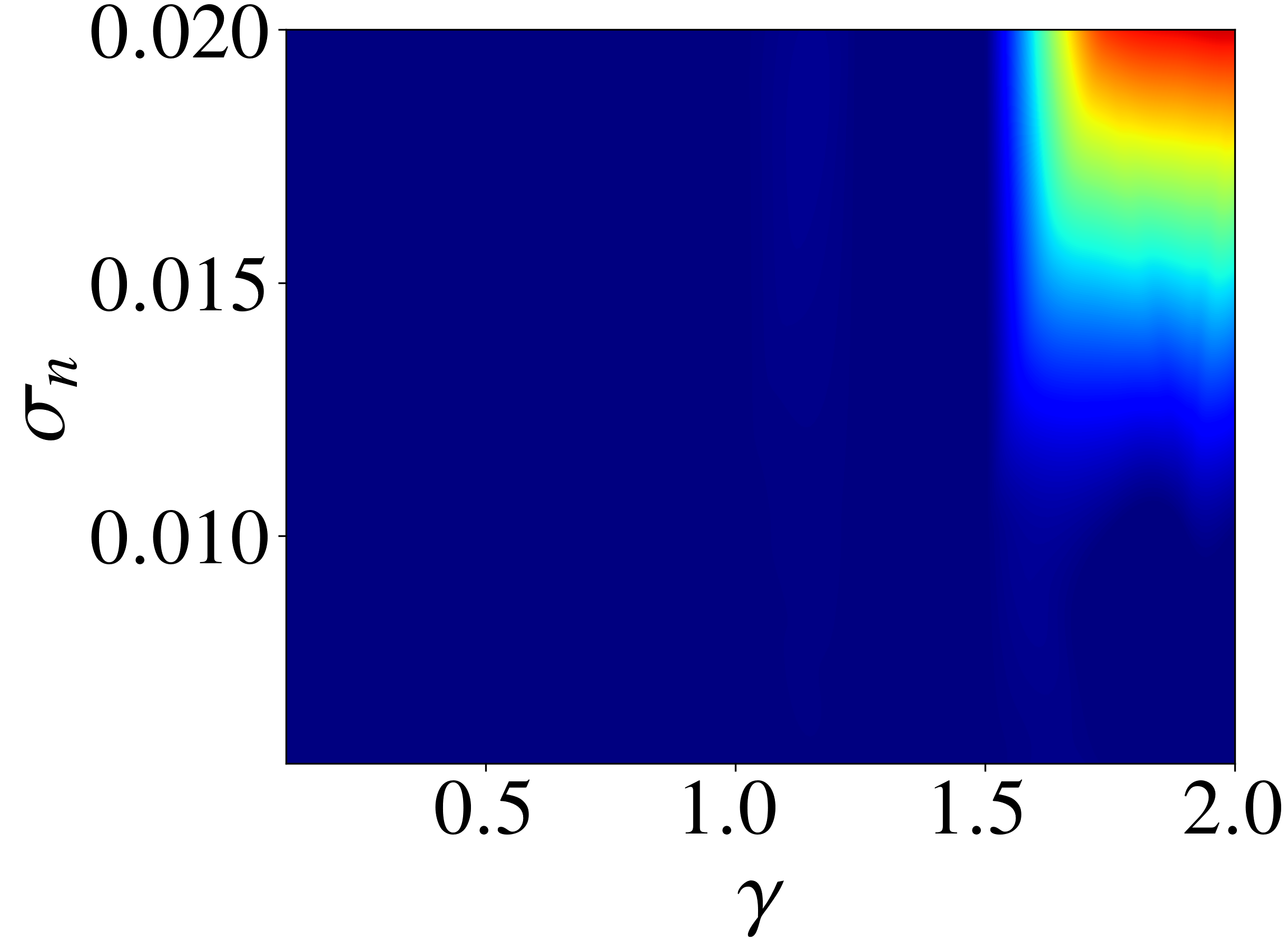
}}\hfill

\caption{Stability regions for Equivariant-PnP, complementing \Cref{fig:rot}. While Equivariant-PnP can improve stability relative to Vanilla-PnP in some parameter regimes, instability remains common.}
\label{fig:rot_equiv}
\end{figure}

\subsection{Computational overhead} 

As the stabilizer $\Sctr$ uses the lightweight shallow denoiser $\Dsig$, the additional per-iteration overhead remains modest, at approximately $15$–$20\%$ across the evaluated settings; see \Cref{tab:time_comp}. The one-time computation of the anchor fixed point $\p$ takes about $1.1$s on average and is negligible when amortized over long runs.

\begin{table}[ht]
\centering
\small
\setlength\tabcolsep{4.5pt}
\caption{Mean \emph{per} iteration computational times (in s) comparison.}
\label{tab:time_comp}
\begin{tabular}{lccc}
\toprule
\textbf{Setting} & \textbf{Vanilla} & \textbf{Equiv} & \textbf{Ours} \\
\midrule
\parbox{8cm}{PnP-HQS~\eqref{eq:pnphqs} + DRUNet~\cite{zhang2021plug}, motion deblur, $400 \times 400$ }
& 0.0478 & 0.0490 & 0.0569  \\
\addlinespace[1.5pt] \hline \addlinespace[1.5pt]
\parbox{8cm}{PnP-PGD~\eqref{eq:pnppgd} + DiffUNet~\cite{choi2021conditioning}, Gaussian deblur, $256 \times 256$ }
& 0.0438 & 0.0471 & 0.0512  \\
\addlinespace[1.5pt] \hline \addlinespace[1.5pt]
\parbox{8cm}{RED-GD~\eqref{eq:redgd} + GSDRUNet~\cite{hurault2022gradient}, $3\times$ SR, $420 \times 360$} 
& 0.0746 & 0.0759 & 0.0885  \\
\bottomrule
\end{tabular}
\end{table}

\subsection{Visual comparisons}
We provide additional qualitative results to illustrate the behaviour of the proposed stabilization framework across different inverse problems, denoisers, and PnP frameworks. 
\Cref{fig:dncnn-castle,fig:gsdrunet-glass,fig:diffunet-leaves,fig:mmo-corals} show representative examples for Gaussian deblurring, motion deblurring, and superresolution using DnCNN~\cite{zhang2017beyond}, DRUNet~\cite{zhang2021plug}, GSDRUNet~\cite{hurault2022gradient}, DiffUNet~\cite{choi2021conditioning}, and MMO~\cite{pesquet_learning_2021} denoisers. 
Across all these settings, Vanilla-PnP and often Equivariant-PnP can achieve strong intermediate reconstructions but later deteriorate, exhibiting the peak-and-collapse (PC) behaviour discussed in the main text, reiterating that the PC behaviour is not specific to a particular task, denoiser, or PnP framework.

In contrast, our method improves and  stabilizes the reconstruction quality without the abrupt degradation seen in the unstabilized baselines. The final output remains close to the best iterate avoiding the need for early stopping. The corresponding PSNR trajectories in \Cref{fig:psnr_plots_app} confirm the same phenomenon quantitatively. \Cref{fig:knee_mri} further demonstrates the same effect in MRI where convergence is deemed to be critical: Vanilla-PnP degrades substantially after its peak, whereas our stabilized reconstruction remains competitive compared to the convergence motivated baselines WCRR~\cite{goujon_learning_2024} and DEAL~\cite{pourya2025dealing}.



For completeness, we also include visual results with Restormer~\cite{zamir2022restormer} and SCUNet~\cite{zhang2023practical}; see \Cref{fig:vanilla-scyscrapper,fig:equiv-skyscraper}. 
These results further demonstrate that the PC phenomenon persists even with recent high-capacity architectures, including Transformer-based models, and that the proposed stabilization mechanism remains effective in these settings as well. 

Overall the additional experiments reinforce that instability is a common feature of deep PnP frameworks, whereas the proposed framework consistently yields reliable reconstructions through a stabilized iterative process.

\begin{figure}[t]
    \captionsetup[subfloat]{labelformat=empty,labelsep=none,justification=centering}
    \centering
    \subfloat[zero-fill ($\A^\top \b$)]{
    \begin{overpic}[width=0.22\linewidth]{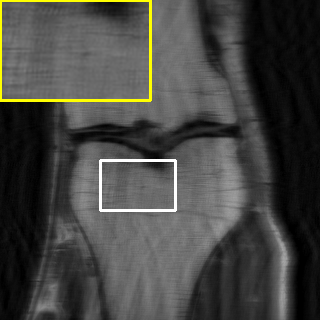}
    \put(4,4){\color{white}\tiny\bfseries PSNR: 29.19 dB}
    \end{overpic}
    }\hfill
    \subfloat[Vanilla (Peak)]{
    \begin{overpic}[width=0.22\linewidth]{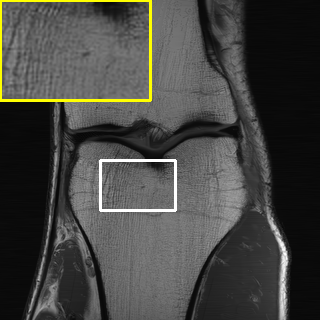}
    \put(4,4){\color{white}\tiny\bfseries PSNR: 42.70 dB}
    \end{overpic}
    }\hfill
    \subfloat[Vanilla (Final)]{
    \begin{overpic}[width=0.22\linewidth]{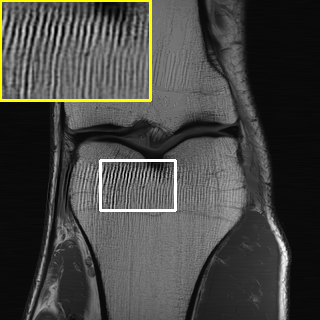}
    \put(4,4){\color{white}\tiny\bfseries PSNR: 33.22 dB}
    \end{overpic}
    }
    \subfloat[$\p$]{
    \begin{overpic}[width=0.22\linewidth]{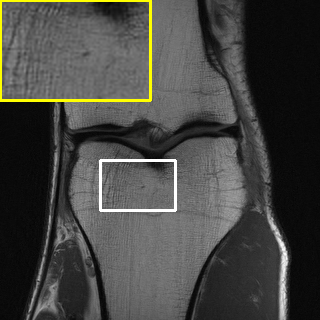}
    \put(4,4){\color{white}\tiny\bfseries PSNR: 42.33 dB}
    \end{overpic}
    }
    \\
    \subfloat[WCRR]{
    \begin{overpic}[width=0.22\linewidth]{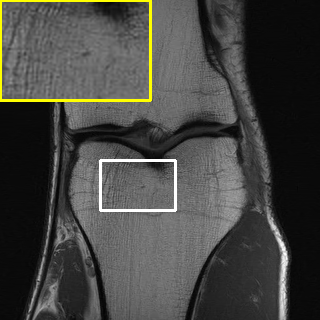}
    \put(4,4){\color{white}\tiny\bfseries PSNR: 42.13 dB}
    \end{overpic}
    }\hfill
    \subfloat[DEAL]{
    \begin{overpic}[width=0.22\linewidth]{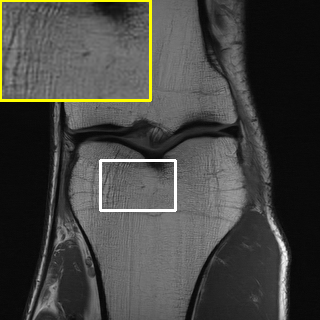}
    \put(4,4){\color{white}\tiny\bfseries PSNR: 43.28 dB}
    \end{overpic}
    }\hfill
    \subfloat[Ours]{
    \begin{overpic}[width=0.22\linewidth]{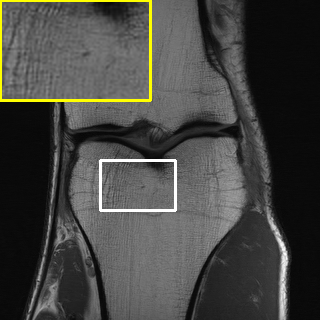}
    \put(4,4){\color{white}\tiny\bfseries PSNR: 43.76 dB}
    \end{overpic}
    }\hfill
    \subfloat[clean]{
    \begin{overpic}[width=0.22\linewidth]{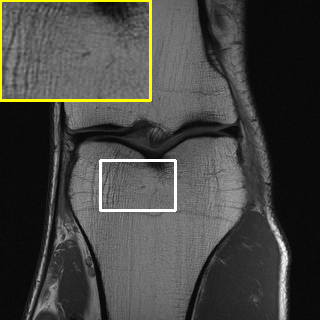}
    \end{overpic}
    }\hfill
    \caption{Knee MRI reconstruction experiment with additive noise $\sigma_{\n} = 0.002$ using RED-GD~\eqref{eq:redgd} + DRUNet~\cite{zhang2021plug}. Results on par with baselines WCRR~\cite{goujon_learning_2024} and DEAL~\cite{pourya2025dealing}.}
    \label{fig:knee_mri}
\end{figure}

\begin{figure}[ht]
    \centering
    \subfloat[blurry]{
    \includegraphics[width=\Rfour]{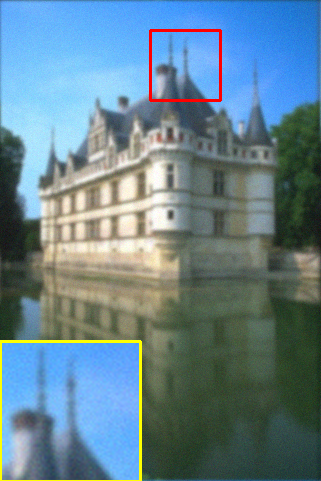}}\hfill
    \subfloat[Vanilla (iter $=1000$)]{
    \includegraphics[width=\Rfour]{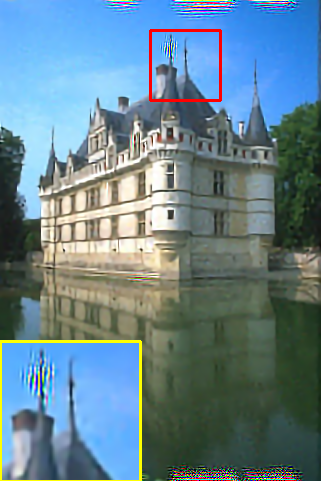}}\hfill
    \subfloat[Equiv (iter $=1000$)]{
    \includegraphics[width=\Rfour]{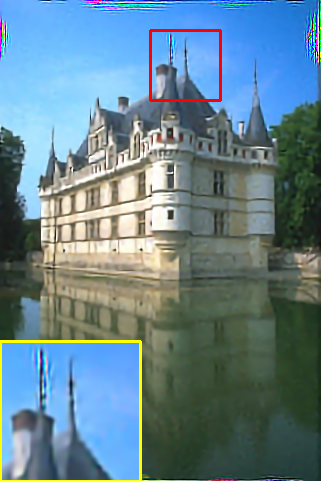}}\hfill
    \subfloat[Ours]{
    \includegraphics[width=\Rfour]{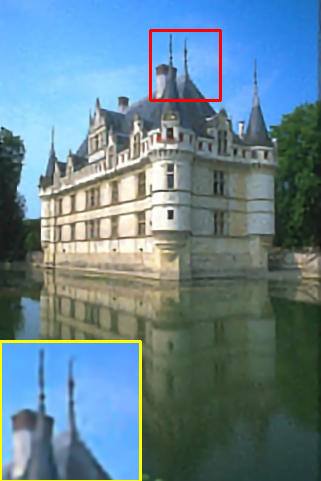}}\hfill
    \subfloat[Vanilla (best)]{
    \includegraphics[width=\Rfour]{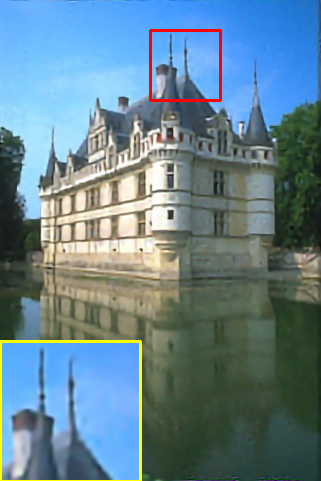}}\hfill
    \subfloat[Equiv (best)]{
    \includegraphics[width=\Rfour]{figures/castle/vanilla_best.png}}\hfill
    \subfloat[$\p$]{
    \includegraphics[width=\Rfour]{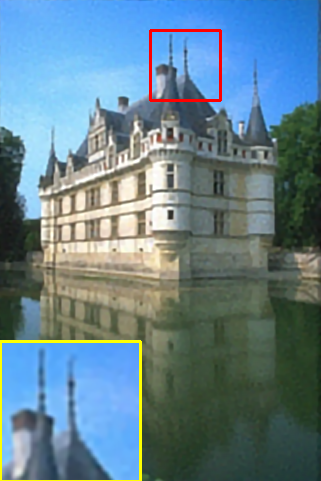}}\hfill
    \subfloat[clean]{
    \includegraphics[width=\Rfour]{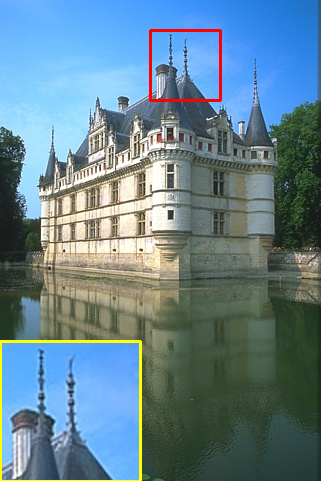}}\hfill
    \caption{Gaussian deblurring results on \textit{castle} from CBSD68 using PnP-PGD~\eqref{eq:pnppgd} + DnCNN~\cite{zhang2017beyond}. We used Gaussian blur with standard deviation $1.6$ and additive noise with $\sigma_{\n}=0.02$. The PSNR(dB) are: (a) $23.94$, (b) $15.32$, (c) $15.74$, (d) $26.94$, (e) $26.88$, (f) $26.57$, and (g) $26.42$.}
    \label{fig:dncnn-castle}
\end{figure}

\begin{figure}[ht]
    \centering
    \subfloat[bicubic]{
    \includegraphics[width=\Rfour]{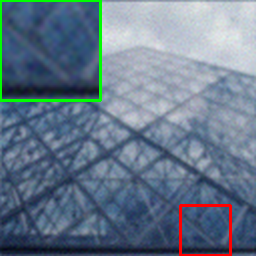}}\hfill
    \subfloat[Vanilla (iter $=1000$)]{
    \includegraphics[width=\Rfour]{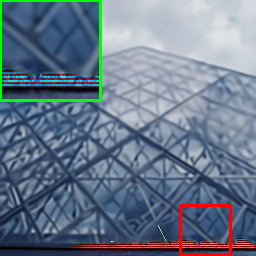}}\hfill
    \subfloat[Equiv (iter $=1000$)]{
    \includegraphics[width=\Rfour]{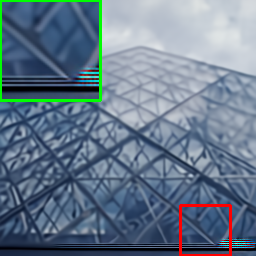}}\hfill
    \subfloat[Ours]{
    \includegraphics[width=\Rfour]{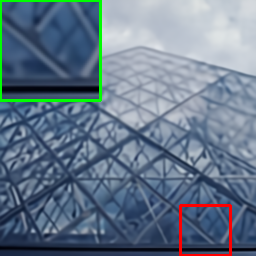}}\hfill
    \subfloat[Vanilla (best)]{
    \includegraphics[width=\Rfour]{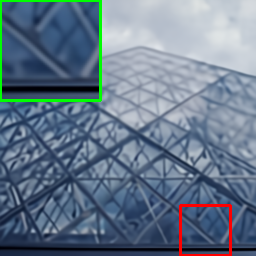}}\hfill
    \subfloat[Equiv (best)]{
    \includegraphics[width=\Rfour]{figures/glass/vanilla_best.png}}\hfill
    \subfloat[$\p$]{
    \includegraphics[width=\Rfour]{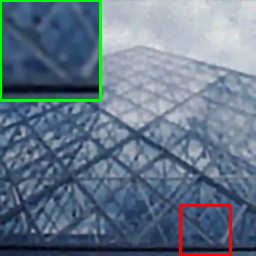}}\hfill
    \subfloat[clean]{
    \includegraphics[width=\Rfour]{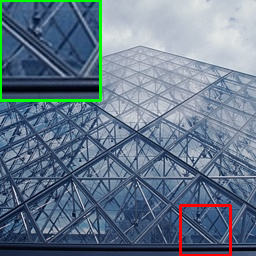}}\hfill
    \caption{$2\times$ superresolution results on \textit{glassdome} from CBSD68 using PnP-HQS~\eqref{eq:pnphqs} + GSDRUNet~\cite{hurault2022gradient}. We used Gaussian blur with standard deviation $1.6$ and additive noise with $\sigma_{\n}=0.02$. The PSNR(dB) are: (a) $21.24$, (b) $20.21$, (c) $20.74$, (d) $23.03$, (e) $23.07$, (f) $23.10$, and (g) $22.42$.}
    \label{fig:gsdrunet-glass}
\end{figure}

\begin{figure}[ht]
    \centering
    \subfloat[blurry]{
    \includegraphics[width=\Rfour]{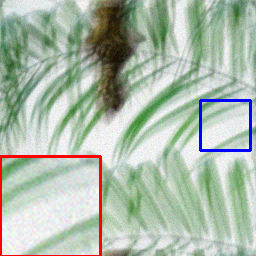}}\hfill
    \subfloat[Vanilla (iter $=1000$)]{
    \includegraphics[width=\Rfour]{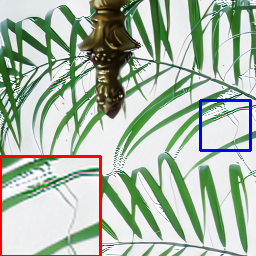}}\hfill
    \subfloat[Equiv (iter $=1000$)]{
    \includegraphics[width=\Rfour]{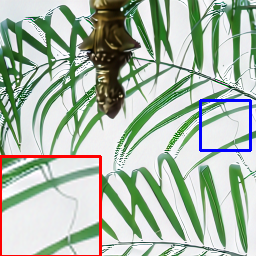}}\hfill
    \subfloat[Ours]{
    \includegraphics[width=\Rfour]{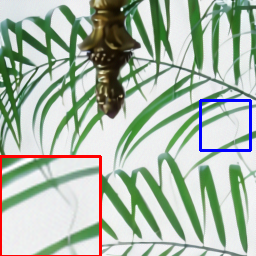}}\hfill
    \subfloat[Vanilla (best)]{
    \includegraphics[width=\Rfour]{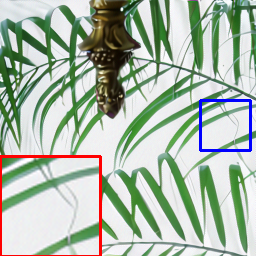}}\hfill
    \subfloat[Equiv (best)]{
    \includegraphics[width=\Rfour]{figures/leaves/vanilla_best.png}}\hfill
    \subfloat[$\p$]{
    \includegraphics[width=\Rfour]{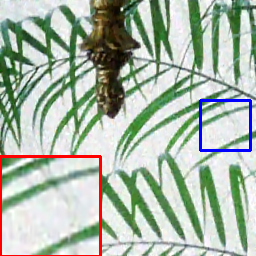}}\hfill
    \subfloat[clean]{
    \includegraphics[width=\Rfour]{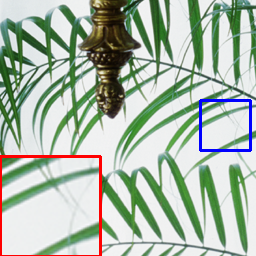}}\hfill
    \caption{Motion deblurring results on \textit{leaves} from set3c using PnP-HQS~\eqref{eq:pnphqs} + DiffUNet~\cite{choi2021conditioning}. We used kernel 8~\cite{levin_kernel_2009} and additive noise with $\sigma_{\n}=0.03$. The PSNR(dB) are: (a) $12.51$, (b) $23.73$, (c) $23.01$, (d) $29.23$, (e) $27.25$, (f) $27.10$, and (g) $24.94$.}
    \label{fig:diffunet-leaves}
\end{figure}

\begin{figure}[ht]
    \centering
    \subfloat[blurry]{
    \includegraphics[width=\Rfour]{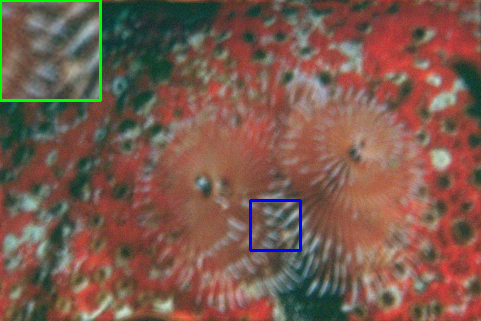}}\hfill
    \subfloat[Vanilla (iter $=1000$)]{
    \includegraphics[width=\Rfour]{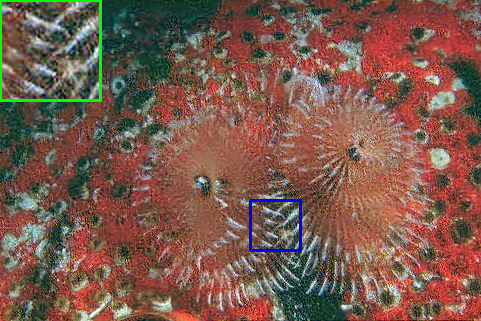}}\hfill
    \subfloat[Equiv (iter $=1000$)]{
    \includegraphics[width=\Rfour]{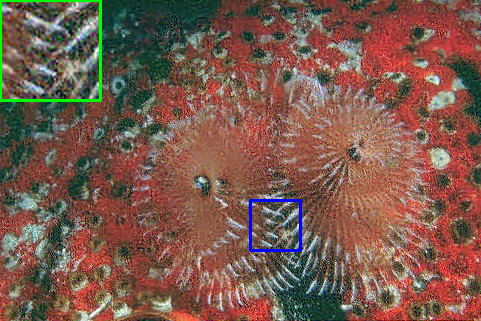}}\hfill
    \subfloat[Ours]{
    \includegraphics[width=\Rfour]{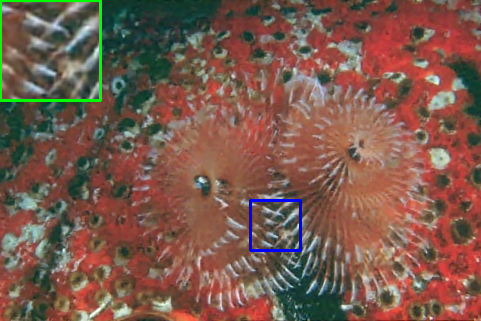}}\hfill
    \subfloat[Vanilla (best)]{
    \includegraphics[width=\Rfour]{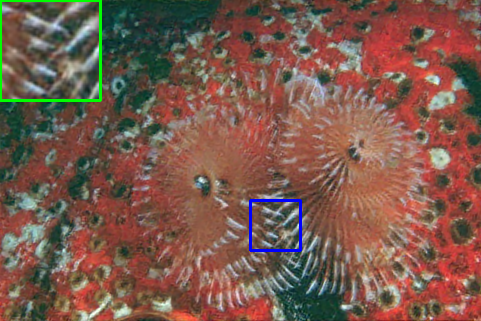}}\hfill
    \subfloat[Equiv (best)]{
    \includegraphics[width=\Rfour]{figures/corals/vanilla_best.png}}\hfill
    \subfloat[$\p$]{
    \includegraphics[width=\Rfour]{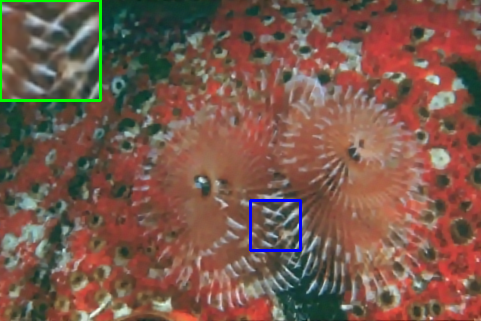}}\hfill
    \subfloat[clean]{
    \includegraphics[width=\Rfour]{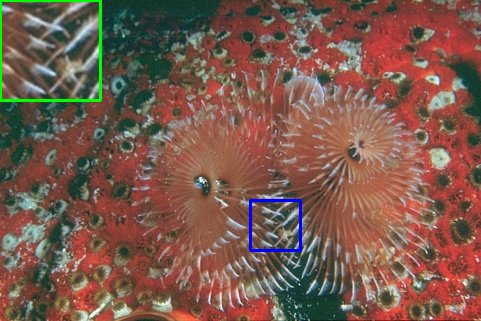}}\hfill
    \caption{Motion deblurring results on \textit{corals} from CBSD68 using PnP-PGD~\eqref{eq:pnppgd} + MMO~\cite{pesquet_learning_2021}. We used kernel 3~\cite{levin_kernel_2009} and additive noise with $\sigma_{\n}=0.02$. The PSNR(dB) are: (a) $22.69$, (b) $22.37$, (c) $22.48$, (d) $29.19$, (e) $28.37$, (f) $28.37$, and (g) $28.81$.}
    \label{fig:mmo-corals}
\end{figure}

\begin{figure}[ht]
    \centering
    \subfloat[PSNR plots for \Cref{fig:dncnn-castle}.]{
    \label{fig:castle-psnr}
    \includegraphics[width=\Rtwo]{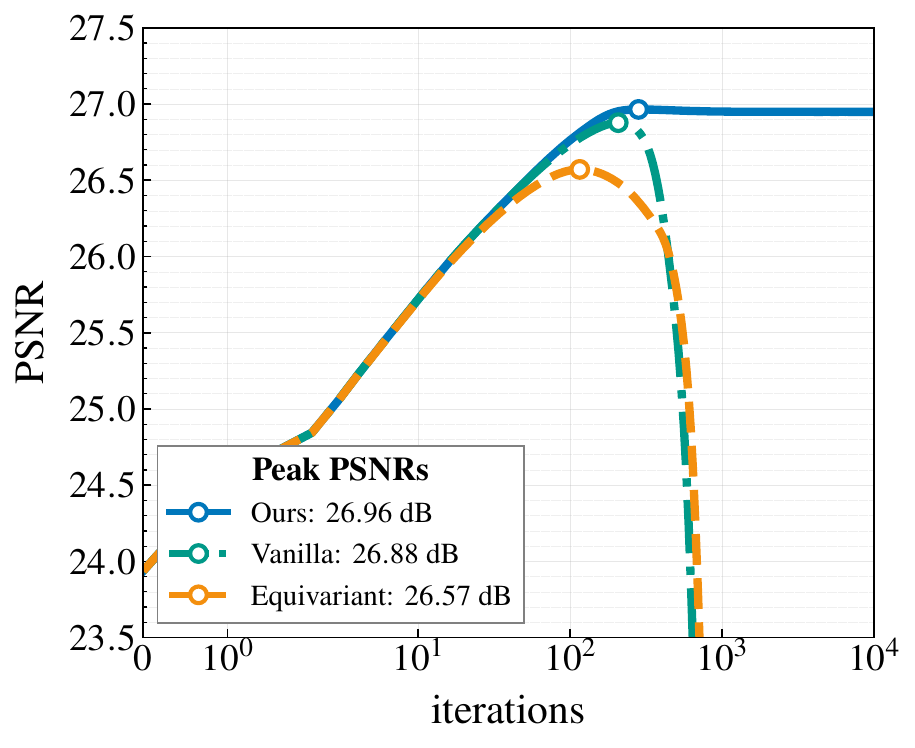}
    }\hfill
    \subfloat[PSNR plots for \Cref{fig:gsdrunet-glass}.]{
    \label{fig:glass-psnr}
    \includegraphics[width=\Rtwo]{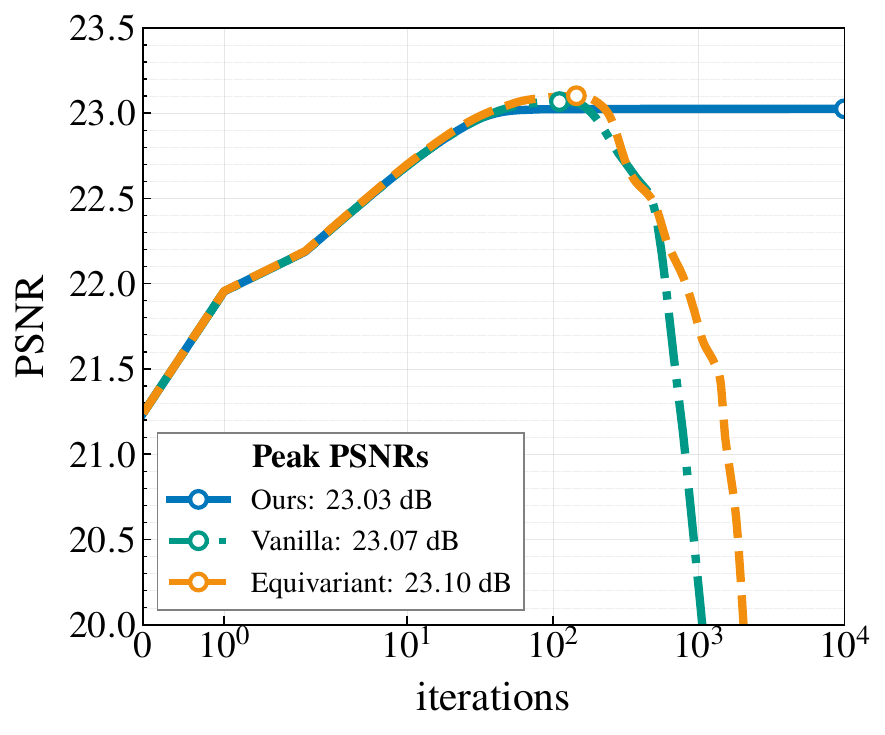}
    }\hfill
    \subfloat[PSNR plots for \Cref{fig:diffunet-leaves}.]{
    \label{fig:leaves-psnr}
    \includegraphics[width=\Rtwo]{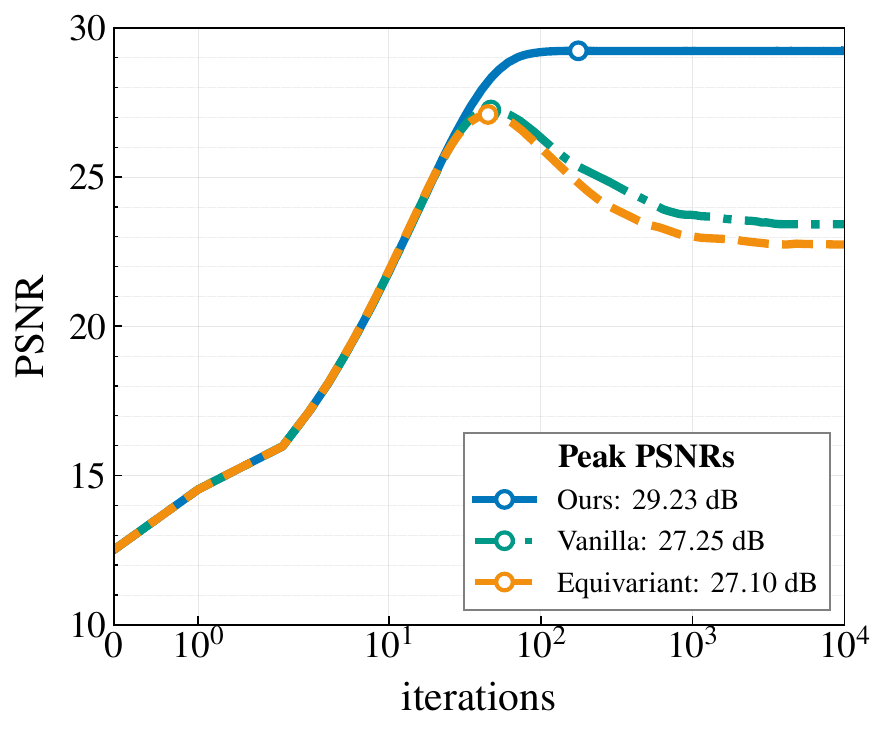}
    }\hfill
    \subfloat[PSNR plots for \Cref{fig:mmo-corals}.]{
    \label{fig:coral-psnr}
    \includegraphics[width=\Rtwo]{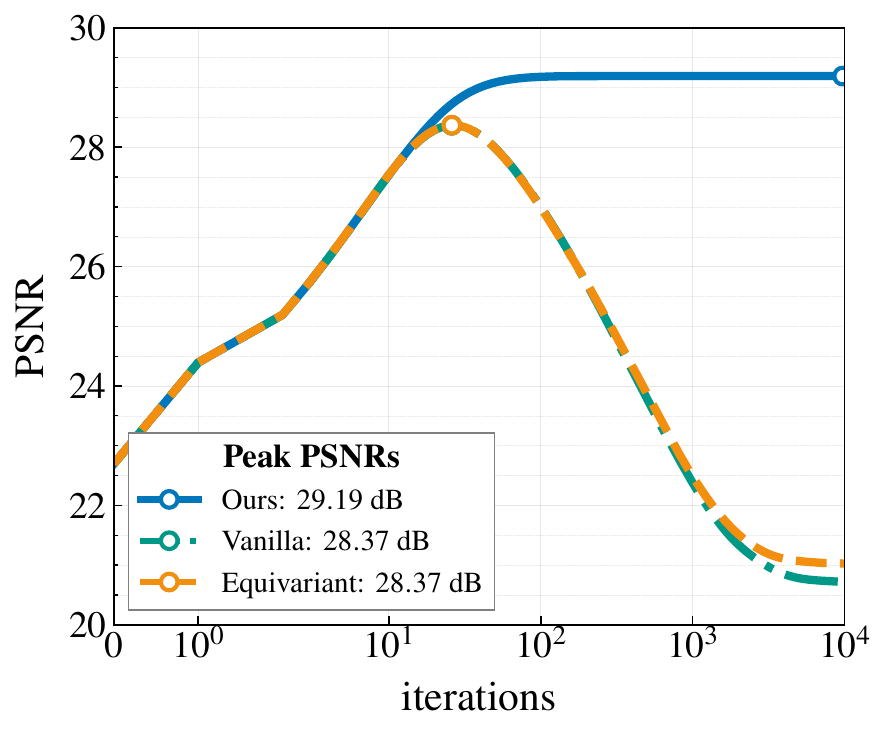}
    }\hfill
    \caption{Vanilla-PnP and Equivariant-PnP might reach high peak PSNR but quite often collapse, showing instability. Our algorithm remains stable and delivers high quality reconstructions.}
    \label{fig:psnr_plots_app}
\end{figure}

\begin{figure}
    \centering
     \captionsetup[subfloat]{labelformat=empty,labelsep=none,justification=centering}
    \subfloat[original]{
    \includegraphics[width=0.30\linewidth]{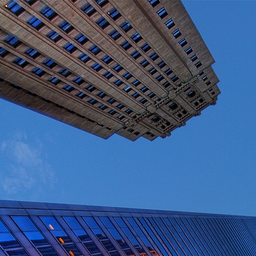}}\hspace{1em}
     \subfloat[blurry]{
    \includegraphics[width=0.30\linewidth]{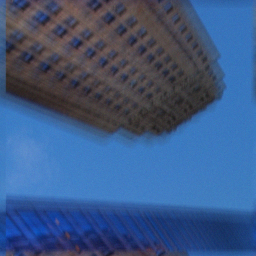}}\\
    
    \begin{sideways}{\parbox{1.2in}{\centering{Vanilla (Peak)}}}\end{sideways}
     \subfloat[DRUNet]{
    \includegraphics[width=0.23\columnwidth]{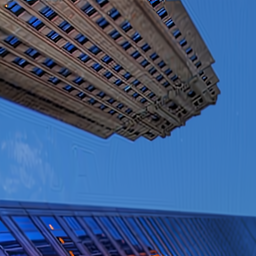}}\hfill
    \subfloat[DiffUNet]{
    \includegraphics[width=0.23\columnwidth]{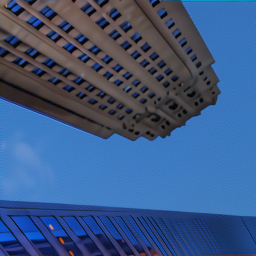}}\hfill
     \subfloat[Restormer]{
    \includegraphics[width=0.23\columnwidth]{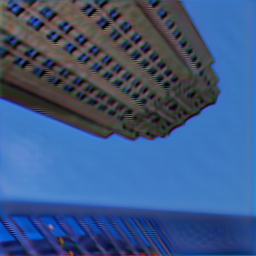}}\hfill
     \subfloat[SCUNet]{
    \includegraphics[width=0.23\columnwidth]{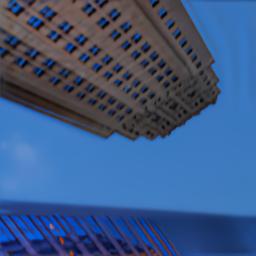}}\\
    
    \begin{sideways}{\parbox{1.2in}{\centering{Ours}}}\end{sideways}
     \subfloat[DRUNet]{
    \includegraphics[width=0.23\columnwidth]{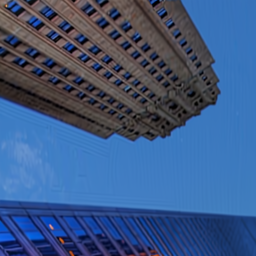}}\hfill
    \subfloat[DiffUNet]{
    \includegraphics[width=0.23\columnwidth]{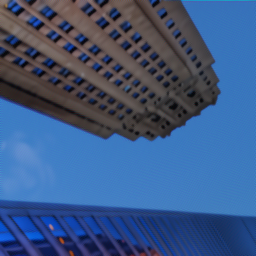}}\hfill
     \subfloat[Restormer]{
    \includegraphics[width=0.23\columnwidth]{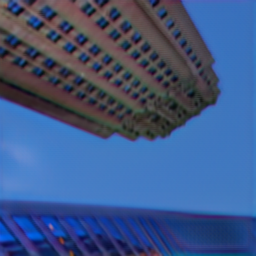}}\hfill
     \subfloat[SCUNet]{
    \includegraphics[width=0.23\columnwidth]{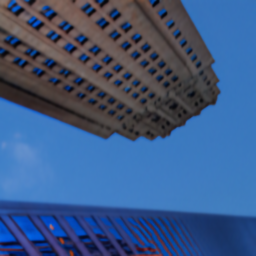}}\\

    \begin{sideways}{\parbox{1.2in}{\centering{Vanilla (Asym.)}}}\end{sideways}
     \subfloat[DRUNet (iter $600$)]{
    \includegraphics[width=0.23\columnwidth]{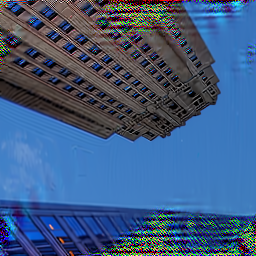}}\hfill
    \subfloat[DiffUNet (iter $=2000$)]{
    \includegraphics[width=0.23\columnwidth]{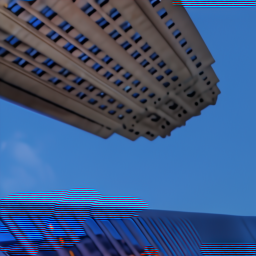}}\hfill
     \subfloat[Restormer (iter $=100$)]{
    \includegraphics[width=0.23\columnwidth]{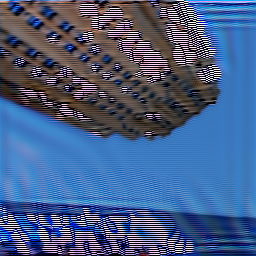}}\hfill
     \subfloat[SCUNet (iter $=100$)]{
    \includegraphics[width=0.23\columnwidth]{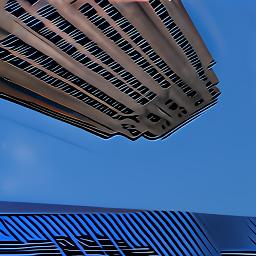}}
    
    \caption{Motion deblurring experiment on \textit{skyscraper} of Urban100~\cite{urban100} with kernel 7 from~\cite{levin_kernel_2009} and 0.01 additive noise. The reconstructions are using the DRUNet~\cite{zhang2021plug}, DiffUNet~\cite{choi2021conditioning}, Restormer~\cite{zamir2022restormer}, SCUNet~\cite{zhang2023practical}, denoisers in the PnP-HQS~\eqref{eq:pnphqs} framework. The PSNR(dB) values are: (b) $20.91$, (c) $28.82$, (d) $26.86$, (e) $24.44$, (f) $24.51$, (g) $28.31$, (h) $26.32$, (i) $24.92$, (j) $25.17$, (k) $5.52$ and (l) $19.75$, (m) $-\,8.44$ and (n) $16.29$.}
    \label{fig:vanilla-scyscrapper}
\end{figure}

\begin{figure}
    \centering
     \captionsetup[subfloat]{labelformat=empty,labelsep=none,justification=centering}
    \subfloat[original]{
    \includegraphics[width=0.30\linewidth]{figures/skyscraper/original_image.png}}\hspace{1em}
     \subfloat[blurry]{
    \includegraphics[width=0.30\linewidth]{figures/skyscraper/observed_image.png}}\\

    \begin{sideways}{\parbox{1.2in}{\centering{Equiv (Peak)}}}\end{sideways}
     \subfloat[DRUNet]{
    \includegraphics[width=0.23\columnwidth]{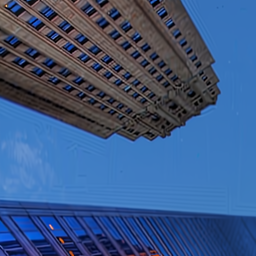}}\hfill
    \subfloat[DiffUNet]{
    \includegraphics[width=0.23\columnwidth]{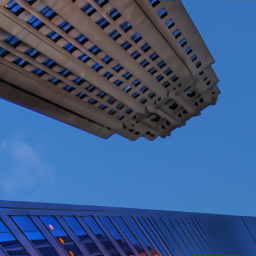}}\hfill
     \subfloat[Restormer]{
    \includegraphics[width=0.23\columnwidth]{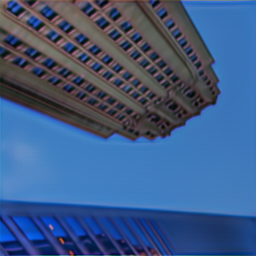}}\hfill
     \subfloat[SCUNet]{
    \includegraphics[width=0.23\columnwidth]{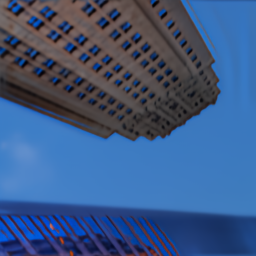}}\\
    
    \begin{sideways}{\parbox{1.2in}{\centering{Ours}}}\end{sideways}
     \subfloat[DRUNet]{
    \includegraphics[width=0.23\columnwidth]{figures/skyscraper/DRUNet/reconstructed_image.png}}\hfill
    \subfloat[DiffUNet]{
    \includegraphics[width=0.23\columnwidth]{figures/skyscraper/DiffUNet/reconstructed_image.png}}\hfill
     \subfloat[Restormer]{
    \includegraphics[width=0.23\columnwidth]{figures/skyscraper/Restormer/nlm_reconstructed_image.png}}\hfill
     \subfloat[SCUNet]{
    \includegraphics[width=0.23\columnwidth]{figures/skyscraper/SCUNet/reconstructed_image.png}}\\

    \begin{sideways}{\parbox{1.2in}{\centering{Equiv (Asym.)}}}\end{sideways}
     \subfloat[DRUNet (iter=1000)]{
    \includegraphics[width=0.23\columnwidth]{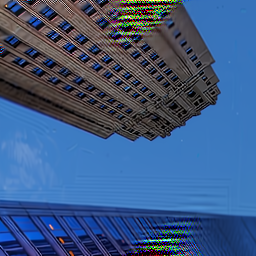}}\hfill
    \subfloat[DiffUNet (iter=2500)]{
    \includegraphics[width=0.23\columnwidth]{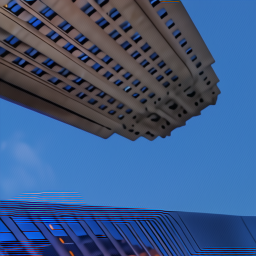}}\hfill
     \subfloat[Restormer (iter=200)]{
    \includegraphics[width=0.23\columnwidth]{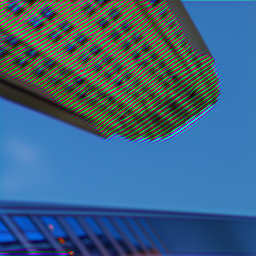}}\hfill
     \subfloat[SCUNet (iter=100)]{
    \includegraphics[width=0.23\columnwidth]{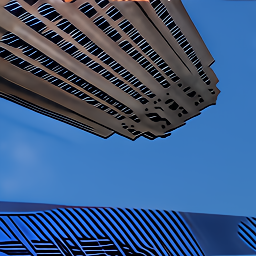}}
    
    \caption{Motion deblurring experiment on \textit{skyscraper} of Urban100~\cite{urban100} with kernel 7 from~\cite{levin_kernel_2009} and 0.01 additive noise. The reconstruction are using the DRUNet~\cite{zhang2021plug}, DiffUNet~\cite{choi2021conditioning}, Restormer~\cite{zamir2022restormer}, SCUNet~\cite{zhang2023practical}, denoisers in the PnP-HQS~\eqref{eq:pnphqs} framework. The PSNR(dB) values are: (b) $20.91$, (c) $28.95$, (d) $26.64$, (e) $25.46$, (f) $24.35$, (g) $28.31$, (h) $26.32$, (i) $24.92$, (j) $25.17$, (k) $7.95$, (l) $23.85$, (m) $17.46$ and (n) $16.86$.}
    \label{fig:equiv-skyscraper}
\end{figure}

%% file: gc.tex
\begin{figure}[ht]
    \centering
    \resizebox{0.9\columnwidth}{!}{
    \begin{tikzpicture}
        \begin{scope}[font=\small]
            \begin{axis}[%
                    table/col sep=comma,
                    width=12cm,
                    height=8cm,
                    ymax=40,
                    ymin=-5,
                    title={PSNR of $\x_k$},
                    xlabel={$k$},
                    ylabel={},
                    restrict y to domain=-5:40,
                    grid=both,
                    grid style={line width=.1pt, draw=gray!10},
                    major grid style={line width=.2pt,draw=gray!50},
                    minor tick num=5,
                    cycle list name=exotic,
                    legend entries={
                    $\x_0=\mathbf 0$\\
                    $\x_0=\mathbf 1$\\
                    $\x_0\sim \cU([0,1])$\\
                    $\x_0\sim \cN(\mathbf{0},\I)$\\$
                    \x_0= \A\bar{\x}+\n$ \\
                    $\x_0=\bar{\x}$\\
                    },
                    legend pos=south east,
                    legend cell align={left},
                ]
                \addplot[color=Fuchsia,line width=0.5pt,mark=o] table[x=k,y=zeros] {figures/gc/psnrs_ccd.csv};
                \coordinate (zeros) at (axis cs:0,3.35) {};
                \addplot[color=orange,line width=0.5pt,mark=square] table[x=k,y=ones] {figures/gc/psnrs_ccd.csv};
                \coordinate (ones) at (axis cs:0,5.75) {};
                \addplot[color=red,line width=0.5pt,mark=x] table[x=k,y=uniform] {figures/gc/psnrs_ccd.csv};
                \coordinate (uniform) at (axis cs:0,7.06) {};
                \addplot[color=green!75!black,line width=0.5pt,mark=star] table[x=k,y=normal] {figures/gc/psnrs_ccd.csv};
                \coordinate (normal) at (axis cs:0,-1.67) {};
                \addplot[color=violet,line width=0.5pt,mark=triangle] table[x=k,y=observed] {figures/gc/psnrs_ccd.csv};
                \coordinate (observed) at (axis cs:0,15.36) {};
                \addplot[color=Blue,line width=0.5pt,mark=diamond] table[x=k,y=clean] {figures/gc/psnrs_ccd.csv};
                \coordinate (gt) at (axis cs:1,38.80) {};
                \coordinate (converged) at (axis cs:70,27.49);
            \end{axis}
        \end{scope}
        \foreach \f/\c [count = \xi from 0] in {zeros/Fuchsia,ones/orange,normal/green!75!black,uniform/red,gt/blue,observed/violet}
        {
            \node[img,draw=\c,ultra thick] (i\f) at (-2,.4+\xi*1.15) {\includegraphics[width=1.1cm, height=1.1cm]{figures/gc/\f.png}
            };
            \draw[\c,ultra thick] (i\f.east) -- (\f);
        }
        \node[xshift=0cm,minimum height=1.75em,anchor=north] at (izeros.south) {$\x_0$};
        \node[img,draw=black,ultra thick] (iconverged) at (5.5,2) {\includegraphics[width=1.1cm, height=1.1cm]{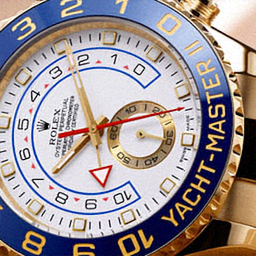}
        };
        \draw[black,ultra thick] (iconverged.north) -- (converged);
    \end{tikzpicture}
    }
    \caption{Independence of initialization of $\Rctr$ \eqref{eq:Rctr}.}
    \label{fig:gc}
\end{figure}